\documentclass[letterpaper]{article}
\newcommand{\mail}[1]{\href{mailto:#1}{\color{blue} #1}}
\newcommand{\blackfootnote}[1]{%
  \begingroup
  \hypersetup{allcolors=black}%
  \footnote{#1}%
  \endgroup
}
\newcommand{\additionalpackages}{\usepackage{fullpage,algorithm,algorithmic}}
\newcommand{\titleandauthors}{
\title{Online Learning in MDPs with Partially Adversarial Transitions and Losses}
\author{Ofir Schlisselberg\blackfootnote{Tel Aviv University; \mail{ofirs4@mail.tau.ac.il}} 
\and
Tal Lancewicki\blackfootnote{Meta AI; Part of this research conducted while the author was a student at Tel Aviv University; \mail{lancewicki@meta.com}} 
\and
Yishay Mansour\blackfootnote{Tel Aviv University and Google Research; \mail{mansour.yishay@gmail.com}}}
\maketitle}

\newcommand{\footer}{\section*{Acknowledgements}
OS, TL and YM are supported by the European Research Council (ERC) under the European Union’s Horizon 2020 research and innovation program (grant agreement No. 882396), by the Israel Science Foundation and the Yandex Initiative for Machine Learning at Tel Aviv University and by a grant from the Tel Aviv University Center for AI and Data Science (TAD). OS is also supported by the TAD Excellence Program for Doctoral Students in Artificial Intelligence and Data Science from the Tel Aviv University Center for AI and Data Science (TAD) and from the Israeli Council for Higher Education (CHE) Fellowship for Outstanding PhD Students in Data Science.

\newpage}
\newcommand{\footerofallfooters}{}

\usepackage{graphicx} 
\usepackage{tikz-cd}
\usepackage{dsfont}
\usepackage{mleftright}
\usepackage[colorlinks=true, allcolors=blue]{hyperref}

\additionalpackages
\usepackage{amsmath}
\usepackage{amssymb}
\usepackage{mathtools}
\usepackage{amsthm}
\usepackage{cleveref}
\usepackage{natbib}
\usepackage{multirow}
\usepackage{xfrac}
\usepackage{geometry}

\usepackage{tikz}
\usepackage{mathrsfs}
\usetikzlibrary{positioning}

\newtheorem{theorem}{Theorem}[section]
\newtheorem*{utheorem}{Theorem}
\newtheorem{lemma}[theorem]{Lemma}

\newtheorem{corollary}[theorem]{Corollary}
\newtheorem{definition}[theorem]{Definition}

\newtheorem{remark}{Remark}

\newenvironment{proofsketch}{{\bf Proof sketch:}}{\hfill\rule{2mm}{2mm}}

\newcommand\E{\mathbb{E}}

\newcommand{\A}{\mathcal{A}}
\renewcommand{\S}{\mathcal{S}}

\newcommand{\C}{\mathcal{C}}
\newcommand{\sinit}{s_{\text{init}}}
\newcommand{\R}{\mathcal{R}}
\newcommand{\calT}{\mathcal{T}}

\newcommand{\onebb}{\mathds{1}}

\DeclareMathOperator*{\argmin}{arg\,min}
\DeclareMathOperator*{\argmax}{arg\,max}

\newcommand{\KL}[2]{\text{KL}(#1 \;\|\; #2)}

\makeatletter
\newcommand{\ForceCrefTypeInEnv}[2]{%
  \AddToHook{env/#1/begin}{%
    \let\cref@oldlabel\label
    \def\label##1{\cref@oldlabel[#2]{##1}}%
  }%
  \AddToHook{env/#1/end}{\let\label\cref@oldlabel}%
}
\makeatother

\ForceCrefTypeInEnv{lemma}{lemma}
\ForceCrefTypeInEnv{definition}{definition}
\ForceCrefTypeInEnv{proposition}{proposition}
\ForceCrefTypeInEnv{claim}{claim}
\ForceCrefTypeInEnv{corollary}{corollary}
\ForceCrefTypeInEnv{example}{example}
\crefname{lemma}{lemma}{lemmas}
\Crefname{lemma}{Lemma}{Lemmas}
\crefname{definition}{definition}{definitions}
\Crefname{definition}{Definition}{Definitions}
\crefname{proposition}{proposition}{propositions}
\Crefname{proposition}{Proposition}{Propositions}
\crefname{claim}{claim}{claims}
\Crefname{claim}{Claim}{Claims}
\crefname{corollary}{corollary}{corollaries}
\Crefname{corollary}{Corollary}{Corollaries}
\crefname{example}{example}{examples}
\Crefname{example}{Example}{Examples}

\usepackage[normalem]{ulem}

\newcommand{\calF}{\mathcal{F}}

\newcommand{\bbE}{\mathbb{E}}

\newcommand{\curly}[1]{ {\left\{ #1 \right\}}}
\newcommand{\roundy}[1]{ {\left( #1 \right)}}
\newcommand{\squary}[1]{ {\left[ #1 \right]}}
\newcommand{\abs}[1]{ {\left | #1 \right |}}
\newcommand{\stat}{\textsf{stat}}

\makeatletter
\newcommand{\alglinelabel}[1]{%
  \addtocounter{ALC@line}{-1}
  \refstepcounter{ALC@line}
  \label{#1}
}
\makeatother
\begin{document}

\titleandauthors

\begin{abstract}
We study reinforcement learning in MDPs whose transition function is stochastic at most steps but may behave adversarially at a fixed subset of $\Lambda$ steps per episode. This model captures environments that are stable except at a few vulnerable points. We introduce \emph{conditioned occupancy measures}, which remain stable across episodes even with adversarial transitions, and use them to design two algorithms. The first algorithm addresses regret w.r.t. history-dependent policies and guarantees an upper bound of $\tilde{O}(H S^{\Lambda}\sqrt{K S A^{\Lambda+1}})$, where $K$ is the number of episodes, $S$ is the number of states, $A$ is the number of actions and $H$ is the episode's horizon; or $\tilde{O}(H\sqrt{K (S A)^{\Lambda+1}})$ when the $\Lambda$ adversarial steps are consecutive. The second algorithm addresses regret w.r.t Markov policies and, under the assumption that the adversarial steps are consecutive, improves the dependence on $S$ to $\tilde{O}(H\sqrt{K S^{3} A^{\Lambda+1}})$. We further give a $K^{2/3}$-regret reduction that removes the need to know which steps are the $\Lambda$ adversarial steps. 

\end{abstract}

\section{Introduction}
The standard Reinforcement Learning (RL) framework \citep{10.5555/3312046,MannorMT-RLbook} assumes a
\emph{stationary} environment: both the transition function and the loss function remain fixed over time.
\citet{even2009online} initiated the study of \emph{adversarial MDPs}, allowing losses to vary adversarially from episode to episode.
However, the classical adversarial-MDP literature almost always assumes that the \emph{transition
function} - which determines the next state - remains stationary.
This assumption is often reasonable: even in adversarial or non-stationary systems, the dynamics are typically stable for most steps.
Yet in many real-world settings this assumption fails at a few structurally vulnerable points.
For example, consider a specific vulnerable router in a communication network, robotic control with few obstacles, or cyber-physical systems that contain
specific steps that are susceptible to faults or attacks. At such settings,
assuming a stationary or stochastic transition is unrealistic, but assuming that almost all steps are stochastic is reasonable.

The fully adversarial transition setting has been studied by \citet{abbasi2013online,wangminimax}, who provided guarantees
under full-information and bandit feedback model, respectively.
On the hardness side, \citet{pmlr-v139-tian21b,liu2022learning} proved that in bandit-feedback settings the regret must be exponential in the episode length $H$.
However, this fully adversarial assumption is overly pessimistic: often, only a \emph{fixed} and
\emph{small} subset of states (or time steps in the episode) might exhibit adversarial behavior, while the rest are governed by a
consistent stochastic model.

This motivates our model, where
we assume that there is a fixed subset of $\Lambda$ time steps within each episode of length $H$ at which the transition
function may be selected adversarially, while the remaining $H-\Lambda-1$ steps follow an unknown but
stationary stochastic model.
We show that in this setting the optimal regret is \emph{exponential only in $\Lambda$} rather than in
the full horizon $H$, making the model particularly appealing when the number of vulnerable steps
is small.
This formulation provides a continuum between the classical stationary setting ($\Lambda=0$) and the
fully adversarial setting ($\Lambda=H - 1$), and isolates the intrinsic difficulty created by adversarial
transition steps.

An alternative popular approach to partially non-stationary dynamics is the \emph{corruption model}, in which
the transition function remains close to a base stochastic model but may drift from it at some time steps.
While corruption models permit adversarial deviations, their regret typically scales with the total
amount of deviation \citep{jin2023no}.
Consequently, a large deviation at even a \emph{single state}  leads to linear regret in the number of episodes $K$.
Such models therefore fail to capture settings where a few fixed steps may be arbitrarily corrupted.
In contrast, our assumption of a known separation between stochastic and adversarial steps removes
this additive corruption penalty, regardless of how extreme the deviations in the adversarial steps
might be.

The standard benchmark in MDP literature is regret with respect to the class of Markov policies, where the action at each step depends only on the current state. We refer to it as Markov regret. This is justified in the stochastic setting, where it is well known that an optimal policy can always be chosen to be Markovian.
However, \citet{wangminimax} recently showed that this property no longer holds when transitions are adversarial. In that case, the optimal policy may depend on the entire history within the episode. Accordingly, they consider a stronger benchmark based on the class of history-dependent policies, which we call history-dependent regret. 
Following this line of work, we study both notions of regret: with respect to Markov policies and with respect to history-dependent policies.

\textbf{Our contributions:}

\textbf{Partially adversarial transitions.}
Our main results concern the setting where the transition function is adversarial only at a fixed
set of $\Lambda$ steps out of $H$.
In \Cref{sec:action_based} we give an algorithm achieving history-dependent regret approximately
\[
    \sqrt{H^2 K S^{2\Lambda+1} A^{\Lambda+1}},
\]
where $K$ is the number of episodes, $S$ is the number of states, $A$ is the number of actions and $H$ is the episode's horizon. Additionally, under the assumption that the adversarial steps are consecutive we achieve an improved history-dependent regret of approximately
\[
    \sqrt{H^2 K S^{\Lambda+1} A^{\Lambda+1}}.
\]
In \Cref{sec:subpolicy} we present a second algorithm which, under the Markov regret benchmark,
improves the dependence on $S$ to be polynomial in the case of consecutive adversarial steps.
In this setting, it achieves regret
\[
    \sqrt{H^2 K S^3 A^{\Lambda+1}}.
\]

\textbf{Unknown adversarial steps.}
Both algorithms mentioned above assume knowledge of which steps are adversarial.
In \Cref{sec:unknown} we give a general reduction that removes this assumption and yields an
algorithm that does not require prior knowledge of which steps are adversarial at the cost of a $K^{2/3}$ regret dependence.

\textbf{Fully adversarial transitions.}
As a side contribution,
we complete the Markov regret landscape for all feedback structures.
We prove a lower bound matching the full-information upper bound of \citet{abbasi2013online}.
For the bandit setting, we slightly strengthen the lower bound of
\cite{pmlr-v139-tian21b,liu2022learning} and provide a matching upper bound.
We also resolve the intermediate regimes where losses are bandit but transitions are observed,
and vice versa.
Due to space limitations we defer all of the fully adversarial results and proofs to 
Appendix ~\ref{sec:full_adv}.

Our results establish a refined picture of adversarial MDPs: while fully adversarial transitions may force
the regret to be exponential in the horizon $H$, the dependence becomes milder when the adversarial influence is
restricted to a small, fixed subset of steps in the horizon.

\subsection{Related Work}
\citet{abbasi2013online} were the first to study MDPs with adversarial transition functions; in addition to providing an upper bound in the full-information setting, they showed that even there achieving sublinear regret is computationally hard.
More recently, \citet{wangminimax} study the fully adversarial setting under the stronger benchmark of history-dependent policies, establishing a minimax regret of $\tilde{\Theta}(\sqrt{(SA)^{H}K})$. Earlier works \citep{bai2020near,pmlr-v139-tian21b,liu2022learning} focused on regret with respect to Markov policies and emphasized the inherent hardness of this setting, which in turn motivated relaxing the learning objective—for example, by measuring regret with respect to the minimax value of a Markov game rather than against the best fixed Markovian policy.

Another line of work considers the \emph{corrupted MDP} model 
\citep{lykouris2019corruption,chen2021improved,wu2021reinforcement,wei2022model,jin2023no}.  
In this setting, both the losses and the transition function come from a fixed stochastic base model,
but an adversary may corrupt them with bounded total variation budgets:
$C^L$ for the losses and $C^P$ for the transition probabilities.  
\citet{wei2022model} proved that, when competing against the base model, regret of order 
$\sqrt{K} + C^L + C^P$ (up to non--$K$-dependent terms) is achievable, and a matching lower bound was proved by \citet{wu2021reinforcement}.
\citet{jin2023no} showed that a regret bound of $\sqrt{K} + C^P$ can be obtained when competing 
against the corrupted model itself.  
A related line of work considers \textit{non-stationary MDPs} \cite{mao2020model,wei2021non} in which both the transitions and the losses may change between episodes, either a limited number of times or with bounded total variation, similarly to the corruption literature. The main difference is that these works study the harder notion of \textit{dynamic regret}, where the benchmark policy is also allowed to change over time. Importantly, in both the corruption literature and the more general non-stationary MDP setting, state-of-the-art upper bounds incur linear regret even when only a single state or a single time step is fully adversarial across all episodes.

\section{Preliminaries}

We consider the problem of learning MDPs under partially (or fully) adversarial  transition and loss function. 
A finite-horizon MDP is defined by a tuple $(\S , \A , H , p, \ell)$, where $H$ is the horizon (i.e., episode length),
$\S$ and $\A$ are finite state and action spaces of sizes $|\S| = S$ and $|\A| = A$, respectively, 
$p: \S \times \A \times [H-1] \to \Delta_{\S}$ is a \textit{transition function} which defines the transition probabilities.
That is, $p_h(s' | s,a)$ is the probability to move to state $s'$ when taking action $a$ in state $s$ at time $h$. 
The  \textit{loss function} is $\ell$, where $\ell_h(s,a)$ is the loss of taking action $a$ in state $s$ at time $h$.

\textbf{Learner-environment interaction.}
Learning proceeds over \(K\) episodes.
At the beginning of each episode \(k\), the learner commits to a \emph{strategy}, meaning a rule that
specifies, for every step \(h \in [H]\) and every trajectory prefix
\((s^k_1, a^k_1, \ldots, s^k_h)\), a distribution over actions in \(\A\).
The episode then unfolds as follows.
The initial state is \(s^k_1 = s_{\mathrm{init}}\).
At each step \(h\), the learner observes the current state \(s^k_h\), samples an action
\(a^k_h\) according to the distribution prescribed by its strategy $\pi^k$ for the prefix
\((s^k_1, a^k_1, \ldots, s^k_h)\).
Then, the learner observes the loss $\ell^k_h(s^k_h, a^k_h)$ and the environment transitions to a next state
\(s^k_{h+1}\) sampled from the transition function \(p^k_h(\cdot \mid s^k_h, a^k_h)\).
This interaction protocol is summarized in Protocol~\ref{alg:protocol}.

\begin{algorithm}[tb]
\floatname{algorithm}{Protocol}
\caption{Interaction Protocol}
\label{alg:protocol}
\begin{algorithmic}
\FOR{$k \in [K]$}
    \STATE Learner selects a strategy $\pi^k$
    \STATE Initialize $s^k_1 = \sinit$
    \FOR{$h \in [H]$}
        \STATE Observe $s^k_h$, sample $a^k_h \sim \pi^k_h(\cdot \mid (s_1^k,a_1^k,\dots s_h^k))$, observe $\ell^k_h(s^k_h, a^k_h))$,
        and transition to $s^k_{h+1} \sim p^k_h(\cdot \mid s^k_h, a^k_h)$
    \ENDFOR
\ENDFOR
\end{algorithmic}
\end{algorithm}

\paragraph{Partially adversarial dynamics.}
Under partially adversarial dynamics, the transition function varies across episodes only at a subset of the time steps, while remaining stationary at the rest of the steps. Formally, let ${\mathbf \Lambda} \subseteq [H-1]$ with $\abs{\mathbf{\Lambda}} = \Lambda$ denote the set of adversarial steps. Then, for every $h \notin {\mathbf \Lambda}$, there exists a stationary transition function $p_h^{\stat}$ such that $p_h^k = p_h^{\stat}$ for all $k \in [K]$. The loss function, in contrast, remains fully adversarial and may change arbitrarily between episodes at all time steps. We note that in the standard finite-horizon adversarial MDP setting (e.g., \citet{zimin2013online,rosenberg2019online}), the transition function is stationary, and thus,  corresponds to the special case where ${\mathbf \Lambda} = \emptyset$.

\textbf{Regret.} 
A Markov policy is a collection of mappings
$
\pi_h : \S \to \Delta_\A, \, h \in [H],
$
and we denote the set of Markov policies by $\mathcal{M}$.
A history-dependent policy is a collection of mappings
$
\pi_h : \S^h \times\A^{h-1} \to \Delta_\A,\, h \in [H],
$
and we denote the set of history-dependent policies by $\mathcal{H}$.
Given a transition function $p$ and losses $\ell$, the value of a policy $\pi$ is 
\[
V_1^{\pi,p}
    = \mathbb{E}_{\pi,p}
        \Bigl[ \sum_{h=1}^H \ell_h(s_h,a_h)
        \,\Big|\, s_1=s_{\mathrm{init}} \Bigr],
\]
where the expectation is taken over the induced trajectory. The Markov and history-dependent regret after $K$ episodes is defined as
\begin{align*}
\mathcal{R}_K^{\mathcal{M}}
    &= \sum_{k=1}^K V_1^{\pi^k,p^k}
      - \min_{\pi \in \mathcal{M}} \sum_{k=1}^K V_1^{\pi,p^k} 
  &\mathcal{R}_K^{\mathcal{H}}
    = \sum_{k=1}^K V_1^{\pi^k,p^k}
      - \min_{\pi \in \mathcal{H}} \sum_{k=1}^K V_1^{\pi,p^k}
\end{align*}

\textbf{Occupancy measure.} 
Given a policy $\pi$ and a transition function $p'$, the \textit{occupancy measure} $q^{\pi,p'} \in [0,1]^{HSA}$ is a vector, where $q^{\pi,p'}_h(s,a)$ is the probability to visit state $s$ at time $h$ and take action $a$. 
%
Importantly, the value of $\pi$ can be written as the dot product between its occupancy measure and the cost function, i.e., $V^{\pi,p'}_1 = \langle q^{\pi,p'} , \ell \rangle$.
Whenever $p'$ is omitted from the notations $q^{\pi,p'}$ and $V^{\pi,p'}$, this means that they are with respect to the true transition 
function $p$.

\textbf{Additional Notations.} For each $h\in[H]$ let $\mathbf{\Lambda}_h=\curly{h'\in \mathbf{\Lambda} \;\mid\; h'< h }$ denote the set of adversarial steps occurring before step $h$ and let $\Lambda_h = \abs{\mathbf{\Lambda}_h}$. We denote $\E_k$ to be the expectation conditioned on the observations of the learner up to (but not including) episode $k$. 

\section{The challenge of adversarial steps} \label{sec:challenge}
A natural starting point for learning in MDPs is the family of occupancy-measure–based (OM) algorithms
\citep{zimin2013online,rosenberg2019online,jin2019learning}.  
These methods optimize an occupancy measure $q_h$ that represents the distribution over $(s,a)$ induced
by a policy, and then extract a policy from the optimized occupancy measure.  
Their analysis fundamentally relies on the assumption that for every policy $\pi$ there exists a
\emph{single} occupancy measure $q^\pi$ that is consistent across episodes.

With adversarial transitions, this structure collapses.  
The adversary can make a fixed policy $\pi$ reach a given state $s$ with high probability in some
episodes and with low probability in others simply by changing the transition at a single step.  
Thus, $\pi$ no longer induces a consistent occupancy measure. That is, the OM becomes a sequence $\{q^{k,\pi}\}_k$ and can vary arbitrarily across episodes.  
The core conceptual difficulty is therefore twofold: first, even though the benchmark policy is itself fixed, its occupancy measure, $q^*$, changes between episodes. Second, the set of occupancy measures, which in the stationary case is a fixed convex set, now also changes between rounds, so it is unclear what object an OM-based method should optimize over.

Another family of algorithms that achieve sub-linear regret in MDPs with non-stochastic losses and stationary dynamics is policy-optimization algorithms \cite{even2009online,shani2020optimistic,luo2021policy}. However, the analysis of these algorithms also heavily relies on the fact that the occupancy measure of the benchmark policy remains fixed.

\paragraph{Conditioned occupancy measure.}
To overcome these difficulties, we introduce the notion of \emph{conditioned occupancy measure (COM)}, denoted by $\mu$, which is a variant of the occupancy measure that remains invariant across episodes, even in the presence of adversarial steps.  
Before defining COM formally, we introduce the notion of a \emph{condition}.

For each step $h$, the set of conditions $\C_h$ is defined to capture all possible behaviors and outcomes of the adversarial steps that occur before $h$.  
Concretely, for every adversarial step $h' < h$, the condition specifies both the realized pair $(s_{h'},a_{h'})$ and the realized outcome of that step, i.e., the next state $s_{h'+1}$.  
The precise definition of $\C_h$ differs between the two algorithms proposed later, and we will specify it separately in each case.

Given this notion of conditions, a COM differs from a standard occupancy measure in that it tracks not only the probability of being in state $s$ and taking action $a$ at step $h$, but also the condition $c \in \C_h$ under which this $(s,a)$ pair is reached.  
Intuitively, instead of quantifying the unconditional probability of $(s_h,a_h)$, a COM quantifies the probability of $(s_h,a_h)$ \emph{together with} a particular outcome of adversarial transitions that occurred earlier in the episode.

Formally, for a condition of the form $c = (s_{h'},a_{h'},s_{h'+1})_{h'\in\mathbf{\Lambda}_h}$, the quantity $\mu_h(s,a,c)$ represents
\begin{align*}
\Pr[s_h = s,\; a_h = a,\; c_h = c \;|\; c \text{ is feasible in the current episode}].
\end{align*}
Here, ``feasible’’ means that the adversarial transition at every $h'\in\mathbf{\Lambda}_h$ is such that playing $a_{h'}$ in $s_{h'}$ indeed leads to $s_{h'+1}$. Importantly, the quantity above does not depend on transition probability on the adversarial steps, and thus, remains stationary over time.


To complete the definition, $\varrho_k(c)$ denotes the probability that condition $c$ is feasible in episode $k$.  
In the example above, this corresponds to the episode-$k$ transition probability
\[
\varrho_k(c) = \prod_{h'\in \mathbf{\Lambda}_h}p_{h'}^k(s_{h'+1}\mid s_{h'},a_{h'}).
\]

Similar to occupancy measures, one can construct a policy from a COM.
We say that a policy $\pi$ is \emph{induced} by a COM $\mu$ if :
\begin{align*}
    \pi_h^{k}(a \mid s, c) = \frac{\mu^k_h(s,a,c)}{\sum_{a'}\mu_h^k(s,a,c)}
\end{align*}

The key structural property of COMs (proved in \Cref{lem:com_to_om}) is the decomposition
\[
    q_h^k(s,a) 
    = \sum_{c \in \C_h} \mu_h(s,a,c)\,\varrho_k(c),
\]
which expresses the usual occupancy measure as the combination of an \emph{episode-independent} component $\mu$ and an \emph{episode-dependent} component $\varrho_k$.  
Thus, the COM framework disassembles $q_h^k$ into a stable part that we can optimize over, and a varying part that we only need to estimate.

This decomposition allows us to rewrite the value function in terms of COMs.  
Indeed,
\begin{align*}
V_{k,1}^\pi
    &= \sum_{h,s,a} q_h^k(s,a)\,\ell_k(s,a)
    = \sum_{h,s,a,c} \mu_h(s,a,c)\,\varrho_h^k(c)\,\ell_k(s,a).
\end{align*}
This leads to a crucial property of COMs:
\begin{lemma}\label{lem:com_best}
For every history dependent policy $\pi$ there exists a COM-induced policy $\pi'$ such that
$
    V_{1}^{\pi} = V_{1}^{\pi'}
$.
\end{lemma}

The construction and lemma above yield two key consequences.
First, the algorithm can optimize over the (episode-invariant) set of COMs, while treating
$\varrho_h^k(c)\,\ell_h^k(s,a)$ as the effective per-episode loss.
Since only bandit feedback is available for both losses and transitions, the algorithm constructs an estimator $\hat{\ell}_h^k(s,a,c)$ satisfying
$
\mathbb{E}\!\left[\hat{\ell}_h^k(s,a,c)\right] \approx \varrho_h^k(c)\,\ell_h^k(s,a).
$
Second, since there exists an optimal policy that is induced by a COM, the algorithm in fact guarantees regret with respect to the optimal (history-dependent) policy.

\section{History Dependent Regret}\label{sec:action_based}
\begin{algorithm}[ht]
    \caption{\texttt{COM-OMD}} 
    \label{alg:bandit-bandit com}
    \begin{algorithmic}[1]
        \STATE \textbf{Input:} Step size $\eta$, implicit exploration constant $\gamma$, confidence constant $\delta$.
        \STATE \textbf{Initialization:} Set $\mu^1$ (\Cref{def:initialization}) and $\pi^1$ to be uniform.
        
        \FOR{$k=1,2,...,K$}

            \STATE $s_1^k = \sinit$
            \FOR{$h = 1,...,H$} 
                \STATE Play action $a_h^k \sim \pi_h^k(\cdot\mid s_h^k, c_h^k)$ where $c_h^k = \roundy{s_{\tilde h}^k, a_{\tilde h}^k, s^k_{\tilde{h}+1}}_{\tilde{h} \in \mathbf{\Lambda}_h}$
                and observe $s_{h+1}^k$
            \ENDFOR

            \STATE Update empirical mean $\{\bar p_h^k\}_{h=1}^H$, confidence radiuses $\{\epsilon_h^k\}_{h=1}^H$ and COM polytope $\Delta_k = \Delta(\{\bar p_h^k, \epsilon_h^k\}_{h=1}^H)$.

            \STATE \alglinelabel{line:opt_estimator} Compute upper COM $u_h^k(s,a,c) =  \max_{\mu \in \Delta_{k-1}} \mu_h(s,a,c) $ for each visited triplet $(s,a,c)$ 
            \STATE Compute loss estimator 
            
            $\hat \ell_h^k(s,a,c) = \frac{\ell_h^k(s,a) \onebb\{s_h^k = s, a_h^k = a, c_h^k = c\} }{u_h^k(s,a,c) + \gamma}$\alglinelabel{line:loss_est}
            
            \STATE Update COM by:
            
            $
                {\hat\mu^{k+1} = \argmin_{\hat\mu \in \Delta_k} \eta \langle \hat\mu , \hat \ell^k \rangle + \KL{\hat\mu}{\hat\mu^k}.}
            $ \alglinelabel{line:omd_step}
            \STATE Update policy: $\pi_{h}^{k+1}(a\mid s,c)
            =\frac{\hat\mu_{h}^{k+1}(s,a,c)}{\sum_{a'}\hat\mu_{h}^{k+1}(s,a',c)}$.\alglinelabel{line:policy}
        \ENDFOR
    \end{algorithmic}
\end{algorithm}
Algorithm~\ref{alg:bandit-bandit com} adapts the occupancy-measure OMD algorithm of \cite{jin2019learning} to the COM framework: instead of optimizing over occupancy measures, it optimizes directly over conditioned occupancy measures.  

We begin by constructing confidence radii for the \emph{stochastic} transition steps.  
As in \citet{jin2019learning}, for every $(s,a,s')$ we define 
\[\epsilon_h^k(s'\mid s,a) = 2\sqrt{\frac{\bar{p}_h^k(s'\mid s,a)\ln\!\left(\frac{KSA}{\delta}\right)}
                   {\max\{1,N_h(s,a)-1\}}}
      + \frac{14\ln\!\left(\frac{KSA}{\delta}\right)}
                   {\max\{1,N_h(s,a)-1\}},
\]
where $N_h(s,a)$ is the number of visits to $(s,a)$ at step $h$ up to episode $k$. Given the empirical transitions $\bar p^k$ and these radii, we construct the COM polytope  
$
\Delta\big(\{\bar p_h^k,\epsilon_h^k\}_{h=1}^H\big),
$
defined formally in \Cref{def:polytope}.  
By \Cref{lem:com_polytope}, this polytope contains all COMs compatible with transition functions lying inside the confidence bounds.

The next component is the construction of the loss estimator $\hat{\ell}$.  
Our goal is to ensure
$
\mathbb{E}\!\left[\hat{\ell}_h^k(s,a,c)\right]
    \approx \varrho_h^k(c)\,\ell_h^k(s,a),
$
since the OMD update is performed on the episode-independent component $\mu$.  
We exploit three facts:
\begin{enumerate}
    \item The algorithm observes $(s_h^k,a_h^k,c_h^k)$ in each step.
    \item The loss $\ell_h^k(s_h^k,a_h^k)$ is available in the bandit-loss setting.
    \item $\Pr\squary{s_h^k=s,a_h^k=a,c_h^k=c} = \mu_h(s,a,c)\varrho_k(c)$
\end{enumerate}
Thus, an unbiased estimator for $\varrho_h^k(c)\,\ell_h^k(s,a)$ can be obtained by dividing the indicator  
$\mathds{1}\{s_h^k=s, a_h^k=a, c_h^k=c\}$  
by $\mu_h(s,a,c)$.  
Since $\mu$ is not known exactly, the algorithm uses an \emph{optimistic} upper bound $u$ as defined in Line~\ref{line:opt_estimator}. It additionally biases the estimator by $\gamma$ to ensure implicit exploration and define the loss estimator in Line~\ref{line:loss_est}. Finally, the algorithm performs an OMD update over the COM polytope to find $\hat\mu^k$ (Line~\ref{line:omd_step}) and recovers the policy $\pi^k$ (Line~\ref{line:policy}).

\begin{theorem}\label{thm:first_regret}
\Cref{alg:bandit-bandit com} with $\eta=\gamma=1/\sqrt{KA^{\Lambda+1}S}$ has w.p $1-9\delta$:
\begin{align*}
    \R_K^{\mathcal{H}}\le \tilde{O}\roundy{HS^{\Lambda}\sqrt{KSA^{\Lambda+1}} + H^3S^2A + \sqrt{H^4S^2AK}}
\end{align*}
If the adversarial steps are consecutive, with $\eta=\gamma=\sqrt{S^{\Lambda-2}/KA^{\Lambda+1}}$, has w.p $1-9\delta$:
\begin{align*}
    \R_K^{\mathcal{H}} \le \tilde{O}\roundy{H\sqrt{KS^{\Lambda+2}A^{\Lambda+1}} + H^3S^2A + \sqrt{H^4S^2AK}}
\end{align*}
and runs in time polynomial in $S^\Lambda,A^\Lambda, K, H$.
\end{theorem}

The bound scales exponentially only in the number of adversarial steps $\Lambda$, matching the intended separation between stochastic and adversarial dynamics. We note that the for $H=\Lambda-1$ (fully adversarial) the bound nearly matches the lower bound for history-dependent regret of \cite{wangminimax},
and when $\Lambda=0$ (stationary transitions), our bound recovers the state-of-the-art regret guarantees for adversarial (non-stochastic) losses with stationary dynamics.

\begin{proofsketch}
We decompose the regret following similar to \citet{jin2019learning}:
\begin{align*}
    \R_K^{\mathcal{H}} &= \underbrace{\sum_{k,h,s,a}\roundy{q_h^{p_k,\pi_k}(s,a)-\sum_{c\in\C_h}\hat{\mu}_h^k(s,a,c)\varrho_k(c)}\ell_h^k(s,a)}_{\textsc{Error}}\\
    &\quad+\underbrace{\sum_{k,h,s,c}\hat{\mu}_h^k(s,a,c)\roundy{\varrho_k(c)\ell_h^k(s,a) - \hat{\ell}_h^k(s,a,c)}}_{\textsc{Bias1}}\\
    & \quad+ \underbrace{\sum_{k,s,h,a,c}\roundy{\hat{\mu}_h^k(s,a,c) - \mu_h^*(s,a,c)}\hat{\ell}_h^k(s,a,c)}_{\textsc{Reg}} \\
    &\quad+ \underbrace{\sum_{k,s,h,a,c}\mu_h^*(s,a,c)\hat{\ell}_h^k(s,a,c) - \sum_{k,s,h,a}q^{p_k,\pi^*}_h(s,a)\ell_h^k(s,a)}_{\textsc{Bias2}}\\
\end{align*}
Recall that $\hat \mu^k$ estimates $\mu^k$ and that
$q_h^{p_k,\pi_k}(s,a) = \sum_{c\in\C_h}{\mu}_h^k(s,a,c)\varrho_k(c)$.
Thus, \textsc{Error} is the deviation due to estimating $\hat \mu^k$ using the estimated stochastic transition steps;  
\textsc{Bias1} and \textsc{Bias2} are the bias of the loss estimator;  
\textsc{Reg} is the regret of the OMD update over the COM polytope.  

The bound on $\textsc{Bias2} \leq \tilde O(H / \gamma)$ is relatively standard. It follows the fact that $\hat{\ell}_h^k(s,a,c)$ is an optimistic estimate of $\varrho_k(c)\ell_h^k(s,a)$. I.e., it is smaller in expectation, given the transition estimate is within the confidence interval (which occurs with high probability). We now turn to analyzing the rest of the terms.

\textbf{Bounding \textsc{Error} and \textsc{Bias1}.} The key technical tool is the following lemma.
\begin{lemma}[Informal; formally in \Cref{lem:huge_lem}]\label{lem:huge_lem_mt}
For every step $h$ and a collection of transitions $\curly{p_k^{c,s}}_{c\in \C_h,\,s\in S}$ such that for all $c,s$, $p_s^{k,c}\in \mathcal{P}_k$ we have:
\begin{align*}
    \sum_{k,s,a,c}\varrho_k(c) \;\;|\mu^{p_s^{k,c},\pi_k}_h(s,a,c) - \mu^{p,\pi_k}_h(s,a,c)| &\lesssim \tilde{O}\roundy{HS\sqrt{AK}}
\end{align*}    
\end{lemma}
This lemma is a COM-analogue of Lemma 4 in \citet{jin2019learning}. 
The main difference is that here the optimistic COM $u_h^k(s,a,c)$ has different transition for every $(s,c)$ pair (not only for each $s$), hence the deviation bound must hold for a collection of of transitions for every $(s,c)$.

Given the lemma, bounding \textsc{Error} follows directly.  
For \textsc{Bias1}, note that in $\mathbb E_k[{\hat \ell_h^k (s,a)}] = \frac{\mu_h^k(s,a,c)\varrho_k(c)}{u_h^k(s,a,c) + \gamma}$, and using standard concentration bounds we show that, 
\begin{align*}  \textsc{Bias1}&\approx\sum_{k,h,s,a,c}\hat{\mu}_h^k(s,a,c)\ell_h^k(s,a)\roundy{\varrho_k(c) - \frac{\mu_h^k(s,a,c)\varrho_k(c)}{u_h^k(s,a,c) + \gamma}}\\
    &= \sum_{k,h,s,a,c}\frac{\hat{\mu}_h^k(s,a,c)}{u_h^k(s,a,c) + \gamma}\ell_h^k(s,a)\varrho_k(c)\Big(u_h^k(s,a,c) \\
    &\qquad\qquad+ \gamma - \mu_h^k(s,a,c)\Big)\\
    &\le \sum_{k,h,s,a,c}\varrho_k(c)\roundy{u_h^k(s,a,c) + \gamma - \mu_h^k(s,a,c)}\\
    &= \sum_{k,h,s,a,c}\squary{\varrho_k(c)\roundy{u_h^k(s,a,c) - \mu_h^k(s,a,c)}} + \gamma\sum_{k,h,s,a,c}\varrho_k(c) \\
\end{align*}
where we used that $\hat{\mu}_h^k(s,a,c)/(u_h^k(s,a,c)+\gamma)\le 1$ by definition of $u^k$.

The first term is controlled by \Cref{lem:huge_lem_mt}. To bound the second term notice that for every $c\in \C_h$, $\varrho_k(c)$ is the probability of $\Lambda_h$ transitions to be connected, specifically:
\begin{align*}
    \sum_{c\in\C_h}\varrho_k(c) \approx \prod_{m\in\mathbf{\Lambda}_h}\sum_{s,a,s'}p_m(s'\mid s,a) = \prod_{m\in\mathbf{\Lambda}_h}SA = \roundy{SA}^{\Lambda_h} \le \roundy{SA}^{\Lambda}
\end{align*}
In the case of consecutive adversarial steps, this bound can be sharpened to $\sum_{c\in\C_h}\varrho_k(c) \le SA^{\Lambda}$, which is precisely the source of the improved dependence in that setting.

Therefore,
    $\textsc{Bias1} \lesssim \tilde{O}\roundy{\sqrt{H^2S^2AK} + \gamma KH\roundy{SA}^{\Lambda+1}}$

\textbf{Bounding \textsc{Reg}.}
By the standard OMD bound,
\begin{align*}
    \textsc{Reg} \le \frac{1}{\eta}KL(\mu^*\|\hat{\mu}_1) + \frac{\eta}{2}\sum_{k,s,h,a,c}\hat{\mu}_h^k(s,a,c)\hat{\ell}^k_h(s,a,c)^2
\end{align*}

To bound the $KL$ term we need to upper bound the $L_1$ norm of the polytope. Notice that unlike occupancy measures, the sum of COM for every step doesn't sum to $1$. That is, for every adversarial step, the COM ``assumes" that the condition is connected (and thus, to get the real occupancy measure we need to multiply by $\varrho$, which is the probability that it is actually connected). Since there are $S$ possible targets for each of these $\Lambda$ adversarial connections, the $L_1$ norm of the polytope is bounded by $S^\Lambda$ (see formal proof in \Cref{lem:polytope_size}).

For the second moment term,
\begin{align*}
    \sum_{k,h,s,a,c}\hat{\mu}_h^k(s,a,c)\hat{\ell}_h^k(s,a,c)^2 &\le \sum_{k,h,s,a,c}\frac{\hat{\mu}_h^k(s,a,c)}{u_h^k(s,a,c)+\gamma}\hat{\ell}_h^k(s,a,c)\\
    &\le \sum_{k,h,s,a,c}\hat{\ell}_h^k(s,a,c)\\
    &\lesssim \sum_{k,h,s,a,c} \varrho_k(c) \\
    &\le KH\roundy{SA}^{\Lambda+1},
\end{align*}
where the last inequality is as in the \textsc{Bias1} analysis.

Therefore, $\textsc{Reg} \lesssim \tilde{O}\roundy{\frac{HS^\Lambda}{\eta} + \eta KH\roundy{SA}^{\Lambda+1}}$
\end{proofsketch}

\section{Markov Regret}\label{sec:subpolicy}
As shown in \Cref{thm:BB-ub}, when using Markov regret as the benchmark, the regret depends exponentially on $A^H$ but not on $S^H$.
\Cref{alg:bandit-bandit subpolicy} (given in the appendix) achieves a similar improvement in the partially adversarial setting, under the assumption that the $\Lambda$ adversarial steps form a \emph{consecutive} block.
Let $\tilde{h}_1$ be the first adversarial step and let $\tilde{h}_2$ denote the step \emph{after} the
last adversarial step, so the adversarial block is $\{\tilde{h}_1,\ldots,\tilde{h}_2-1\}$.

At step $\tilde{h}_1$, the algorithm does not select a single action.
Instead, it selects a \emph{sub-policy}~$\sigma$: a deterministic Markov policy defined only on the adversarial
block $\{\tilde{h}_1,\ldots,\tilde{h}_2-1\}$.
For all steps $h \ge \tilde{h}_2$, the condition takes the form $(s,\sigma,s')$, where $s$ is the state at step $\tilde{h}_1$, $\sigma$ is the realized sub-policy chosen at step $\tilde{h}_1$, and  $s'$ is the state at step $\tilde{h}_2$ obtained by executing $\sigma$.
Aside from this modification, which affects the COM polytope definition and the policy-induction step,
the algorithm is identical to \Cref{alg:bandit-bandit com}.

A key difference from the action-based COM algorithm is that a single episode provides feedback to
\emph{multiple} conditions.
Assume for simplicity that the adversarial transitions are deterministic, in episode $k$ each deterministic sub-policy induces a length-$\Lambda$ action sequence
$\vec{a}\in\A^\Lambda$ on the adversarial block.
Crucially, there may be many distinct sub-policies that induce the \emph{same} action sequence
$\vec{a}$ in that episode.
As a result, the observed trajectory and loss information is simultaneously informative for all
conditions corresponding to sub-policies that would have played $\vec{a}$.


\begin{theorem}
\Cref{alg:bandit-bandit subpolicy} with $\eta = \gamma = 1/\sqrt{S K A^{\Lambda+1}}$ satisfies, with
probability at least $1 - 9\delta$,
\[
\R_K^{\mathcal{M}} \le \tilde{O}\!\left(
        H\sqrt{K S^{3} A^{\Lambda+1}}
        + H^{3} S^{2} A
        + \sqrt{H^{4} S^{2} A K}
    \right).
\]
\end{theorem}

The proof is deferred to Appendix~\ref{apx:subpolicy}.  
We provide here an intuitive explanation for why the exponential dependence improves from
$(SA)^\Lambda$ to $A^\Lambda$, and in particular why the exponential dependence on $S$ disappears.
Recall that in the history-dependent COM algorithm, the exponential dependence on $\Lambda$ arose,
in the consecutive-steps case, from the $L_1$ diameter of the COM polytope, which scales as $S^\Lambda$.
This behavior stems from the fact that each adversarial step contributes a multiplicative factor of $S$,
since every step is treated as a separate ``connection'' in the COM representation.

In contrast, in this algorithm the entire adversarial block is represented by a \emph{single}
condition component.
This significantly reduces the size of the polytope: every vector in the COM polytope now has
$L_1$ norm at most $S$ (see \Cref{lem:polytope_size2}).
As a result, the $S^\Lambda$ factor disappears from the OMD analysis, leaving only an exponential
dependence on $A^\Lambda$.


\section{Unknown adversarial steps}\label{sec:unknown}

Finally, we show how to remove the assumption that the learner knows which steps are adversarial.
We instantiate an inner algorithm for each of the $\binom{H}{\Lambda}$ possible choices of
$\Lambda$ adversarial steps; exactly one of them, denoted $\mathcal{A}^*$, corresponds to the true
set.  
The outer algorithm runs EXP3 \cite{auer2002nonstochastic} over these inner algorithms, and therefore suffers at most the regret of
$\mathcal{A}^*$ plus an EXP3 regret that scales as $\sqrt{H^\Lambda K}$, as it operates over $\binom{H}{\Lambda} \approx H^\Lambda$ ``actions".

The difficulty is that, in each episode, we only obtain feedback for the inner algorithm actually
selected, since each inner algorithm runs a different strategy and expects feedback generated under
its own trajectory distribution.  
To enable unbiased estimation, we introduce an exploration probability $\xi$ in which the outer
algorithm selects an inner algorithm uniformly at random.
This ensures that importance-weighted estimators for the losses have
second moment $\lessapprox 1/\xi$, contributing an additional $\sqrt{K/\xi}$ to the regret.

Moreover, in \Cref{lem:reduction_huge_lemma} we show that the transition-estimation error is also bounded by
$\sqrt{K/\xi}$.  
Setting $\xi = K^{-1/3}$ balances these terms and yields an overall additive contribution of
$K^{2/3}$.  
Thus, combining the regret of $\mathcal{A}^*$, the EXP3 term, and the $K^{2/3}$ additive term gives a
reduction that does not require prior knowledge of the adversarial steps.
Formal proofs appear in Appendix~\ref{apx:unknown}.

\section{Discussion}

\textbf{Computational considerations.}
In this work we focused on the statistical limits of learning with partially adversarial
transitions. 
Understanding the computational hardness of this framework is an interesting direction for future
work.  
Even in the fully information setting, \citet{abbasi2013online} showed that achieving sublinear regret
is computationally hard in the fully adversarial model, suggesting that any algorithm may have to incur an exponential dependence on the horizon~$H$.  
In our framework, a natural question is whether one can design a full-information algorithm with optimal regret bound and whose runtime is exponential only in~$\Lambda$ (the number of adversarial steps), rather than in~$H$.
A second computational question arises in our Markov regret algorithm, which achieves statistical
dependence of only $A^\Lambda$ but remains computationally inefficient.  
It would be interesting to determine whether an \emph{efficient} algorithm with runtime polynomial in
$S,H,K,A^\Lambda$ is possible.
 
\textbf{Unknown adversarial steps.}
Our reduction for the setting where the adversarial steps are unknown yields a
$K^{2/3}$ regret term, which we believe is unlikely to be optimal.
This rate arises from the generic EXP3-based reduction and the need to control second moments via explicit exploration, rather than from an inherent statistical limitation of the problem. An interesting direction for future work is to go beyond black-box reductions and design algorithms
that explicitly reason about which steps may be adversarial.

\footer

\bibliographystyle{plainnat}   
\bibliography{refs}            

\newpage
\appendix
\onecolumn

\section{Fully Adversarial MDPs Under Different Feedback Models}\label{sec:full_adv}

We study the classical setting in which \emph{all} losses and \emph{all} transition functions may be adversarial, under Markov regret.  
All algorithms presented here optimize directly over the set of all deterministic Markov policies, and are therefore computationally inefficient. Recall that this inefficiency is unavoidable due to a hardness result of \cite{abbasi2013online} that shows that in fully adversarial MDPs, no polynomial-time algorithm can achieve no regret, even under full-information feedback.
The purpose of this section is to characterize the Markov regret landscape under different feedback models and motivate the partially adversarial setting.

When referring to full-information feedback for the losses, we mean that after each episode the learner observes the entire loss vector $\{\ell_h^k(s,a)\}_{s,a}$.  
When referring to full-information feedback for the transitions, we mean that after each episode the learner observes the full transition function $p^k$.

Our results in this section also yield several conceptual insights: (i) under full-information on the transition, the lower bounds continue to hold even if the transition
functions are known to the learner in advance;
and
(ii) exponential regret arises only when the learner receives bandit feedback on the
dynamics, whereas full information on the dynamics 
does not introduce such dependence, even if the losses still have bandit feedback.

\subsection{F/F: Full-information losses and dynamics.}
\cite{abbasi2013online} showed that running Hedge with the policies as actions yields $\R_K^{\mathcal{M}} = O\!\left(\sqrt{H^3 S K \log A}\right).$
We prove that this bound is \emph{tight}.

\begin{theorem}[Lower bound for F/F]
\label{thm:FF-lb}
Any algorithm in the F/F setting satisfies
\[
\R_K^{\mathcal{M}} \ge \Omega\!\left(\sqrt{H^3 S K \log A}\right).
\]
\end{theorem}

\begin{proof}
Let $\tilde S := S-3$ and $\tilde H := \lfloor H/2\rfloor$, $\bar H := \tilde H-1\ge 1$ and $
T := \frac{K}{\tilde S\,\bar H}
$.
Assume w.l.o.g.  that
$T$
is an integer (otherwise replace $T$ by $\lfloor K/(\tilde S\bar H)\rfloor$, losing only constant factor in the regret throughout this proof).

We embed $\tilde S\bar H$ independent copies of a hard $A$-expert problem into disjoint
\emph{episode blocks}, and inflate each expert loss by a factor $\tilde H$ by forcing the incurred loss
to be consistent across half of the episode.

We construct the MDP with set of states,
\[
\mathcal{S}
=
\{s_{\mathrm{init}}\}
 \cup 
\{s_1,\dots,s_{\tilde S}\}
 \cup 
\{\tilde s_0,\tilde s_1\}
\]

We create one copy of a hard expert problem for each pair $(s_i,h)$ with $i\in[\tilde S]$ and $h\in\{2,3,\dots,\tilde H\}$.

Partition the $K$ episodes into $\tilde S\bar H$ consecutive blocks of equal length $T$. Formally, the block associated with the pair $(s_i,h)$ is
\[
\mathcal{K}_{s_i,h}  :=  \{\, bT+1,  bT+2,  \dots,  (b+1)T \,\},
\]
where
$
b = (i-1)\bar H + (h-2)
$.

Fix a pair $(s_i,h)$.
Within this copy, we consider an adversarial expert-loss sequence
$\tilde \ell^{(i,h)}_1,\dots,\tilde \ell^{(i,h)}_T\in\{0,1\}^A$ over $T$ rounds.
We will reduce learning in this block to online learning with expert feedback on this sequence where the known regret lower bound is $\Omega(\sqrt{T \ln A})$ (see e.g., \citet{slivkins2024introductionmultiarmedbandits}).

We will now describe the transition dynamics within the block $\mathcal{K}_{s_i,h}$. For $k \in \mathcal{K}_{s_i,h}$,

\begin{itemize}

\item From the initial state we transition to  state $s_i$ deterministically:
\[
p^{k}_{1}(s_i \mid s_{\mathrm{init}}, a) = 1
\quad\text{for all }a\in[A].
\]

\item For $h'=2,3,\dots,h-1$,  we stay in $s_i$ deterministically:
\[
p^{k}_{h'}(s_i \mid s_i, a)=1
\quad\text{for all }a\in[A].
\]

\item For step $h$ (the ``expert decision step''),
let $t_k := k-bT \in\{1,\dots,T\}$ be the round index of episode $k$ inside block,
 define the step-$h$ transition as:
\[
p^{k}_{h}(\tilde s_1 \mid s_i, a) = \tilde \ell^{(i,h)}_{t_k}(a),
\qquad
p^{k}_{h}(\tilde s_0 \mid s_i, a) = 1-\tilde \ell^{(i,h)}_{t_k}(a).
\]
That is, choosing an action $a$ sends the agent to $\tilde s_1$ if the expert loss is $1$,
and to $\tilde s_0$ if the expert loss is $0$.

\item For steps $h' = h+1,\dots,H$, we keep the state fixed until the end of the episode:
\[
p^{k}_{h'}(s\mid s,a)=1\quad\text{for }s\in\{\tilde s_0,\tilde s_1\},
\]
\end{itemize}

Losses are deterministic and defined as,
\[
\ell_{h'}^k(s,a)
=
\begin{cases}
1, & \text{if } h'\in\{\tilde H,\tilde H+1,\dots,H\}\text{ and } s=\tilde s_1,\\
0, & \text{otherwise}.
\end{cases}
\]
Hence, if the trajectory reaches $\tilde s_1$ by time $h\le \tilde H$, then the episode incurs loss
exactly $H-\tilde H+1 = \Theta(H)$.
If the trajectory reaches $\tilde s_0$, the episode incurs loss $0$.

Combining this with the transition definition, for episode $k\in\mathcal{K}_{s_i,h}$, the episode loss is $\Theta( H \cdot \tilde \ell^{(i,h)}_{t_k}(a_k))$ where $a_k$ is the action that the learner choose in episode $k$ in $s$ at step $h$. This equivalent to scaling the expert loss $\tilde \ell^{(i,h)}_{t_k}(a_k)$ by a factor $\Theta(H)$, and thus the regret in these rounds is $\Omega(H\sqrt{T \ln A})$.
Summing over the $\tilde S\bar H = \Theta(HS)$ blocks we get that the total regret is $\Omega(H^2 S \sqrt{T \ln A}) = \Omega(  \sqrt{H^3 S K \ln A})$ as desired.
\end{proof}
Thus the F/F regime is completely resolved.

\subsection{B/F: Bandit losses, full-information dynamics.}
\subsubsection{Upper bound}
Hedge over policies fails since some losses are unobserved.  
However, we present an EXP4-style algorithm:
\begin{algorithm}[tb]
\caption{Bandit/Full Policy-based EXP4 (\texttt{BF-Pb-EXP4})}
\label{alg:bandit-full}
\begin{algorithmic}
    \STATE \textbf{Initialization:} Set $\rho^1$ to be the uniform distribution over all deterministic Markovian policies
    \FOR{$k = 1,...,K$}
        \STATE Sample a policy $\pi^k \sim \rho^k$, execute it and observe $\{(s_h^k, a_h^k, \ell_h^k)\}_{h=1}^H$
        \STATE Compute estimated loss for each $s\in \S, a\in\A$ by,
        \begin{align*}
            \hat{\ell}_h^k(s,a) &= \frac{\ell_k^h(s,a)\mathds{1}_k^h\squary{s,a}}{\sum_\pi p_k(\pi)q_{k,h}^\pi(s,a)}\\
        \end{align*}
        \STATE Compute $\hat{\ell}_k(\pi)$ for each policy by its trajectory.
        \STATE Update distribution over policies by,
        \begin{align*}
            \rho^{k+1}(\pi) = \frac{\rho^{k}(\pi) e^{-\eta\hat\ell^k(\pi)}}{\sum_{\pi'}  \rho^{k}(\pi') e^{-\eta\hat\ell^k(\pi')}}
        \end{align*}
    \ENDFOR
\end{algorithmic}
\end{algorithm}

\begin{theorem}[Upper bound for B/F]
\label{thm:BF-ub}
\Cref{alg:bandit-full} has a Markov regret of
\[
\R_K^{\mathcal{M}} = O\!\left(\sqrt{H^3 S^2 A K \log A}\right).
\]
\end{theorem}
\begin{proof}
From Hedge guarantee and Jensen inequality:
\begin{align*}
    \R_K^{\mathcal{M}}&\le \frac{\log\roundy{A^{HS}}}{\eta} + \frac{\eta}{2}\sum_k\sum_\pi p_k(\pi)\E_k\squary{\roundy{\sum_{h,s,a}q^\pi_{k,h}(s,a)\hat{\ell}_h^k(s,a)}^2}\\
    &\le \frac{HS\log\roundy{A}}{\eta} + H\frac{\eta}{2}\sum_{k,h}\sum_\pi p_k(\pi)\E_k\squary{\roundy{\sum_{s,a}q^\pi_{k,h}(s,a)\hat{\ell}_h^k(s,a)}^2}\\
\end{align*}

Fix $k,h$, the events $\mathds{1}_k^h\squary{s,a}$ for every $s,a$ are mutually exclusive conditioned on the history up to time $k$, which means that:
\begin{align*}
    \E_k\squary{\roundy{\sum_{s,a}q^\pi_{k,h}(s,a)\hat{\ell}_h^k(s,a)}^2} = \E_k\squary{\sum_{s,a}q^\pi_{k,h}(s,a)^2\hat{\ell}_h^k(s,a)^2} = \sum_{s,a}q^\pi_{k,h}(s,a)^2\E_k\squary{\hat{\ell}_h^k(s,a)^2}
\end{align*}
where the second equality follows from the linearity of the expectation. Thus, we can say:
\begin{align*}
    \R_K^{\mathcal{M}}&\le \frac{HS\log\roundy{A}}{\eta} + H\frac{\eta}{2}\sum_{k,h}\sum_\pi\sum_{s,a}p_k(\pi)q^\pi_{k,h}(s,a)^2\E_k\squary{\hat{\ell}_h^k(s,a)^2}\\
    &= \frac{HS\log\roundy{A}}{\eta} + H\frac{\eta}{2}\sum_{k,h,s,a}\sum_\pi p_k(\pi)q^\pi_{k,h}(s,a)^2\frac{\ell_k^h(s,a)^2}{\sum_\pi p_k(\pi)q_k^\pi(s,a)}\\
    &= \frac{HS\log\roundy{A}}{\eta} + H\frac{\eta}{2}\sum_{k,h,s,a}\ell_k^h(s,a)^2\\
    &\le \frac{HS\log\roundy{A}}{\eta} + \frac{\eta}{2}H^2KSA
\end{align*}
\end{proof}

\subsubsection{Lower Bound}
The following lower bound follows from the same reduction used in
\Cref{thm:FF-lb}, except that each of the $HS$ expert problems is replaced by a bandit problem with $A$ actions. 

\begin{theorem}[Lower bound for B/F]
\label{thm:BF-lb}
Any algorithm in the B/F regime satisfies
\[
\R_K^{\mathcal{M}} \ge \Omega\!\left(\sqrt{S A H^3 K}\right).
\]
\end{theorem}

\begin{proof}
The proof follows a similar structure as \Cref{thm:FF-lb}.
Let $\tilde S := S - A - 1$ and $\tilde H := \lfloor H/2\rfloor$, $\bar H := \tilde H-1\ge 1$ and $
T := \frac{K}{\tilde S\,\bar H}
$.
Assume w.l.o.g.  that
$T$
is an integer (otherwise replace $T$ by $\lfloor K/(\tilde S\bar H)\rfloor$, losing only constant factor in the regret throughout this proof).

We embed $\tilde S\bar H$ independent copies of a hard $A$ armed bandit problem into disjoint
\emph{episode blocks}, and inflate each bandit loss by a factor $\tilde H$ by forcing the incurred loss
to be consistent across half of the episode.

We construct the MDP with set of states,
\[
\mathcal{S}
=
\{s_{\mathrm{init}}\}
 \cup 
\{s_1,\dots,s_{\tilde S}\}
 \cup 
\{\tilde s_1,...,\tilde s_A\}
\]

We create one copy of a hard MAB problem for each pair $(s_i,h)$ with $i\in[\tilde S]$ and $h\in\{2,3,\dots,\tilde H\}$.

Partition the $K$ episodes into $\tilde S\bar H$ consecutive blocks of equal length $T$. Formally, the block associated with the pair $(s_i,h)$ is
\[
\mathcal{K}_{s_i,h}  :=  \{\, bT+1,  bT+2,  \dots,  (b+1)T \,\},
\]
where
$
b = (i-1)\bar H + (h-2)
$.

Fix a pair $(s_i,h)$.
Within this copy, we consider an adversarial losses
$\tilde \ell^{(i,h)}_1,\dots,\tilde \ell^{(i,h)}_T\in\{0,1\}^A$ over $T$ rounds.
We will reduce learning in this block to online learning with bandit feedback on this sequence where the known regret lower bound is $\Omega(\sqrt{T A})$ (see e.g., \citet{slivkins2024introductionmultiarmedbandits}).

We will now describe the transition dynamics within the block $\mathcal{K}_{s_i,h}$. For $k \in \mathcal{K}_{s_i,h}$,

\begin{itemize}

\item From the initial state we transition to  state $s_i$ deterministically:
\[
p^{k}_{1}(s_i \mid s_{\mathrm{init}}, a) = 1
\quad\text{for all }a\in[A].
\]

\item For $h'=2,3,\dots,h-1$,  we stay in $s_i$ deterministically:
\[
p^{k}_{h'}(s_i \mid s_i, a)=1
\quad\text{for all }a\in[A].
\]

\item For step $h$ (the ``bandit decision step''),
 define the step-$h$ transition as:
\[
p^{k}_{h}(\tilde s_{a'} \mid s_i, a)=
\begin{cases}
     1 & \text{if } a = a'
    \\
        0 & \text{if } a \ne a'
\end{cases}
\]
That is, choosing an action $a$ sends the agent to $\tilde s_a$, deterministically.

\item For steps $h' = h+1,\dots,H$, we keep the state fixed until the end of the episode:
\[
p^{k}_{h'}(s\mid s,a)=1\quad\text{for }s\in\{\tilde s_1,...,\tilde s_A\},
\]
\end{itemize}

For the
losses let $t_k := k-bT \in\{1,\dots,T\}$ be the round index of episode $k$ inside block. We define,
\[
\ell_{h'}^k(s,a')
=
\begin{cases}
\tilde \ell^{(i,h)}_{t_k} (a), & \text{if } h'\in\{\tilde H,\tilde H+1,\dots,H\}, s=\tilde s_a \text{ where } a\in\A,\\
0, & \text{otherwise}.
\end{cases}
\]
Hence, if the trajectory reaches $\tilde s_a$ by time $h\le \tilde H$, then the episode incurs loss
exactly $(H-\tilde H+1) \ell^{(i,h)}_{t_k} (a)= \Theta(H\ell^{(i,h)}_{t_k} (a))$. Importantly, the learner does not observe $\ell^{(i,h)}_{t_k} (a')$ if it did not reach $s_{a'}$ (i.e., the feedback is bandit).

Combining this with the transition definition, for episode $k\in\mathcal{K}_{s_i,h}$, the episode loss is $\Theta( H \cdot \ell^{(i,h)}_{t_k}(a_k))$ where $a_k$ is the action that the learner choose in episode $k$ in $s$ at step $h$. This equivalent to scaling the $A$-armed bandit loss $\ell^{(i,h)}_{t_k}(a_k)$ by a factor $\Theta(H)$, and thus the regret in these rounds is $\Omega(H\sqrt{T A})$.
Summing over the $\tilde S\bar H = \Theta(HS)$ blocks we get that the total regret is $\Omega(H^2 S \sqrt{T A}) = \Omega(  \sqrt{H^3 S K A})$ as desired.
\end{proof}

Thus the B/F regime still admits polynomial regret and does not explain the exponential hardness of
adversarial MDPs. Notice that this lower bound has an extra $\sqrt{H}$ over the classical lower bound of stationary MDP with bandit feedback, which comes from the fact that they can create only $S$ parallel bandit problems (with loss scale of $\Theta(H)$), while we use the adversariality of the dynamics to construct $HS$ such problems. We additionally note that there is a gap of $\sqrt{S}$ between our lower and upper bound in this regime, a gap that exists also in the stationary MDP with bandit feedback problem.

\subsection{B/B: Bandit losses and bandit dynamics.}
With bandit feedback on both losses and transition functions, the learner observes only trajectories.  
\citet{pmlr-v139-tian21b} proved that $\R_K^{\mathcal{M}} \ge \Omega\!\left(\sqrt{2^H K}\right)$ for $A = 2$ and $S=2$,
 revealing exponential dependence on $H$.  
We extend this lower bound for general $A$ and $S$ and provide a matching upper bound.

\subsubsection{Upper bound}
For simplicity, we assume the adversarial transitions are deterministic, but the proof follows for the case that the transitions are stochastic.
\begin{algorithm}[tb]
\caption{Bandit/Bandit Policy-based EXP4 (\texttt{BB-Pb-EXP4})}
\label{alg:bandit-bandit}
\begin{algorithmic}
    \STATE \textbf{Initialization:} Set $\rho^1$ to be the uniform distribution over all deterministic Markovian policies
    \FOR{$k = 1,...,K$}
        \STATE Sample a policy $\pi^k \sim \rho^k$, execute it and observe $\{(s_h^k, a_h^k, \ell_h^k)\}_{h=1}^H$
        \STATE Compute estimated policy loss for each $\pi \in \Pi_{\textsf{det}}$ by,
        \begin{align*}
            \hat \ell^k (\pi) = \frac{\sum_{h=1}^H \ell_h^k(s_h^k ,a_h^k ) \onebb\{ \pi \in \Pi_k \}}{\sum_{\pi \in \Pi_k}\rho^k(\pi)}
        \end{align*}
        where $\Pi_k = \{ \pi \in  \Pi_{\textsf{det}} \mid \forall h \in [H]:
        \pi(s_h^k) = \pi_h^k(s_h^k) \}$
        \STATE Update distribution over policies by,
        \begin{align*}
            \rho^{k+1}(\pi) = \frac{\rho^{k}(\pi) e^{-\eta\hat\ell^k(\pi)}}{\sum_{\pi'}  \rho^{k}(\pi') e^{-\eta\hat\ell^k(\pi')}}
        \end{align*}
    \ENDFOR
\end{algorithmic}
\end{algorithm}
\begin{theorem}[Upper bound for B/B]
\label{thm:BB-ub}
\Cref{alg:bandit-bandit} has a Markov regret of
\[
\R_K^{\mathcal{M}} = \tilde{O}\!\left(\sqrt{A^H S K}\right).
\]
\end{theorem}
\begin{proof}
For $\vec{a}=(a_1,a_2,\dots,a_H)$, denote $\Pi_k^{\vec{a}}$ the set of policies that will play $\vec{a}$ in episode $k$. One can see that for every $\pi\in\Pi_k^{\vec{a}}$, $\hat\ell^k(\pi)$ is the  same. Thus, we can denote this as $\hat\ell^k(\vec{a})$.

For very $\vec{a}$:
\begin{align*}
    \E\squary{\hat\ell^k(\vec{a})} = \frac{\sum_{h=1}^H \ell_h^k(s_h^k ,a_h^k ) \E\squary{\onebb\{ \pi \in \Pi_k \}}}{\sum_{\pi \in \Pi_k}\rho^k(\pi)} = \sum_{h=1}^H \ell_h^k(s_h^k ,a_h^k )
\end{align*}

From Hedge guarantee:
\begin{align*}
    \R_K^{\mathcal{M}}&\le \frac{\log \roundy{A^{HS}}}{\eta} + \frac{\eta}{2}\sum_k\sum_\pi p(\pi)\E\squary{\hat{\ell}_k(\pi)^2} \\
    &= \frac{HS\log(A)}{\eta} + \frac{\eta}{2}\sum_k\sum_{\vec{a}}\E\squary{\hat\ell_k(\vec{a})^2}\roundy{\sum_{\pi'\in\Pi_k^{\vec{a}}}p(\pi')} \\
    &= \frac{HS\log(A)}{\eta} + \frac{\eta}{2}\sum_k\sum_{\vec{a}}\frac{\ell_k(\vec{a})^2}{\roundy{\sum_{\pi'\in\Pi_k^{\vec{a}}}p(\pi')}}\roundy{\sum_{\pi'\in\Pi_k^{\vec{a}}}p(\pi')} \\
    &= \frac{HS\log(A)}{\eta} + \frac{\eta}{2}\sum_k\sum_{\vec{a}}\ell_k(\vec{a})^2 \\
    &\le \frac{HS\log(A)}{\eta} + \frac{\eta}{2}KA^HH^2
\end{align*}
\end{proof}

\subsubsection{Lower Bound}
\begin{lemma}\label{lem:small_S}
There exists an instance with $S=2$ states, $A$ actions, horizon $H$, and $K$ episodes such that any
algorithm incurs regret
\[
    \R_K^{\mathcal{M}}= \Omega\!\left(\sqrt{A^H K}\right).
\]
\end{lemma}
\begin{proof}
We reduce from a standard hard instance for the $M$-armed bandit problem with $M=A^H$ arms
(e.g., \cite{slivkins2024introductionmultiarmedbandits}).
In that construction, one arm has losses drawn from $\mathrm{Ber}(1/2-\epsilon)$ and all other arms
have losses drawn from $\mathrm{Ber}(1/2)$, where $\epsilon = \Theta(\sqrt{M/K})$.
It is shown that for any algorithm there exists a choice of the optimal arm such that the expected
regret is $\Omega(\sqrt{M K})$.

We encode each arm as a length-$H$ action sequence.
Specifically, fix a sequence $\vec{a}=(a_1,\ldots,a_H)\in\A^H$, chosen uniformly at random.
We construct an MDP such that the learner receives a loss drawn from $\mathrm{Ber}(1/2-\epsilon)$ if
and only if it plays the exact action sequence $\vec{a}$, and otherwise receives loss
$\mathrm{Ber}(1/2)$.
Thus, identifying the optimal policy is equivalent to identifying the optimal bandit arm.

The MDP has two states, denoted $s_1$ and $s_2$.
For each episode $k$ and step $h$, one of these states is designated uniformly at random as the
\emph{good} state $s^G_{k,h}$ and the other as the \emph{bad} state $s^B_{k,h}$.
The transition dynamics are defined as follows: for every episode $k$, step $h$, and action $a$,
\begin{align*}
    p_h^k(s^G_{k,h+1} \mid s^G_{k,h}, a)
    &=
    \begin{cases}
        1 & \text{if } a = a_h, \\
        0 & \text{otherwise},
    \end{cases} \\
    p_h^k(s^B_{k,h+1} \mid s^G_{k,h}, a)
    &= 1 - p_h^k(s^G_{k,h+1} \mid s^G_{k,h}, a), \\
    p_h^k(s^B_{k,h+1} \mid s^B_{k,h}, a)
    &= 1.
\end{align*}
That is, the learner remains in the good state only if it selects the correct action at every step;
upon choosing any incorrect action, it transitions to the bad state and remains there for the rest
of the episode.

Crucially, since the identity of the good state is resampled uniformly at each $(k,h)$, the learner
cannot determine whether it is currently in the good or bad state.
As a result, the state observations provide no information about partial correctness of the action
sequence.
Only the terminal loss reveals information, and this feedback depends solely on whether the entire
action sequence matches $\vec{a}$.

Therefore, the learning problem reduces to a bandit problem with $A^H$ arms and bandit feedback.
By the bandit lower bound, the regret is $\Omega(\sqrt{A^H K})$, completing the proof.
\end{proof}

\begin{theorem}[Lower bound for B/B]
\label{thm:BB-lb}
Any algorithm in the B/B regime satisfies
\[
\R_K^{\mathcal{M}} \ge \Omega\!\left(\sqrt{A^H S K}\right).
\]
\end{theorem}
\begin{proof}
We reduce to the hard instance constructed in \Cref{lem:small_S}.
Partition the state space into $S/2$ disjoint pairs of states.
For each pair, we embed an independent copy of the two-state MDP from
\Cref{lem:small_S}, and assign it a disjoint subset of $2K/S$ episodes.

In each copy, any algorithm incurs regret
$\Omega\!\left(\sqrt{A^H \cdot (2K/S)}\right)$.
Since the $S/2$ instances are independent and the learner receives no information that couples them,
the total regret is the sum over all copies, yielding
\[
    \frac{S}{2} \cdot \Omega\!\left(\sqrt{A^H \cdot \frac{2K}{S}}\right)
    = \Omega\!\left(\sqrt{S A^H K}\right).
\]
\end{proof}

The same construction used to prove \Cref{thm:BB-lb} also implies a lower bound for the F/B regime (full-information losses but bandit dynamics) of,
\[
\R_K^{\mathcal{M}} \ge \Omega\!\left(\sqrt{A^{H-1} S K}\right),
\]
at the small price of replacing $H$ by $H-1$. 



\section{History Dependent Regret}\label{apx:action_based}



\subsection{General definitions}
$\C_h$ is the set of conditions in step $h$. That is, a tuple $(s,a,s')$ for each step in ${\mathbf \Lambda}_h$. It contains only consistent conditions, namely if $h',h'+1\in\mathbf{\Lambda}$, the relevant tuples isn't contradictive. Formally, $(c_{h'})_3 = (c_{h'+1})_1$.

We denote by $\calT_h$ the set of all trajectories of $h$ steps. We denote by $\calT_c$ for $c\in \C_h$ to be all the trajectories that are possible for condition $c$ - namely, for every $\tilde{h}\in {\mathbf \Lambda}$ such that $\tilde{h}\le h$, let $(s,a,s')$ be the condition tuple for $c$ for $\tilde{h}$, then the trajectory must contains $s,a,s'$ in the corresponding steps. Notice that those trajectories are disjoint and $\bigcup_{c\in\C_h}\calT_c = \calT_h$. Finally, we denote the possible states for condition $c$ at step $h$ as $S_h^c$ (it is all steps if $h$ and $h-1$ are stochastic and the state in $c$ otherwise).

\begin{definition}\label{def:com}
\begin{align*}
    \mu_h^{p^{\stat},\pi}(s_h,a_h,c) = \sum_{\curly{s_{h'},a_{h'}}_{h'=1}^{h-1}\in \calT_c}\prod_{h'=1}^{h}\pi_h(a_{h'}\mid s_1,\dots,s_{h'}, a_1,\dots,a_{h'})\prod_{h'\notin{\mathbf \Lambda}_h}p_{h'}(s_{h'+1}\mid s_{h'},a_{h'}) 
\end{align*}

And the conditioned version:
\begin{align*}
    \mu_h^{p^{\stat},\pi}(s_h,a_h,c\mid s_{\hat{h}},a_{\hat{h}}) = \sum_{\curly{s_{h'},a_{h'}}_{h'=\hat{h}+1}^{h-1}\in \calT_c}\prod_{h'=\hat{h}+1}^{h}\pi_h(a_{h'}\mid s_1,\dots,s_{h'}, a_1,\dots,a_{h'})\prod_{h'\notin{\mathbf \Lambda}_h,h'>\hat{h}}p_{h'}(s_{h'+1}\mid s_{h'},a_{h'}) 
\end{align*}

Additionally, we abuse the notation and use:
\begin{align*}
    \mu_h^{p^{\stat},\pi}(s_h,a_h,s',c) = \begin{cases}
        \mu_h^{p^{\stat},\pi}(s,a,c) & h\in\mathbf{\Lambda}\\
        \mu_h^{p^{\stat},\pi}(s,a,c)p(s'\mid s,a) & h\notin\mathbf{\Lambda}\\
    \end{cases}
\end{align*}

\end{definition}

\begin{definition}
Fix some $h$ and  some $c\in\C_h$. For every $\tilde{h}\in {\mathbf \Lambda}_h$, denote $s_{\tilde{h}},a_{\tilde{h}},s'_{\tilde{h}}$ to be the triplet of $\tilde{h}$ in $c$. We have:
\begin{align*}
    \varrho^{p}(c) = \prod_{\tilde{h}\in{\mathbf \Lambda}_h}p_h(s'_{\tilde{h}}\mid s_{\tilde{h}},a_{\tilde{h}})
\end{align*}
We denote $\varrho_k(c) = \varrho^{p^k}(c)$
\end{definition}

\begin{definition}\label{def:polytope}
Given confidence radiuses and empirical transition  $\epsilon_{h}(s,a,s'), \bar p_{h}(s'\mid s,a)$ ($ h\in[H],\; s,s'\in\S,\; a\in\A$), define the Polytope $\Delta(\{\epsilon_h,\bar p_h\}_{h=1}^H)$ where $\{\hat\mu_h\}_{h=1}^H \in \Delta(\{\epsilon_h,\bar p_h\}_{h=1}^H)$ if and only if,

\begin{align}
     &\sum_{a}\hat{\mu}_{h+1}(s,a,c)=\sum_{s',a'}\hat{\mu}_{h}(s',a',s,c) 
        \qquad\qquad\qquad\qquad\ \forall h\in \curly{h\notin {\mathbf \Lambda}\mid h+1\in{\mathbf \Lambda}},\; s\in\S,\;c\in \C_h \label{eq:polytope_first}
        \\
	&\sum_{a,s'}\hat{\mu}_{h+1}(s,a,s',c)=\sum_{s',a'}\hat{\mu}_{h}(s',a',s,c) 
        \qquad\qquad\qquad\qquad\ \forall h\in \curly{h\notin {\mathbf \Lambda}\mid h+1\notin{\mathbf \Lambda}},\; s\in\S,\;c\in \C_h \label{eq:polytope_second}
    \\
    & \sum_{a} \hat{\mu}_{h + 1} (s,a,c\,\|\,(\tilde s, \tilde a,s)) = \hat{\mu}_{ h}(\tilde s, \tilde a, c) 
        \qquad\qquad\qquad\qquad\ \forall  h\in \curly{h\in {\mathbf \Lambda}\mid h+1\in{\mathbf \Lambda}},\; s,\tilde s\in\S, \; \tilde a\in \A,\;c\in C_h\label{eq:polytope_third}
    \\
    & \sum_{a,s'} \hat{\mu}_{ h + 1} (s,a,s',c\,\|\,(\tilde s, \tilde a,s)) = \hat{\mu}_{ h}(\tilde s, \tilde a, c) 
        \qquad\qquad\qquad\qquad\ \forall  h\in \curly{h\in {\mathbf \Lambda}\mid h+1\notin{\mathbf \Lambda}},\; s,\tilde s\in\S, \; \tilde a\in \A,\;c\in C_h      \label{eq:polytope_fourth}
    \\
    & \begin{cases}
        \sum_{s,a}  \hat{\mu}_{1} (s,a,()) = 1 & 1\in\mathbf{\Lambda}\\
        \sum_{s,a,s'}  \hat{\mu}_{1} (s,a,s',()) = 1 & 1\notin\mathbf{\Lambda}
    \end{cases} 
        \label{eq:polytope_fifth}
    \\
    & 
    \begin{aligned}
        \big| \hat{\mu}_{h}(s,a,s',c)-\sum_{s''}  \hat{\mu}_{h}(s,a,s'',c) & \bar{p}(s'\mid s,a) \big|
        \\
        & \leq\sum_{s''}\hat{\mu}_{h}(s,a,s'',c)\epsilon_{h}(s,a,s')
    \end{aligned}
            \qquad\qquad\qquad\qquad\ \forall h\notin\mathbf{\Lambda}, \; s,s'\in\S,\; a\in\A,\; c\in\mathcal{C}_h \label{eq:polytope_sixth}\\
    & \hat{\mu}_h(s,a,c\| (\tilde{s},\tilde{a},s')) = 0 & \forall \curly{h\in\mathbf{\Lambda}\mid h-1 \in\mathbf{\Lambda}}, s\ne s'\in \S, a\in \A, c\in \C_{h-1} \label{eq:polytope_seventh}\\
    & \hat{\mu}_h(s,a,s'',c\| (\tilde{s},\tilde{a},s')) = 0 & \forall \curly{h\notin\mathbf{\Lambda}\mid h-1 \in\mathbf{\Lambda}}, s\ne s'\in \S,s''\in \S, a\in \A, c\in \C_{h-1} \label{eq:polytope_eighth}\\
    & \begin{cases}
        \hat{\mu}_{1} (s,a,()) = 0 & 1\in\mathbf{\Lambda}\\
        \hat{\mu}_{1} (s,a,s',()) = 0 & 1\notin\mathbf{\Lambda}
    \end{cases} \qquad\qquad\qquad\qquad\ \forall s\ne s_{init},a\in \A, s'\in \S
        \label{eq:polytope_ninth}
\end{align}
In \Cref{eq:polytope_first,eq:polytope_second} equation $\C_h=C_{h+1}$ (since $h\notin{\mathbf \Lambda}$), and thus it is well defined.

We may sometimes refer to a member in this polytope as $\hat{\mu}_h(s,a,c)$ even when $h\notin \mathbf{\Lambda}$. In that case we denote $\hat{\mu}_h(s,a,c) = \sum_{s'}\hat{\mu}_h(s,a,s',c)$.
\end{definition}

\begin{definition}\label{def:initialization}
The initialization is defined as:

For $h=1$:
\begin{align*}
 \begin{cases}
    \hat{\mu}^1_1(s_{init},a,()) = \frac{1}{A} & 1\in\mathbf{\Lambda}\\
    \hat{\mu}^1_h(s_{init},a,s',()) = \frac{1}{SA} & 1\notin\mathbf{\Lambda}
\end{cases}   
\end{align*}

If $h-1\notin\mathbf{\Lambda}$:
\begin{align*}
\begin{cases}
    \hat{\mu}^1_h(s,a,c) = \frac{S^{\Lambda_h}}{SAC_h} & h\in\mathbf{\Lambda}\\
    \hat{\mu}^1_h(s,a,s',c) = \frac{S^{\Lambda_h}}{S^2AC_h} & h\notin\mathbf{\Lambda}
\end{cases}
\end{align*}
If $h-1\in\mathbf{\Lambda}$:
\begin{align*}
\begin{cases}
    \hat{\mu}^1_h(s,a,c) = \frac{S^{\Lambda_h}}{AC_h} & h\in\mathbf{\Lambda}, \roundy{c_{h-1}}_3 = s\\
    \hat{\mu}^1_h(s,a,s',c) = \frac{S^{\Lambda_h}}{SAC_h} & h\notin\mathbf{\Lambda}, \roundy{c_{h-1}}_3 = s \\
    \hat{\mu}^1_h(s,a,c) = 0 & h\in\mathbf{\Lambda}, \roundy{c_{h-1}}_3 \ne s
    \\
    \hat{\mu}^1_h(s,a,s',c) = 0 & h\notin\mathbf{\Lambda}, \roundy{c_{h-1}}_3 \ne s
\end{cases}
\end{align*}
\end{definition}

\subsection{Good event}
\begin{definition}\label{def:good_event}
The event $G_1$ - for every $s,a,s',h\notin\mathbf{\Lambda},k$:
\begin{align}\label{eq:G1}
    \abs{p_h(s' \mid s,a) - \bar{p}_h^k(s'\mid s,a)} \le \epsilon_h(s,a,s')
\end{align}

The event $G_2$ - for every $h$: 
\begin{align}\label{eq:G2}
    \sum_{k,s,a,c}\ell_h^k(s,a)\squary{\hat{\ell}_h^k(s,a,c) - \varrho_k(c)\ell_h^k(s,a)} \le \frac{1}{2\gamma}\ln\roundy{\frac{H}{\delta}}
\end{align}

The event $G_3$ - for every $h$: 
\begin{align}\label{eq:G3}
    \sum_{k,s,a,c}\mu^*_h(s,a,c)\squary{\hat{\ell}_h^k(s,a,c) - \varrho_k(c)\ell_h^k(s,a)} \le \frac{1}{2\gamma}\ln\roundy{\frac{H}{\delta}}
\end{align}

The event $G_4$ - 
\begin{align}\label{eq:G4}
    \sum_{k,h,s,a,c}\hat{\mu}_h^k(s,a,c)\roundy{\hat{\ell}_h^k(s,a,c) - \E\squary{\hat{\ell}_h^k(s,a,c)}} &\le H\sqrt{2K\ln\roundy{\frac{1}{\delta}}}
\end{align}

The event $G_5$ - 
\begin{align}\label{eq:G5}
    \max_h\sum_{k,s,a}\frac{q_h^k(s,a)}{\max\curly{1, N_k(s,a)}} &\le SA\ln\roundy{K} + \ln\roundy{\frac{H}{\delta}}\\
    \max_h\sum_{k,s,a}\frac{q_h^k(s,a)}{\sqrt{\max\curly{1, N_k(s,a)}}} &\le \sqrt{SAK} + HSA\ln\roundy{K} + \ln\roundy{\frac{H}{\delta}}
\end{align}

The intersection good event $G$ - 
\begin{align*}
    G = G_1 \cap G_2 \cap G_3 \cap G_4 \cap G_5
\end{align*}
\end{definition}

\begin{lemma}[Adapted from Lemma 11 of \cite{jin2019learning}]\label{lem:neu_stuff}
Let $\curly{A_k}_{k=1}^K$ be a set of $\mathcal{F}_{k+1}$-measurable random variables taking values from some finite set $\mathcal{A}$. For every $a\in \A$, let $p_k(a)$ be a $\mathcal{F}_k$-measurable random variable such that $\Pr\squary{A_k=a \mid \calF_k}=p_k(a)$ and let $\bar{p}_k(a)$ be an upper bound for $p_k(a)$. Additionally, let $\curly{\alpha_k(a)}$ for every $a\in\A$ and $k$ be a ${\cal F}_k$-measurable random variable in $[0,1]$. For every $\gamma > 0$, with probability $1-\delta$:
\begin{align*}
    \sum_k\sum_{a\in \A}\alpha_k(a)\roundy{\frac{\mathds{1}_k\squary{A_k=a}}{\bar{p}_k(a) + \gamma} - \frac{p_k(a)}{\bar{p}_k(a)}} \le \frac{1}{\gamma}\ln\roundy{\frac{1}{\delta}}
\end{align*}
\end{lemma}

\begin{lemma}\label{lem:hp_estimator_bound}
Let $\curly{\alpha_k(s,a,c)}_{k=1}^K$ be a set of ${\cal F}_k$-measurable random variables taking values in $[0,1]$. W.p $1-5\delta$, for every $h\in[H]$:
\begin{align*}
    \sum_{k,s,a,c}\alpha_k(s,a,c)\roundy{\hat{\ell}_h^k(s,a,c) - \varrho_k(c)\ell_h^k(s,a))} &\le \frac{1}{2\gamma}\ln\roundy{\frac{H}{\delta}}
\end{align*}
\end{lemma}
\begin{proof}
Fix $h\in [H]$. Let $A$ be the random variable of the values of $s_h,a_h,c_h$ at time $k$. We have that:
\begin{align*}
    \Pr\squary{s_h^k=s,a_h^k=a,c_h^k=c} = \mu_h^k(s,a,c)\varrho_k(c)
\end{align*}
 From \Cref{lem:neu_stuff} with $\bar{p}_k(s,a,c) = \mu_h^k(s,a,c)$, w.p $1-\frac{\delta}{H}$:
\begin{align*}
    \frac{1}{2\gamma}\ln\roundy{\frac{H}{\delta}} &\ge \sum_{k,s,a,c}\ell_h^k(s,a)\alpha_k(s,a,c)\roundy{\frac{\mathds{1}_k\squary{{s_h=s,a_h=a,c_h=c}}}{\mu_h^k(s,a,c) + \gamma} - \varrho_k(c)} \\
\end{align*}
With union bound we get that the above is true for all $h\in[H]$ w.p $1-\delta$.

From \Cref{lem:real_in_polytope} we have that w.p $1-4\delta$ that the real dynamics are in the polytope and $\mu_h^{p_k,\pi_k}(s,a,c) \le u_h^k(s,a,c)$. 
Union bound that, we get w.p $1-5\delta$:
\begin{align*}
    \frac{1}{2\gamma}\ln\roundy{\frac{H}{\delta}} &\ge \sum_{k,s,a,c}\ell_h^k(s,a)\alpha_k(s,a,c)\roundy{\frac{\mathds{1}_k\squary{{s_h=s,a_h=a,c_h=c}}}{u_h^k(s,a,c) + \gamma} - \varrho_k(c)}\\
    &= \sum_{k,s,a,c}\alpha_k(s,a,c)\roundy{\hat{\ell}_h^k(s,a,c)- \varrho_k(c)\ell_h^k(s,a)}
\end{align*}
\end{proof}

\begin{lemma}\label{lem:good_event}
\begin{align*}
    \Pr[G] \ge 1 - 9\delta
\end{align*}
\end{lemma}
\begin{proof}
$G_1$ is true w.p at least $1-4\delta$ from Lemma 2 of \cite{jin2019learning}.

$G_2$ and $G_3$ (conditioned on $G_1$) is true w.p at least $1-\delta$ (each) from \Cref{lem:hp_estimator_bound}.

We'll now prove $G_4$ is true w.p at least $1-\delta$. We have for every $k\in[K]$
\begin{align*}
    \sum_{h,s,a,c}\hat{\mu}_h^k(s,a,c)\hat{\ell}_h^k(s,a,c)&= \sum_{h,s,a,c}\frac{\hat{\mu}_h^k(s,a,c)\mathds{1}_k\squary{s,a,c}}{u_h^k(s,a,c)+\gamma}\\
    &\le \sum_{h,s,a,c}\mathds{1}_k\squary{s,a,c}\\
    &= H
\end{align*}
Thus, from Hoeffding-Azuma inequality w.p $1-\delta$:
\begin{align*}
    \sum_{k,h,s,a,c}\hat{\mu}_h^k(s,a,c)\roundy{\hat{\ell}_h^k(s,a,c) - \E\squary{\hat{\ell}_h^k(s,a,c)}} &\le H\sqrt{2K\ln\roundy{\frac{1}{\delta}}}
\end{align*}

$G_5$ is true w.p at least $1-2\delta$ from Lemma 10 of \cite{jin2019learning}.
\end{proof}

\subsection{Polytope/COM properties}
\begin{lemma}\label{lem:com_to_om}
\begin{align*}
    q_h^{p,\pi}(s,a) = \sum_{c\in \C_h}\mu_h^{p,\pi}(s,a,c)\varrho^p(c)
\end{align*}
\end{lemma}
\begin{proof}
  \begin{align*}
    q_h^{p,\pi}(s,a) &= \sum_{\curly{s_{h'},a_{h'}}_{h'=1}^{h-1}\in \calT_h}\prod_{h'=1}^{h}\pi(a_{h'}\mid s_1,\dots,s_{h'}, a_1,\dots,a_{h'})\prod_{h'=1}^hp(s_{h'+1}\mid s_{h'},a_{h'}) \\
    &= \sum_{c\in\C_h}\sum_{\curly{s_{h'},a_{h'}}_{h'=1}^{h-1}\in \calT_c}\prod_{h'=1}^{h}\pi(a_{h'}\mid s_1,\dots,s_{h'}, a_1,\dots,a_{h'})\prod_{h'=1}^hp(s_{h'+1}\mid s_{h'},a_{h'})\\
    &= \sum_{c\in\C_h}\sum_{\curly{s_{h'},a_{h'}}_{h'=1}^{h-1}\in \calT_c}\prod_{h'=1}^{h}\pi(a_{h'}\mid s_1,\dots,s_{h'}, a_1,\dots,a_{h'})\prod_{h\notin{\mathbf \Lambda}_h}p(s_{h+1}\mid s_h,a_h) \varrho^p(c)\\
    &= \sum_{c\in \C_h}\mu_h^{p,\pi}(s,a,c)\varrho^p(c)
\end{align*}  
\end{proof}

\begin{lemma}[Restatement of \Cref{lem:com_best}]
For every history-dependent policy $\pi$ there exists a COM-induced policy $\pi'$ such that:
\begin{align*}
    V_{1}^{\pi} = V_{1}^{\pi'}
\end{align*}
\end{lemma}
\begin{proof}
Fix some MDP with transition $p$ and losses $\curly{\ell}$. Fix some history-dependent $\pi$, and let $\pi'$ be the policy induced from $\mu^{p,\pi}$. 

From \Cref{lem:com_to_om}:
\begin{align*}
    V_{1}^\pi &= \sum_{h,s,a} q_h^{p,\pi}(s,a)\,\ell(s,a)
    \\&= \sum_{h,s,a,c} \mu_h^{p,\pi}(s,a,c)\,\varrho_h^p(c)\,\ell(s,a)\\
    &= \sum_{h,s,a,c} \mu_h^{p,\pi'}(s,a,c)\,\varrho_h^p(c)\,\ell(s,a)\\
    &= \sum_{h,s,a} q_h^{p,\pi'}(s,a)\,\ell(s,a)\\
    &= V_{1}^{\pi'}
\end{align*}
\end{proof}

\begin{lemma}
$\hat{\mu}^1$ defined in \Cref{def:initialization} is in the polytope defined in \Cref{def:polytope}.
\end{lemma}
\begin{proof}
Notice that in all cases of \Cref{eq:polytope_first,eq:polytope_second} we have $h\notin\mathbf{\Lambda}$ which means $C_{h+1} = C_h$ and $\Lambda_{h+1} = \Lambda_h$.

\Cref{eq:polytope_first}, if $h=1$ (which means $1\notin\mathbf{\Lambda}$):
\begin{align*}
    \sum_{a}\hat{\mu}_{h+1}(s,a,c) &= \sum_a \frac{S^{\Lambda_{h+1}}}{SAC_{h+1}}\\
    &= \sum_a \frac{1}{SA}     \tag{since $C_{h+1} = 1$ and $\tilde H_{h+1} = 0$} \\ 
    &= \sum_a \hat{\mu}^1_{h}(s_{init},a,s,())\\
    &= \sum_{a,s'} \hat{\mu}^1_{h}(s',a,s,c) \tag{$\hat{\mu}^1_{h}(s',a,s,c) = 0$ for $s'\ne s_{init}$}
\end{align*}

\Cref{eq:polytope_first}, if $h-1\in\mathbf{\Lambda}$:
\begin{align*}
    \sum_{a}\hat{\mu}_{h+1}(s,a,c) &= \sum_a \frac{S^{\Lambda_{h+1}}}{SAC_{h+1}}\\
    &= \sum_a \frac{S^{\Lambda_{h}}}{SAC_{h}}     \\
    &= \sum_a \hat{\mu}_{h}((c_{h-1})_3,a,s,c)           \\
    &= \sum_{a,s'} \hat{\mu}_{h}(s',a,s,c)
\end{align*}

\Cref{eq:polytope_first}, if $h-1\notin\mathbf{\Lambda}$:
\begin{align*}
    \sum_{a}\hat{\mu}_{h+1}(s,a,c) &= \sum_a \frac{S^{\Lambda_{h+1}}}{SAC_{h+1}}\\
    &= \sum_a \frac{S^{\Lambda_{h}}}{SAC_{h}}\\
    &= \sum_{a,s'} \frac{S^{\Lambda_h}}{S^2AC_{h}} \\
    &= \sum_{a,s'} \hat{\mu}_{h}(s',a,s,c)
\end{align*}

\Cref{eq:polytope_second}, if $h=1$ (which means $1\notin\mathbf{\Lambda}$):
\begin{align*}
    \sum_{a}\hat{\mu}_{h+1}(s,a,s',c) &= \sum_{a,s'} \frac{S^{\Lambda_{h+1}}}{S^2AC_{h+1}}\\
    &= \sum_{a,s'} \frac{1}{S^2A}\\
    &= \sum_{a} \frac{1}{SA}\\
    &= \sum_a \hat{\mu}_{h}(s_{init},a,s,())\\
    &= \sum_{a,s'} \hat{\mu}_{h}(s',a,s,c)
\end{align*}

\Cref{eq:polytope_second}, if $h-1\in\mathbf{\Lambda}$:
\begin{align*}
    \sum_{a,s'}\hat{\mu}_{h+1}(s,a,s',c) &= \sum_{a,s'} \frac{S^{\Lambda_{h+1}}}{S^2AC_{h+1}}\\
    &= \sum_{a,s'} \frac{S^{\Lambda_{h}}}{S^2AC_{h}}\\
    &= \sum_{a} \frac{S^{\Lambda_{h}}}{SAC_{h}}\\
    &= \sum_a \hat{\mu}_{h}((c_{h-1})_3,a,s,c)\\
    &= \sum_{a,s'} \hat{\mu}_{h}(s',a,s,c)
\end{align*}

\Cref{eq:polytope_second}, if $h-1\notin\mathbf{\Lambda}$:
\begin{align*}
    \sum_{a,s'}\hat{\mu}_{h+1}(s,a,s',c) &= \sum_{a,s'} \frac{S^{\Lambda_{h+1}}}{S^2AC_{h+1}}\\
    &=\sum_{a,s'} \frac{S^{\Lambda_{h}}}{S^2AC_{h}}\\
    &= \sum_{a,s'} \hat{\mu}_{h}(s',a,s,c)
\end{align*}

Notice that in all cases of \Cref{eq:polytope_third,eq:polytope_fourth} we have $h\in\mathbf{\Lambda}$ which means $\Lambda_{h+1} = \Lambda_h+1$.

\Cref{eq:polytope_third}, start with the case that $h-1\in\mathbf{\Lambda}$. Notice that this means that $C_{h+1} = SAC_h$ because there are two consecutive adversarial steps, which means that $\roundy{c_{h-1}}_3=\roundy{c_{h}}_1$ to make the condition consistent.
\begin{align*}
    \sum_a \hat{\mu}_{h+1}(s,a,c \|(\tilde{s},\tilde{a},s)) &= \sum_a\frac{S^{\Lambda_{h+1}}}{AC_{h+1}}\\
    &= \frac{S^{\Lambda_{h+1}}}{C_{h+1}}\\
    &= \frac{S^{\Lambda_h+1}}{SAC_h}\\
    &= \frac{S^{\Lambda_h}}{AC_h}\\
    &= \hat{\mu}_h(\tilde{s},\tilde{a},c)
\end{align*}

\Cref{eq:polytope_third}, if $h-1\notin\mathbf{\Lambda}$ (or $h=1$). In this case $C_{h+1} = S^2AC_h$.
\begin{align*}
    \sum_a \hat{\mu}_{h+1}(s,a,c \|(\tilde{s},\tilde{a},s)) &= \sum_a\frac{S^{\Lambda_{h+1}}}{AC_{h+1}}\\
    &= \frac{S^{\Lambda_{h+1}}}{C_{h+1}}\\
    &= \frac{S^{\Lambda_h+1}}{S^2AC_h}\\
    &= \frac{S^{\Lambda_h}}{SAC_h}\\
    &= \hat{\mu}_h(\tilde{s},\tilde{a},c)
\end{align*}

\Cref{eq:polytope_fourth}, if $h-1\in\mathbf{\Lambda}$ again we have $C_{h+1} = SAC_h$.
\begin{align*}
    \sum_{a,s'} \hat{\mu}_{h+1}(s,a,s',c \|(\tilde{s},\tilde{a},s)) &= \sum_{a,s'}\frac{S^{\Lambda_{h+1}}}{SAC_{h+1}}\\
    &= \frac{S^{\Lambda_{h+1}}}{C_{h+1}}\\
    &= \frac{S^{\Lambda_h+1}}{SAC_h}\\
    &= \frac{S^{\Lambda_h}}{AC_h}\\
    &= \hat{\mu}_h(\tilde{s},\tilde{a},c)
\end{align*}

\Cref{eq:polytope_fourth}, if $h-1\notin\mathbf{\Lambda}$ (or $h=1$). In this case $C_{h+1} = S^2AC_h$.
\begin{align*}
    \sum_{a,s'} \hat{\mu}_{h+1}(s,a,c \|(\tilde{s},\tilde{a},s)) &= \sum_{a,s'}\frac{S^{\Lambda_{h+1}}}{SAC_{h+1}}\\
    &= \frac{S^{\Lambda_{h+1}}}{C_{h+1}}\\
    &= \frac{S^{\Lambda_h+1}}{S^2AC_h}\\
    &= \frac{S^{\Lambda_h}}{SAC_h}\\
    &= \hat{\mu}_h(\tilde{s},\tilde{a},c)
\end{align*}

\Cref{eq:polytope_fifth}, if $1\in\mathbf{\Lambda}$:
\begin{align*}
    \sum_{s,a}\hat{\mu}_1(s,a,()) = \sum_{s,a}\frac{1}{SA} = 1
\end{align*}

\Cref{eq:polytope_fifth}, if $1\notin\mathbf{\Lambda}$:
\begin{align*}
    \sum_{s,a,s'}\hat{\mu}_1(s,a,s',()) = \sum_{s,a,s'}\frac{1}{S^2A} = 1
\end{align*}

\Cref{eq:polytope_sixth}, notice that $\epsilon_h(s,a,s') = 1$ and $\bar{p}(s'\mid s,a) = \frac{1}{S}$.
\begin{align*}
    \hat{\mu}_h(s,a,s',c) &= \sum_{s''}\hat{\mu}_h(s,a,s',c)\bar{p}(s'\mid s,a)\\
    &= \sum_{s''}\hat{\mu}_h(s,a,s'',c)\bar{p}(s'\mid s,a)
\end{align*}
The last is because $ \hat{\mu}_h(s,a,s',c)= \hat{\mu}_h(s,a,s'',c)$ for every $s',s''$.

\Cref{eq:polytope_seventh,eq:polytope_eighth,eq:polytope_ninth} are true by definition.
\end{proof}

\begin{lemma}\label{lem:com_polytope}
Every COM with dynamics inside the confidence set is in the polytope.
\end{lemma}
\begin{proof}
Let $\mu$ be a COM defined with dynamics $\tilde{p}$ and policy $\pi$.

\Cref{eq:polytope_first} -  for every $h\notin\mathbf\Lambda$ and $s_{h+1}\in \S$:

\begin{align*}
    \sum_{a_h,s_h}\mu_{h}(s_h,a_h,s_{h+1},c) &= \sum_{s_h,a_h}\mu_{h}(s_h,a_h,c)\tilde{p}_h(s_{h+1}\mid s_h,a_h)\\
    &= \sum_{s_h,a_h}\tilde{p}_h(s_{h+1}\mid s_h,a_h)\sum_{\curly{s_{h'},a_{h'}}_{h'=1}^{h-1}\in \calT_c}\prod_{h'=1}^{h}\pi_h(a_{h'}\mid s_{h'})\prod_{h'\notin{\mathbf \Lambda}_h}\tilde{p}_h(s_{h'+1}\mid s_{h'},a_{h'})\\
    &= \sum_{\curly{s_{h'},a_{h'}}_{h'=1}^{h}\in \calT_c}\prod_{h'=1}^{h}\pi_h(a_{h'}\mid s_{h'})\prod_{h'\notin{\mathbf \Lambda}_{h+1}}\tilde{p}_h(s_{h'+1}\mid s_{h'},a_{h'})\\
    &= \sum_{a_{h+1}}\sum_{\curly{s_{h'},a_{h'}}_{h'=1}^{h}\in \calT_c}\prod_{h'=1}^{h+1}\pi_h(a_{h'}\mid s_{h'})\prod_{h'\notin{\mathbf \Lambda}_{h+1}}\tilde{p}_h(s_{h'+1}\mid s_{h'},a_{h'})
    \tag{$\sum_{a_{h+1}}\pi_{h+1}(a_{h+1}\mid s_{h+1}) = 1$}
    \\
    &= \sum_{a_{h+1}}\mu_{h+1}(s_{h+1},a_{h+1},c)
\end{align*}

\Cref{eq:polytope_second} - for every $h\notin \mathbf\Lambda$ and $s_{h+1}\in \S$:
\begin{align*}
    \sum_{a_h,s_h}\mu_{h}(s_h,a_h,s_{h+1},c) &= \sum_{a_{h+1}}\mu_{h+1}(s_{h+1},a_{h+1},c)\\
    &= \sum_{a_{h+1},s_{h+2}}\mu_{h+1}(s_{h+1},a_{h+1},s_{h+2},c)
\end{align*}
Where the first equation is exactly the same as the proof of \Cref{eq:polytope_first}.

\Cref{eq:polytope_third} -  for every $h\in{\mathbf \Lambda}$, $s_h,s_{h+1}$, $a_h\in\A$ and $c\in \C_h$:
\begin{align*}
    &\sum_{a_{h+1}}\mu_{h+1}(s_{h+1}, a_{h+1},c\|(s_h,a_h,s_{h+1})) \\ \tag{$\Lambda_h=\Lambda_{h+1}$}
    &\qquad= \sum_{a_{h+1}}\sum_{\curly{s_{h'},a_{h'}}_{h'=1}^{h}\in \calT_{c\|(s_h,a_h,s_{h+1})}}\prod_{h'=1}^{h+1}\pi(a_{h'}\mid s_{h'})\prod_{h'\notin{\mathbf \Lambda}_h}\tilde{p}_{h'}(s_{h'+1}\mid s_{h'},a_{h'}) \\\tag{$\sum_{a_{h+1}}\pi_{h+1}(a_{h+1}\mid s_{h+1}) = 1$}
    &\qquad= \sum_{\curly{s_{h'},a_{h'}}_{h'=1}^{h-1}\in \calT_{c}}\prod_{h'=1}^{h}\pi(a_{h'}\mid s_{h'})\prod_{h'\notin{\mathbf \Lambda}_h}\tilde{p}_{h'}(s_{h'+1}\mid s_{h'},a_{h'})\sum_{a_{h+1}}\pi(a_{h+1}\mid s_{h+1})\\
    &\qquad= \mu_h(s_h,a_h, c)
\end{align*}

\Cref{eq:polytope_fourth} - for every $h\in\mathbf\Lambda$, $s_h,s_{h+1}$, $a_h\in\A$ and $c\in \C_h$:
\begin{align*}
    \sum_{a_{h+1},s_{h+2}}\mu_{h+1}(s_{h+1},a_{h+1},s_{h+2}, c\|(s_h,a_h,s_{h+1})) &= \sum_{a_{h+1}}\mu_{h+1}(s_{h+1}, a_{h+1},c\|(s_h,a_h,s_{h+1}))\\
    &= \mu_h(s_h,a_h, c)
\end{align*}
Where the last equation is exactly the same as the proof of \Cref{eq:polytope_third}.



\Cref{eq:polytope_fifth} - 
\begin{align*}
    \sum_{s,a}\mu_1(s,a,()) = \sum_{a}\mu_1(s_{init},a,()) = \sum_{a}\pi(a\mid s_{init}) = 1
\end{align*}

\Cref{eq:polytope_sixth} -  we first show, for every $s,s'\in \S, a\in \A$ and $h\notin\mathbf\Lambda$:
\begin{align*}
    \mu_h(s,a,s',c) &= \tilde{p}_h(s'\mid s,a)\mu_h(s,a,c)\\
    &= \tilde{p}_h(s'\mid s,a)\sum_{s''}\mu_h(s,a,c)p(s''\mid s,a)\\
    &= \tilde{p}_h(s'\mid s,a)\sum_{s''}\mu_h(s,a,s'',c)\\
\end{align*}

Which means that \Cref{eq:polytope_sixth} can be written as:
\begin{align*}
    \abs{\tilde{p}_h(s'\mid s,a) - \bar{p}_h(s'\mid s,a)} \le  \epsilon_h(s,a,s')
\end{align*}
Which is true if the dynamics are inside the confidence set.

\Cref{eq:polytope_seventh,eq:polytope_eighth,eq:polytope_ninth} are true by definition.
\end{proof}

\begin{corollary}\label{lem:real_in_polytope}
Assume $G$, the optimal COM (real dynamics with optimal policy) is inside the polytope
\end{corollary}
\begin{proof}
Directly from \Cref{lem:com_polytope} and the definition of $G_1$ (\Cref{eq:G1}).
\end{proof}

\begin{lemma}\label{lem:polytope_to_com}
Fix $c\in \C_H$ and $\hat \mu$ in the polytope, with $\pi$ being the policy corresponding to $\hat\mu$. There are dynamics in the confidence set $p_c$, and COM $\mu^{p^c,\pi}$ such that for every $s,a$ and $h\in\mathbf{\Lambda}$: 
\begin{align*}
    \mu_h^{p^c,\pi}(s,a,c_{:h}) = \hat{\mu}_{h}(s,a,c_{:h})
\end{align*}
And for every $s,a,s'$ and $h\notin\mathbf{\Lambda}$:
\begin{align*}
    \mu_h^{p^c,\pi}(s,a,s',c_{:h}) = \hat{\mu}_{h}(s,a,s',c_{:h})
\end{align*}
\end{lemma}
\begin{proof}
We will prove it with induction on $h$. In all the proof we short $\mu = \mu^{p^c,\pi}$. We fix $p^c$ to be, for every $s,a,s',c$:
\begin{align*}
    p^c_h(s'\mid s,a) = \frac{\hat\mu(s,a,s',c)}{\sum_{s''}\hat\mu(s,a,s'',c)}
\end{align*}
From \Cref{eq:polytope_sixth} it is inside the confidence set.

For each adversarial step $h$ we care only about the case that $c_{h,3}=s_h$ since when it is not the case $\mu$ and $\hat{\mu}$ are both zero. 

For base $h=1$ - if $1\in \mathbf{\Lambda}$:
\begin{align*}
    \hat{\mu}_1(s_{init},a,()) &= \pi_1(a\mid s_{init})\sum_{a'}\hat{\mu}_1(s_{init},a',())\\
    &= \pi_1(a\mid s_{init})\\
    &= \mu_1(s_{init},a,())
\end{align*}

If $1\notin\mathbf{\Lambda}$:
\begin{align*}
    \hat{\mu}_1(s_{init},a,s',()) &= p^c_1(s'\mid s,a)\pi_1(a\mid s_{init})\sum_{a'}\hat{\mu}_1(s_{init},a',())\\
    &= p^c_1(s'\mid s,a)\pi_1(a\mid s_{init})\\
    &= \mu_1(s_{init},a,s',())
\end{align*}

We first observe that:
\begin{equation}\label{eq:both_adv}
\begin{aligned}
    \pi(a\mid s_{h+1},c){\mu}_{h}(s_h,a_h,c_{:h}) &= \pi(a\mid s_{h+1},c)\sum_{\curly{s_{h'},a_{h'}}_{h'=1}^{h-1}\in \calT_c}\prod_{h'=1}^{h}\pi_h(a_{h'}\mid s_{h'},c)\prod_{h'\notin{\mathbf \Lambda}_h}p_{h'}(s_{h'+1}\mid s_{h'},a_{h'}) \\
    &= \pi(a\mid s_{h+1},c)\sum_{\curly{s_{h'},a_{h'}}_{h'=1}^{h}\in \calT_c}\prod_{h'=1}^{h}\pi_h(a_{h'}\mid s_{h'},c)\prod_{h'\notin{\mathbf \Lambda}_{h+1}}p_{h'}(s_{h'+1}\mid s_{h'},a_{h'}) \\
    &= \pi(a\mid s_{h+1},c)\sum_{a_{h+1}}\sum_{\curly{s_{h'},a_{h'}}_{h'=1}^{h}\in \calT_c}\prod_{h'=1}^{h+1}\pi_h(a_{h'}\mid s_{h'},c)\prod_{h'\notin{\mathbf \Lambda}_{h+1}}p_{h'}(s_{h'+1}\mid s_{h'},a_{h'})\\
    &= \pi(a\mid s_{h+1},c)\sum_{a_{h+1}}\mu_{h+1}(s_{h+1},a_{h+1},c_{:h+1})\\
    &= \mu_{h+1}(s_{h+1},a,c_{:h+1})
\end{aligned}
\end{equation}

\begin{equation}\label{eq:second_adv}
\begin{aligned}
    &\pi_{h+1}(a\mid s_{h+1})\sum_{s_h,a_{h}}{\mu}_{h}(s_{h},a_{h},s_{h+1},c)\\
    &\qquad= \pi_{h+1}(a\mid s_{h+1})\sum_{s_h,a_{h}}p_{h}^c(s_{h+1}\mid s_h,a_h){\mu}_{h}(s_{h},a_{h},c)\\
    &\qquad= \pi_{h+1}(a\mid s_{h+1})\sum_{s_h,a_h}p_{h}^c(s_{h+1}\mid s_h,a_h)\sum_{\curly{s_{h'},a_{h'}}_{h'=1}^{h-1}\in \calT_c}\prod_{h'=1}^{h}\pi_h(a_{h'}\mid s_{h'},c)\prod_{h'\notin{\mathbf \Lambda}_h}p_{h'}(s_{h'+1}\mid s_{h'},a_{h'})\\
    &\qquad= \pi_{h+1}(a\mid s_{h+1})\sum_{\curly{s_{h'},a_{h'}}_{h'=1}^{h}\in \calT_c}\prod_{h'=1}^{h}\pi_h(a_{h'}\mid s_{h'},c)\prod_{h'\notin{\mathbf \Lambda}_{h+1}}p_{h'}(s_{h'+1}\mid s_{h'},a_{h'})\\
    &\qquad= \pi_{h+1}(a\mid s_{h+1})\sum_{a_{h+1}}\sum_{\curly{s_{h'},a_{h'}}_{h'=1}^{h}\in \calT_c}\prod_{h'=1}^{h+1}\pi_h(a_{h'}\mid s_{h'},c)\prod_{h'\notin{\mathbf \Lambda}_{h+1}}p_{h'}(s_{h'+1}\mid s_{h'},a_{h'})\\
    &\qquad= \pi_{h+1}(a\mid s_{h+1})\sum_{a_{h+1}}\mu_{h+1}(s_{h+1},a_{h+1},c)\\
    &\qquad= \mu_{h+1}(s_{h+1},a,c)
\end{aligned}
\end{equation}


If $h,h+1\in\mathbf{\Lambda}$, let $s_h,a_h$ be the state and action conditioned for $h$ in $c$. We have:
\begin{align*}
    \hat{\mu}_{h+1}(s_{h+1},a,c_{:h+1}) &= \pi(a\mid s_{h+1},c)\sum_{a_{h+1}}\hat{\mu}_{h+1}(s_{h+1},a_{h+1},c_{:h+1}) \tag{(definition of $\pi$)} \\
    &= \pi(a\mid s_{h+1},c)\hat{\mu}_{h}(s_h,a_h,c_{:h}) \tag{(\Cref{eq:polytope_third})}\\
    &=  \pi(a\mid s_{h+1},c){\mu}_{h}(s_h,a_h,c_{:h}) \tag{induction assumption}\\
    &= \mu_{h+1}(s_{h+1},a,c_{:h+1}) \tag{\Cref{eq:both_adv}}
\end{align*}

If $h\in \mathbf{\Lambda}$ and $h+1\notin\mathbf{\Lambda}$, we have:
\begin{align*}
    \hat{\mu}_{h+1}(s_{h+1},a,s,c_{:h+1}) &= p_{h+1}^c(s\mid s_{h+1},a)\pi_{h+1}(a\mid s_{h+1})\sum_{s_{h+2},a_{h+1}}\hat{\mu}_{h+1}(s_{h+1},a_{h+1},s_{h+2},c_{:h+1})\tag{definition of $\pi,p^c$}\\
    &= p_{h+1}^c(s\mid s_{h+1},a)\pi_{h+1}(a\mid s_{h+1})\hat{\mu}_{h}(s_{h},a_{h},c_{:h})\tag{\Cref{eq:polytope_fourth}}\\
    &= p_{h+1}^c(s\mid s_{h+1},a)\pi_{h+1}(a\mid s_{h+1}){\mu}_{h}(s_{h},a_{h},c_{:h}) \tag{induction assumption}\\
    &= p_{h+1}^c(s\mid s_{h+1},a){\mu}_{h+1}(s_{h+1},a,c_{:h+1}) \tag{\Cref{eq:both_adv}}\\
    &= {\mu}_{h+1}(s_{h+1},a,s,c_{:h+1})
\end{align*}

If $h\notin \mathbf{\Lambda}$ and $h+1\in\mathbf{\Lambda}$, from \Cref{eq:polytope_first}:
\begin{align*}
    \hat{\mu}_{h+1}(s_{h+1},a,c) &= \pi_{h+1}(a\mid s_{h+1})\sum_{a_{h+1}}\hat{\mu}_{h+1}(s_{h+1},a_{h+1},c)\\
    &= \pi_{h+1}(a\mid s_{h+1})\sum_{s_h,a_{h}}\hat{\mu}_{h}(s_{h},a_{h},s_{h+1},c)\\
    &= \pi_{h+1}(a\mid s_{h+1})\sum_{s_h,a_{h}}{\mu}_{h}(s_{h},a_{h},s_{h+1},c)\tag{induction assumption}\\
    &= \mu_{h+1}(s_{h+1},a,c) \tag{\Cref{eq:second_adv}}
\end{align*}

If $h\notin \mathbf{\Lambda}$ and $h+1\notin\mathbf{\Lambda}$, from \Cref{eq:polytope_second}:
\begin{align*}
    \hat{\mu}_{h+1}(s_{h+1},a,s,c_{:h+1}) &= p_{h+1}^c(s\mid s_{h+1},a)\pi_{h+1}(a\mid s_{h+1})\sum_{s_{h+2},a_{h+1}}\hat{\mu}_{h+1}(s_{h+1},a_{h+1},s_{h+2},c_{:h+1})\\
    &= p_{h+1}^c(s\mid s_{h+1},a)\pi_{h+1}(a\mid s_{h+1})\sum_{s_h,a_h}\hat{\mu}_{h}(s_{h},a_{h},c_{:h})\\
    &= p_{h+1}^c(s\mid s_{h+1},a)\pi_{h+1}(a\mid s_{h+1})\sum_{s_h,a_h}{\mu}_{h}(s_{h},a_{h},c_{:h}) \tag{induction assumption}\\
    &= p_{h+1}^c(s\mid s_{h+1},a){\mu}_{h+1}(s_{h+1},a,c_{:h+1}) \tag{\Cref{eq:second_adv}}\\
    &= {\mu}_{h+1}(s_{h+1},a,s,c_{:h+1})
\end{align*}
\end{proof}

\begin{lemma}\label{lem:polytope_size}
For every $\hat{\mu}$ in the polytope, we have for $h\in\mathbf{\Lambda}$:
\begin{align*}
    \sum_{s,a,c}\hat{\mu}_h(s,a,c) = S^{\Lambda_h}
\end{align*}
And for $h\notin\mathbf{\Lambda}$:
\begin{align*}
    \sum_{s,a,s',c}\hat{\mu}_h(s,a,s',c) =S^{\Lambda_h}
\end{align*}
\end{lemma}
\begin{proof}
We will prove by induction on $h$. The base is from \Cref{eq:polytope_fifth}.

We first prove the step for the case that $\Lambda_h = \Lambda_{h+1}$. 
If $h+1\in\mathbf{\Lambda}$, from \Cref{eq:polytope_first}:
\begin{align*}
    \sum_{a,s,c} \hat{\mu}_{h+1}(s,a,c) = \sum_{s',a',s,c}\hat{\mu}_h(s',a',s,c) = S^{\Lambda_h} = S^{\Lambda_{h+1}}
\end{align*}
If $h+1\notin\mathbf{\Lambda}$, from \Cref{eq:polytope_second}:
\begin{align*}
    \sum_{a,s,s',c} \hat{\mu}_{h+1}(s,a,s',c) = \sum_{s',a',s,c}\hat{\mu}_h(s',a',s,c) = S^{\Lambda_h} = S^{\Lambda_{h+1}}
\end{align*}

Now assume $\Lambda_{h+1} = \Lambda_h + 1$, which means that $h\in\mathbf{\Lambda}$. If $h+1\in\mathbf{\Lambda}$:
\begin{align*}
    \sum_{s,a,c\in \C_{h+1}}\hat{\mu}_{h+1}(s,a,c) &= \sum_{s,a,s',c\in \C_{h},\tilde{s},\tilde{a}}\hat{\mu}_{h+1}(s,a,c\| (\tilde{s}, \tilde{a}, s')) \\
    &= \sum_{s,a,c\in \C_{h},\tilde{s},\tilde{a}}\hat{\mu}_{h+1}(s,a,c\| (\tilde{s}, \tilde{a}, s)) \tag{\Cref{eq:polytope_seventh}}\\
    &= \sum_s\sum_{c\in \C_{h},\tilde{s},\tilde{a}}\hat{\mu}_h(\tilde{s},\tilde{a},c) \tag{\Cref{eq:polytope_third}}\\
    &= S\cdot S^{\Lambda_h}\\
    &= S^{\Lambda_{h+1}}
\end{align*}

If $h+1 \notin\mathbf{\Lambda}$:
\begin{align*}
        \sum_{s,a,s'',c\in \C_{h+1}}\hat{\mu}_{h+1}(s,a,s'',c) &= \sum_{s,a,s',s'',c\in \C_{h},\tilde{s},\tilde{a}}\hat{\mu}_{h+1}(s,a,s'',c\| (\tilde{s}, \tilde{a}, s')) \\
    &= \sum_{s,a,s'',c\in \C_{h},\tilde{s},\tilde{a}}\hat{\mu}_{h+1}(s,a,s'',c\| (\tilde{s}, \tilde{a}, s)) \tag{\Cref{eq:polytope_eighth}}\\
    &= \sum_s\sum_{c\in \C_{h},\tilde{s},\tilde{a}}\hat{\mu}_h(\tilde{s},\tilde{a},c) \tag{\Cref{eq:polytope_fourth}}\\
    &= S\cdot S^{\Lambda_h}\\
    &= S^{\Lambda_{h+1}}
\end{align*}
\end{proof}

\begin{lemma}\label{lem:adv_dyn_bound}
For every $k,h$:
\begin{align*}
    \sum_{c\in\C_h}\varrho_k(c) \le \roundy{SA}^{\Lambda}
\end{align*}
\end{lemma}
\begin{proof}
Fix some $k,h$, we have:
\begin{align*}
    \sum_{c\in\C_h}\varrho_k(c) = \sum_{\curly{s_h,a_h,s_{h+1}}_{h\in \mathbf{\Lambda}_h}}\prod_{m\in\mathbf{\Lambda}_h}p_m(s_{h+1}\mid s_h,a_h) \le \prod_{m\in\mathbf{\Lambda}_h}\sum_{s,a,s'}p_m(s'\mid s,a) = \prod_{m\in\mathbf{\Lambda}_h}SA = \roundy{SA}^{\Lambda_h} &\le \roundy{SA}^{\Lambda}
\end{align*}
\end{proof}

\begin{lemma}\label{lem:adv_dyn_bound_consecutive}
Assume the adversarial steps are consecutive. For every $k,h$:
\begin{align*}
    \sum_{c\in\C_h}\varrho_k(c) \le S\roundy{A}^{\Lambda}
\end{align*}
\end{lemma}
\begin{proof}
Denote $h_1$ to be the first adversarial step and $h_2$ be the last.
We will prove by induction that for every $h\in[h_1,h_2]$:
\begin{align*}
    \sum_{c\in\C_h}\varrho_k(c) \le S\roundy{A}^{\Lambda_{h}+1}
\end{align*}

The base:
\begin{align*}
    \sum_{c\in\C_{h_1}}\varrho_k(c) = \sum_{s_{h_1},a_{h_1},s_{h_1+1}}p_{h_1}(s_{h_1+1}\mid s_{h_1},a_{h_1})= \sum_{s_{h_1},a_{h_1}}1 =  SA
\end{align*}

Now induction step. Assume true for $h-1 \ge h_1$ and we'll prove for $h$.

Since $\C_h$ can only have $(c_h)_3 = (c_{h+1})_1$, we have for every $h \le h_2$
\begin{align*}
    \sum_{c\in\C_h}\varrho_k(c) &= \sum_{s_{h+1}}\sum_{\curly{s_m,a_m}_{m\in [h_1,h]}}\prod_{m\in [h_1,h]}p_m(s_{m+1}\mid s_m,a_m) \\
    &= \sum_{\curly{s_m,a_m}_{m\in [h_1,h]}}\sum_{s_{h+1}}p_h(s_{h+1}\mid s_h,a_h) \prod_{m\in [h_1,h-1]}p_m(s_{m+1}\mid s_m,a_m) \\
    &= \sum_{\curly{s_m,a_m}_{m\in [h_1,h]}} \prod_{m\in [h_1,h-1]}p_m(s_{m+1}\mid s_m,a_m)\\
    &= A\sum_{s_{h}}\sum_{\curly{s_m,a_m}_{m\in [h_1,h-1]}} \prod_{m\in [h_1,h-1]}p_m(s_{m+1}\mid s_m,a_m)\\
    &= A\sum_{c\in\C_{h-1}}\varrho_k(c)\\
    &\le SAA^{\Lambda_{h-1}+1}\\
    &\le SA^{\Lambda_h+1}
\end{align*}

For every $h > h_2$ we have $\Lambda_h = \Lambda_{h_2} + 1$ (since $h_2$ is the last adversarial step) which concludes the proof.
\end{proof}

\subsection{Regret bound}
\begin{lemma}\label{lem:regret_decomposition}
\begin{align*}
    \R_K^{\mathcal{H}} &= \underbrace{\sum_{k,h,s,a}\roundy{q_h^{p_k,\pi_k}(s,a)-\sum_{c\in\C_h}\hat{\mu}_h^k(s,a,c)\varrho_k(c)}\ell_h^k(s,a)}_{\textsc{Error}} + \underbrace{\sum_{k,h,s,a,c}\hat{\mu}_h^k(s,a,c)\roundy{\varrho_k(c)\ell_h^k(s,a) - \hat{\ell}_h^k(s,a,c)}}_{\textsc{Bias1}}\\
    &\quad + \underbrace{\sum_{k,s,h,a,c}\roundy{\hat{\mu}_h^k(s,a,c) - \mu_h^*(s,a,c)}\hat{\ell}_h^k(s,a,c)}_{\textsc{Reg}} + \underbrace{\sum_{k,s,h,a,c}\mu_h^*(s,a,c)\hat{\ell}_h^k(s,a,c) - \sum_{k,s,h,a}q^{p_k,\pi^*}_h(s,a)\ell_h^k(s,a)}_{\textsc{Bias2}}\\
\end{align*}
\end{lemma}
\begin{proof}
\begin{align*}
    \R_K^{\mathcal{H}} &= \sum_{k,h,s,a}\roundy{q_h^{p_k,\pi_k}(s,a)-q^{p_k,\pi^*}_h(s,a)}\ell_h^k(s,a)\\
    &= \textsc{Error} + \sum_{k,h,s,c}\hat{\mu}_h^k(s,a,c)\varrho_k(c)\ell_h^k(s,a) - \sum_{k,s,h,a}q^{p_k,\pi^*}_h(s,a)\ell_h^k(s,a)\\
    &= \textsc{Error} + \textsc{Bias1} + \sum_{k,s,h,a,c}\hat{\mu}_h^k(s,a,c)\hat{\ell}_h^k(s,a,c) - \sum_{k,s,h,a}q^{p_k,\pi^*}_h(s,a)\ell_h^k(s,a)\\
    &= \textsc{Error} +\textsc{Bias1} + \underbrace{\sum_{k,s,h,a,c}\roundy{\hat{\mu}_h^k(s,a,c) - \mu_h^*(s,a,c)}\hat{\ell}_h^k(s,a,c)}_{\textsc{Reg}} \\
    &\quad+ \underbrace{\sum_{k,s,h,a,c}\mu_h^*(s,a,c)\hat{\ell}_h^k(s,a,c) - \sum_{k,s,h,a}q^{p_k,\pi^*}_h(s,a)\ell_h^k(s,a)}_{\textsc{Bias2}}\\
\end{align*}
\end{proof}

\begin{lemma}\label{lem:reg}
Assume $G$, we have:
\begin{align*}
    \textsc{Reg} \le \frac{\ln\roundy{SAC}HS^{\Lambda}}{\eta} + \frac{\eta}{2}\roundy{KH\roundy{SA}^{\Lambda+1}+ \frac{H}{2\gamma}\ln\roundy{\frac{H}{\delta}}}
\end{align*}
And if the adversarial steps are consecutive:
\begin{align*}
    \textsc{Reg} \le \frac{\ln\roundy{SAC}HS^{\Lambda}}{\eta} + \frac{\eta}{2}\roundy{KHS^2\roundy{A}^{\Lambda+1}+ \frac{H}{2\gamma}\ln\roundy{\frac{H}{\delta}}}
\end{align*}
\end{lemma}
\begin{proof}
The optimal COM $\mu^*$ is in the polytope (\Cref{lem:real_in_polytope}), and thus the expression $\textsc{Reg}$ matches exactly the regret promise of the OMD the algorithm runs. Thus, it has a standard OMD upper bound (see e.g., Lemma 13 of \cite{jin2019learning}):
\begin{align*}
    \textsc{Reg} \le \frac{1}{\eta}KL(\mu^*\|\hat{\mu}_1) + \frac{\eta}{2}\sum_{k,s,h,a,c}\hat{\mu}_h^k(s,a,c)\hat{\ell}^2_k(s,a,c)
\end{align*}

We will now bound each term separately.
\begin{align*}
    KL(\mu^*\|\hat{\mu}^1) &= \sum_{s,h,a,c}\mu^*_h(s,a,c)\ln\roundy{\frac{\mu^*_h(s,a,c)}{\hat{\mu}^{1}_h(s,a,c)}}\\
    &\le \sum_{s,h,a,c}\mu^*_h(s,a,c)\ln\roundy{\frac{1}{\hat{\mu}^{1}_h(s,a,c)}}\\
    &\le \ln\roundy{SAC}\sum_{s,h,a,c}\mu^*_h(s,a,c) \tag{$\mu^1$ is uniform}\\
    &\le \ln\roundy{SAC}HS^{\Lambda},
\end{align*}
where the last is due to \Cref{lem:polytope_size}. Notice that the locations for which $\mu^1=0$ are if $(c_{h-1})_3 \ne s$ and in that case also $\mu^*=0$ due to \Cref{eq:polytope_seventh,eq:polytope_eighth} so it is not part of the sum.

The second term:
\begin{align*}
    \sum_{k,h,s,a,c}\hat{\mu}_h^k(s,a,c)\hat{\ell}_h^k(s,a,c)^2 &\le \sum_{k,h,s,a,c}\frac{\hat{\mu}_h^k(s,a,c)\ell_h^k(s,a)}{u_h^k(s,a,c)+\gamma}\hat{\ell}_h^k(s,a,c)\\
    &\le \sum_{k,h,s,a,c}\frac{\hat{\mu}_h^k(s,a,c)}{u_h^k(s,a,c)}\ell_h^k(s,a)\hat{\ell}_h^k(s,a,c) \tag{$\gamma > 0$}\\
    \tag{by definition of $u$}
    &\le \sum_{k,h,s,a,c}\hat{\ell}_h^k(s,a,c)\ell_h^k(s,a)\\
    &\le \sum_{k,h,s,a,c} \ell_h^k(s,a)^2 p_k(c) + \frac{H}{2\gamma}\ln\roundy{\frac{H}{\delta}}\tag{$G_2$(\Cref{eq:G2})}\\ 
    &\le KH\roundy{SA}^{\Lambda+1}+ \frac{H}{2\gamma}\ln\roundy{\frac{H}{\delta}}\tag{\Cref{lem:adv_dyn_bound}}\\ 
\end{align*}

The bound for consecutive adversarial steps is the same with \Cref{lem:adv_dyn_bound_consecutive} instead of \Cref{lem:adv_dyn_bound}.

\end{proof}

\begin{lemma}\label{lem:bias2}
Assume $G$, we have:
\begin{align*}
    \textsc{Bias2} \le \frac{H}{2\gamma}\ln\roundy{\frac{H}{\delta}}
\end{align*}
\end{lemma}
\begin{proof}
Fix $h\in[H]$. Using $G_3$ (\Cref{eq:G3}):
\begin{align*}
    \sum_{k,s,a,c}\mu_h^*(s,a,c)\hat{\ell}(s,a,c) &\le\sum_{k.s,a,c} \mu_h^*(s,a,c)\varrho_k(c)\ell(s,a) + \frac{1}{2\gamma}\ln\roundy{\frac{H}{\delta}}\\
    &=\sum_{k.s,a} q_h^*(s,a)\ell(s,a) + \frac{1}{2\gamma}\ln\roundy{\frac{H}{\delta}}
\end{align*}

Thus:
\begin{align*}
    \textsc{Bias2} \le \frac{H}{2\gamma}\ln\roundy{\frac{H}{\delta}}
\end{align*}

\end{proof}

\begin{lemma}\label{lem:huge_lem}
Assume $G$, for every step $h$ and a collection of transitions $\curly{p_k^{c,s}}_{c\in \C_h,\,s\in S}$ such that for all $c,s$, $p_s^{k,c}\in \mathcal{P}_k$ we have:
\begin{align*}
    \sum_{k,s,a,c}\varrho_k(c) \abs{\mu^{p_s^{k,c},\pi_k}_h(s,a,c) - \mu^{p,\pi_k}_h(s,a,c)} \le &H^2S\ln\roundy{\frac{KASH}{\delta}}\roundy{SA\ln\roundy{K} + \ln\roundy{\frac{H}{\delta}}} \\&+ H\sqrt{S\ln\roundy{\frac{KASH}{\delta}}}\roundy{\sqrt{SAK} + HSA\ln(K) + \ln\roundy{\frac{H}{\delta}}}
\end{align*}    
\end{lemma}
\begin{proof}
Denote:
\begin{align*}
    \epsilon^*_k(s'\mid s,a) &= \Theta\roundy{\sqrt{\frac{p(s'\mid s,a)\ln\roundy{\frac{KSA}{\delta}}}{\max\curly{1, N_k(s,a)}}} + \frac{\ln\roundy{\frac{KSA}{\delta}}}{\max\curly{1, N_k(s,a)}}}\\
    B_1 &\coloneqq \max_h\sum_{k,s,a}\frac{q_h^k(s,a)}{\max\curly{1, N_k(s,a)}}\\
    B_2 &\coloneqq \max_h\sum_{k,s,a}\frac{q_h^k(s,a)}{\sqrt{\max\curly{1, N_k(s,a)}}}\\
\end{align*}

From Lemma 8 in \cite{jin2019learning} we have under $G_1$ (\Cref{eq:G1}) that for every dynamics $\hat{p}$ in the confidence set after episode $k$:
\begin{align*}
    \abs{\hat{p}(s'\mid s,a) - p(s'\mid s,a)} \le \epsilon^*_k(s'\mid s,a)
\end{align*}

We denote $\hat{\cal H} = {\cal H} \setminus {\mathbf \Lambda}$. We denote $\hat{\cal H}_h$ to be all the steps in $\hat{\cal H}$ until step $h$ (non-inclusive), and $\hat{\cal H}_{m:h}$ to be all the steps in $\hat{\cal H}$ from $m$ to $h$, including $m$ (if $m\in\hat{\cal H}$) but not $h$. Additionally, we denote for some $h'\le h(c)$, $\calT_{c, h'}$ to be all trajectories of $c$ up to step $h'$ (non-inclusive), and for $h'\le h'' \le h(c)$ $\calT_{c, h':h''}$ to be the trajectories from $h'$ to $h''$, not including $h''$ and including the action of $h'$ but not the state.

We have for every transition function $p$:
\begin{align*}
    \mu^{p,\pi}_h(s,a,c) &= \sum_{\curly{s_{h'},a_{h'}}_{h'=1}^{h-1}\in \calT_c}\prod_{h'=1}^{h}\pi_{h'}(a_{h'}\mid s_{h'})\prod_{h'\in\hat{\cal H}_{h}}p_h(s_{h'+1}\mid s_{h'},a_{h'}) 
\end{align*}
With $s_{h}=s$.

Let $\mu^{p_s^{k,c},\pi_k}_h = \mu_{h,s}^{k,c}$ and $\mu^{p^*,\pi_k}_h=\mu_h^{k}$. We have:
\begin{align*}
    \abs{\mu_{h,s}^{k,c}(s,a,c)-\mu_h^{k}(s,a,c)} &= 
    \sum_{\curly{s_{h'},a_{h'}}_{h'=1}^{h-1}\in \calT_c}\prod_{h'=1}^{h}\pi_{h'}(a_{h'}\mid s_{h'})\abs{\prod_{h'\in\hat{\cal H}_h}p_{h',s}^{k,c}(s_{h'+1}\mid s_{h'},a_{h'}) - \prod_{h'\in\hat{\cal H}_h}p_{h'}(s_{h'+1}\mid s_{h'},a_{h'}) }\\
\end{align*}

Focus on the expression in the abs:
\begin{align*}
    &\abs{\prod_{h'\in\hat{\cal H}_h}p_{h',s}^{k,c}(s_{h'+1}\mid s_{h'},a_{h'}) - \prod_{h'\in\hat{\cal H}_h}p_{h'}(s_{h'+1}\mid s_{h'},a_{h'}) } \\
    &=\Bigg|\prod_{h'\in\hat{\cal H}_h}p_{h',s}^{k,c}(s_{h'+1}\mid s_{h'},a_{h'}) - \prod_{h'\in\hat{\cal H}_h}p_{h'}(s_{h'+1}\mid s_{h'},a_{h'})  \\
    &\quad\pm \sum_{m\in\hat{\cal H}_{\hat{h}(2):h}}\prod_{h'\in\hat{\cal H}_{m}}p_{h'}(s_{h'+1}|s_{h'},a_{h'})\prod_{h'\in\hat{\cal H}_{m:h}}p_{h',s}^{k,c}(s_{h'+1}|s_{h'},a_{h'})\Bigg|\\
    &= \Bigg|\sum_{m\in\hat{\cal H}_{h}}\prod_{h'\in\hat{\cal H}_{m}}p_{h'}(s_{h'+1}|s_{h'},a_{h'})\prod_{h'\in\hat{\cal H}_{m:h}}p_{h',s}^{k,c}(s_{h'+1}|s_{h'},a_{h'}) - \\
    &\quad\sum_{m\in\hat{\cal H}_{\hat{h}(2):(h+1)}}\prod_{h'\in\hat{\cal H}_{m}}p_{h'}(s_{h'+1}|s_{h'},a_{h'})\prod_{h'\in\hat{\cal H}_{m:h}}p_{h',s}^{k,c}(s_{h'+1}|s_{h'},a_{h'})\Bigg| \\
    &=\sum_{m\in\hat{\cal H}_{h}}\abs{p_{m,s}^{k,c}(s_{m+1}|s_m,a_m) - p_{m}(s_{m+1}|s_m,a_m)}\prod_{h'\in\hat{\cal H}_{m}}p_{h'}(s_{h'+1}|s_{h'},a_{h'})\prod_{h'\in\hat{\cal H}_{(m+1):h}}p_{h',s}^{k,c}(s_{h'+1}|s_{h'},a_{h'})\\
\end{align*}

Combining the with the original statement:
\begin{align*}
    &\abs{\mu_{h,s}^{k,c}(s,a,c)-\mu_h^{k}(s,a,c)} \\
    &\le \sum_{\curly{s_{h'},a_{h'}}_{h'=1}^{h-1}\in \calT_c}\prod_{h'=1}^{h}\pi_{h'}(a_{h'}\mid s_{h'})\sum_{m\in\hat{\cal H}_{h}}\abs{p_{m,s}^{k,c}(s_{m+1}|s_m,a_m) - p_{m}(s_{m+1}|s_m,a_m)}\prod_{h'\in\hat{\cal H}_{m}}p_{h'}(s_{h'+1}|s_{h'},a_{h'})\\
    &\hspace{34em}\prod_{h'\in\hat{\cal H}_{(m+1):h}}p_{h',s}^{k,c}(s_{h'+1}|s_{h'},a_{h'})\\    
    &=\sum_{m\in\hat{\cal H}_{h}}\sum_{\curly{s_{h'},a_{h'}}_{h'=1}^{h-1}\in \calT_c}\abs{p_{m,s}^{k,c}(s_{m+1}|s_m,a_m) - p_{m}(s_{m+1}|s_m,a_m)}\roundy{\prod_{h'=1}^{m}\pi_{h'}(a_{h'}|s_{h'})\prod_{h'\in\hat{\cal H}_{m}}p_{h'}(s_{h'+1}|s_{h'},a_{h'})}\\
    &\hspace{24em}\roundy{\prod_{h'={m+1}}^{h}\pi_{h'}(a_{h'}|s_{h'})\prod_{h'\in\hat{\cal H}_{(m+1):h}}p_{h',s}^{k,c}(s_{h+1}|s_h,a_h)}\\
    &=\sum_{m\in\hat{\cal H}_{:h}}\sum_{s_m,a_m}\sum_{s_{m+1}\in S_{m+1}^c}\abs{p_{m,s}^{k,c}(s_{m+1}|s_m,a_m) - p_{m}(s_{m+1}|s_m,a_m)}\\
    &\hspace{11em}\roundy{\sum_{\curly{s_{h'},a_{h'}}\in \calT_{c,m}}\prod_{h'=1}^m\pi(a_{h'}|s_{h'})\prod_{h'\in\hat{H}_{m}}p_{h'}(s_{h'+1}|s_{h'},a_{h'})}\\
    &\hspace{11em}\roundy{\sum_{\curly{s_{h'},a_{h'}}\in \calT_{c,(m+1:h)}}\prod_{h'=m+1}^{h}\pi(a_{h'}|s_{h'})\prod_{h'\in\hat{\cal H}_{(m+1):h}}p_{h',s}^{k,c}(s_{h'+1}|s_{h'},a_{h'})}\\
    &=\sum_{m\in\hat{\cal H}_{:h}}\sum_{s_m,a_m}\min\curly{2,\sum_{s_{m+1}\in S_{m+1}^c}\epsilon^*_{i_k}(s_{m+1}|s_m,a_m)}\roundy{\sum_{\curly{s_{h'},a_{h'}}\in \calT_{c,m}}\prod_{h'=1}^m\pi(a_{h'}|s_{h'})\prod_{h'\in\hat{H}_{m}}p_{h'}(s_{h'+1}|s_{h'},a_{h'})}\\
    &\hspace{17em}\roundy{\sum_{\curly{s_{h'},a_{h'}}\in \calT_{c,(m+1:h)}}\prod_{h'=m+1}^{h}\pi(a_{h'}|s_{h'})\prod_{h'\in\hat{\cal H}_{(m+1):h}}p_{h',s}^{k,c}(s_{h'+1}|s_{h'},a_{h'})}\\
    &= \sum_{m\in\hat{\cal H}_{:h}}\sum_{s_m,a_m}\sum_{s_{m+1}\in S_{m+1}^c}\min\curly{2,\sum_{s_{m+1}\in S_{m+1}^c}\epsilon^*_{i_k}(s_{m+1}|s_m,a_m)}\mu_m^k(s_m,a_m,c_{:m})\mu_{h,s}^{k,c}(s,a,c_{(m+1):}\mid s_{m+1}, c_{:m})
\end{align*}
Where we use $c_{:m}$ to be the parts of the conditions until $m$ and $c_{(m+1):}$ to be the parts from $m+1$, including the action in $m+1$ (if $m+1$ is adversarial) and not the state of $m+1$.

Using the same logic:
\begin{align*}
    &\abs{\mu_{h,s}^{k,c}(s,a,c_{(m+1):}\mid s_{m+1}, c_{:m})-\mu_h^{k}(s,a,c_{(m+1):}\mid s_{m+1}, c_{:m})} \\
    &\le \sum_{h'\in\hat{\cal H}_{(m+1):h}}\sum_{s_{h'}',a_{h'}'}\sum_{s_{h'+1}'\in S_{h'+1}^c}\min\curly{2,\sum_{s_{h'+1}'\in S_{h'+1}^c}\epsilon^*_{k}(s'_{h'+1}|s_{h'}',a_{h'}')}\mu_{h'}^k(s_{h'}',a_{h'}',c_{(m+1):}\mid s_{m+1}, c_{:m})\\
    &\hspace{34em}\mu_{h,s}^{k,c}(s,a,c_{(h'+1):}|s'_{h'+1}, c_{:h'})\\
\end{align*}

Denote $w_m=(s_m,a_m,s_{m+1})$. Summing both we have:
\begin{align*}
    &\sum_{k,s,a,c}\varrho_k(c) \abs{\mu_{h,s}^{k,c}(s,a,c_{(m+1):}\mid s_{m+1}, c_{:m})-\mu_h^{k}(s,a,c_{(m+1):}\mid s_{m+1}, c_{:m})}   \\
    &\le \underbrace{\sum_{k,s,a,c}\sum_{m\in\hat{\cal H}_{:h}}\sum_{w_m}\epsilon^*_{i_k}(s_{m+1}|s_m,a_m)\varrho_k(c)\mu_m^k(s_m,a_m,c_{:m})\mu_{h}^{k}(s,a,c_{(m+1):}\mid s_{m+1}, c_{:m})}_{(i)}\\
    &+\underbrace{\sum_{k,s,a,c}\sum_{m\in\hat{\cal H}_{:h}}\sum_{w_m}\epsilon^*_{i_k}(s_{m+1}|s_m,a_m)\varrho_k(c)\mu_m^k(s_m,a_m,c_{:m})\roundy{\mu_{h,s}^{k,c}(s,a,c_{(m+1):}| s_{m+1}, c_{:m}) - \mu_{h}^{k}(s,a,c_{(m+1):}| s_{m+1}, c_{:m})}}_{(ii)}\\
\end{align*}

Now we bound both terms:
\begin{align*}
    (i) &= \sum_{k,c_{:m}}\sum_{m\in\hat{\cal H}_{:h}}\sum_{w_m}\epsilon^*_{i_k}(s_{m+1}|s_m,a_m)\varrho_k(c_{:m})\mu_m^k(s_m,a_m,c_{:m})\sum_{s,a,c_{(m+1):}}\varrho_k(c_{(m+1):})\mu_{h}^{k}(s,a,c_{(m+1):}\mid s_{m+1}, c_{:m})\\
    &=\sum_{k,c_{:m}}\sum_{m\in\hat{\cal H}_{:h}}\sum_{w_m}\epsilon^*_{i_k}(s_{m+1}|s_m,a_m)\varrho_k(c_{:m})\mu_m^k(s_m,a_m,c_{:m})\\
    &= \Theta\roundy{\sum_{k, c_{:m}}\sum_{m\in\hat{\cal H}_{:h}}\sum_{w_m}\varrho_k(c_{:m})\mu_m^k(s_m,a_m,c_{:m})\roundy{\sqrt{\frac{p_m(s_{m+1}|s_m,a_m)\ln\roundy{\frac{KASH}{\delta}}}{\max\curly{1,N_{i_k}(s_m,a_m)}}} + \frac{\ln\roundy{\frac{KASH}{\delta}}}{\max\curly{1,N_{i_k}(s_m,a_m)}}}}\\
    &\le \Theta\roundy{\sum_{k, c_{:m}}\sum_{m\in\hat{\cal H}_{:h}}\sum_{s_m,a_m}\varrho_k(c_{:m})\mu_m^k(s_m,a_m,c_{:m})\roundy{\sqrt{S\sum_{s_{m+1}}\frac{p_m(s_{m+1}|s_m,a_m)\ln\roundy{\frac{KASH}{\delta}}}{\max\curly{1,N_{i_k}(s_m,a_m)}}} + \frac{\ln\roundy{\frac{KASH}{\delta}}}{\max\curly{1,N_{i_k}(s_m,a_m)}}}}\\
    &= \Theta\roundy{\sum_{k, c_{:m}}\sum_{m\in\hat{\cal H}_{:h}}\sum_{s_m,a_m}\varrho_k(c_{:m})\mu_m^k(s_m,a_m,c_{:m})\roundy{\sqrt{S\frac{\ln\roundy{\frac{KASH}{\delta}}}{\max\curly{1,N_{i_k}(s_m,a_m)}}} + \frac{\ln\roundy{\frac{KASH}{\delta}}}{\max\curly{1,N_{i_k}(s_m,a_m)}}}}\\
    &= \Theta\roundy{\sum_{k}\sum_{m\in\hat{\cal H}_{:h}}\sum_{s_m,a_m}\roundy{\sqrt{S\frac{\ln\roundy{\frac{KASH}{\delta}}}{\max\curly{1,N_{i_k}(s_m,a_m)}}} + \frac{\ln\roundy{\frac{KASH}{\delta}}}{\max\curly{1,N_{i_k}(s_m,a_m)}}}}\sum_{c_{:m}} \varrho_k(c_{:m})\mu_m^k(s_m,a_m,c_{:m})\\
    &= \Theta\roundy{\sum_{k}\sum_{m\in\hat{\cal H}_{:h}}\sum_{s_m,a_m}q_m^k(s_m,a_m)\roundy{\sqrt{S\frac{\ln\roundy{\frac{KASH}{\delta}}}{\max\curly{1,N_{i_k}(s_m,a_m)}}} + \frac{\ln\roundy{\frac{KASH}{\delta}}}{\max\curly{1,N_{i_k}(s_m,a_m)}}}}\\
    &=\Theta\roundy{HB_2\sqrt{S\ln\roundy{\frac{KASH}{\delta}}} + HB_1\ln\roundy{\frac{KASH}{\delta}}}
\end{align*}
The inequality is due to Cauchy-Schwarz. 

And the second:
\begin{align*}
    (ii) &\le \sum_{k,s,a,c}\sum_{m\in\hat{\cal H}_{:h}}\sum_{w_m}\epsilon^*_{i_k}(s_{m+1}|s_m,a_m)\varrho_k(c)\mu_m^k(s_m,a_m,c_{:m})\\
    &\hspace{-3em}\roundy{\sum_{h'\in\hat{\cal H}_{(m+1):h}}\sum_{s_{h'}',a_{h'}'}\min\curly{2,\sum_{s_{h'+1}'\in S_{h'+1}^c}\epsilon^*_{k}(s'_{h'+1}|s_{h'}',a_{h'}')}\mu_{h'}^k(s_{h'}',a_{h'}',c_{(m+1):h'}\mid s_{m+1}, c_{:m})\mu_{h,s}^{k,c}(s,a,c_{(h'+1):}|s'_{h'+1}, c_{:h'})}\\
    &= \sum_{k,c}\sum_{m\in\hat{\cal H}_{:h}}\sum_{w_m}\sum_{h'\in\hat{\cal H}_{(m+1):h}}\sum_{s_{h'}',a_{h'}'}\epsilon^*_{i_k}(s_{m+1}|s_m,a_m)\varrho_k(c_{:h'})\mu_m^k(s_m,a_m,c_{:m})\min\curly{2,\sum_{s_{h'+1}'\in S_{h'+1}^c}\epsilon^*_{k}(s'_{h'+1}|s_{h'}',a_{h'}')}\\
    &\quad\mu_{h'}^k(s_{h'}',a_{h'}',c_{(m+1):h'}\mid s_{m+1}, c_{:m})\roundy{\sum_{s,a}\varrho_k(c_{(h'+1):}) \mu_{h,s}^{k,c}(s,a,c_{(h'+1):}|s'_{h'+1}, c_{:h'})}\\
    &= \sum_{k,c}\sum_{m\in\hat{\cal H}_{:h}}\sum_{w_m}\sum_{h'\in\hat{\cal H}_{(m+1):h}}\sum_{s_{h'}',a_{h'}'}\epsilon^*_{i_k}(s_{m+1}|s_m,a_m)\varrho_k(c_{:h'})\mu_m^k(s_m,a_m,c_{:m})\min\curly{2,\sum_{s_{h'+1}'\in S_{h'+1}^c}\epsilon^*_{k}(s'_{h'+1}|s_{h'}',a_{h'}')}\\
    &\quad\mu_{h'}^k(s_{h'}',a_{h'}',c_{(m+1):h'}\mid s_{m+1}, c_{:m})\\
    &\le \Theta(\sum_{m,h'\in\hat{\cal H}:m<h'}\\
    &\underbrace{  
    \begin{aligned}
        \sum_{k,c_{:m}}\sum_{w_m,w_{h'}'}\varrho_k(c_{:h'})&\sqrt{\frac{p_m(s_{m+1}|s_m,a_m)\ln\roundy{\frac{KASH}{\delta}}}{\max\curly{1,N_{k}(s_m,a_m)}}}\mu_m^k(s_m,a_m,c_{:m})
        \\
        &\qquad\cdot\sqrt{\frac{p_{h'}(s_{h'+1}'|s_{h'}',a_{h'}')\ln\roundy{\frac{KASH}{\delta}}}{\max\curly{1,N_{k}(s_{h'}',a_{h'}')}}}\mu_{h'}^k(s_{h'}',a_{h'}',c_{(m+1):h'}\mid s_{m+1}, c_{:m})
    \end{aligned}
    }_{(iii)}\\
    &+ \sum_{m,h'\in\hat{H}:m<h}\underbrace{\sum_{k,c}\sum_{w_m,w_{h'}'}\varrho_k(c_{:h'})\frac{\ln\roundy{\frac{KASH}{\delta}}\mu_m^k(s_m,a_m,c_{:m})\mu_{h'}^k(s_{h'}',a_{h'}',c_{(m+1):h'}\mid s_{m+1}, c_{:m})p(s_{m+1}\mid s_m, a_m)}{\max\curly{1,N_{k}(s_{h'}',a_{h'}')}}}_{(iv)}\\
    &+\sum_{m,h'\in\hat{H}:m<h}\underbrace{\sum_{k,c}\sum_{w_m,s_{h'}', a_{h'}'}\varrho_k(c_{:h'})\frac{\ln\roundy{\frac{KASH}{\delta}}\mu_m^k(s_m,a_m,c_{:m})\mu_{h'}^k(s_{h'}',a_{h'}',c_{(m+1):h'}\mid s_{m+1}, c_{:m})}{\max\curly{1,N_{k}(s_{m}',a_{m})}}}_{(v)})
\end{align*}
In the last we used $\sqrt{xy} \le x + y$ for all $x,y\ge 0$.

Bounding $(iii)$ we get:
\begin{align*}
    &(iii) \le \sqrt{\sum_{k,c}\sum_{w_m,w_h'}\frac{\varrho_k(c_{:h'})\mu_{m}^k(s_m,a_m,c_{:m})p_{h'}(s_{h'+1}'|s_{h'}',a_{h'}')\ln\roundy{\frac{KASH}{\delta}}\mu_{h'}^k(s_h',a_h',c_{(m+1):}|s_{m+1},c_{:m})}{\max\curly{1,N_k(s_m,a_m)}}}\\
    &\quad\sqrt{\sum_{k,c}\sum_{w_m,w_h'}\frac{\varrho_k(c_{:h'})\mu_{m}^k(s_m,a_m,c_{:m})p_m(s_{m+1}|s_m,a_m)\ln\roundy{\frac{KASH}{\delta}}\mu_{h'}^k(s_h',a_h',c_{(m+1):}|s_{m+1}, c_{:m})}{\max\curly{1,N_k(s_h',a_h')}}}\\
    &\hspace{-3em}=\sqrt{\sum_{k,c_{:m}}\sum_{w_m}\frac{\varrho_k(c_{:m})\mu_{m}^k(s_m,a_m,c_{:m})\ln\roundy{\frac{KASH}{\delta}}}{\max\curly{1,N_k(s_m,a_m)}}\hspace{-0.7em}\sum_{s_h',a_h', c_{(m+1):}}\hspace{-1.5em}\varrho_k(c_{(m+1):h'})\mu_{h'}^k(s_h',a_h',c_{(m+1):}|s_{m+1}, c_{:m})\sum_{s_{h+1}'}p_{h'}(s_{h'+1}'|s_{h'}',a_{h'}')}\\
    &\quad\sqrt{\sum_{k,c}\sum_{w_h'}\frac{\varrho_k(c_{:h'})\ln\roundy{\frac{KASH}{\delta}}\sum_{s_{m+1}}\mu_{h'}^k(s_h',a_h',c_{(m+1):}|s_{m+1}, c_{:m})\sum_{s_m, a_m}\mu_{m}^k(s_m,a_m,c_{:m})p_m(s_{m+1}|s_m,a_m)}{\max\curly{1,N_k(s_h',a_h')}}}\\
    &\hspace{-3em}=\sqrt{\sum_{k,c_{:m}}\sum_{w_m}\frac{\varrho_k(c_{:m})\mu_{m}^k(s_m,a_m,c_{:m})\ln\roundy{\frac{KASH}{\delta}}}{\max\curly{1,N_k(s_m,a_m)}}\hspace{-0.7em}\sum_{s_h',a_h', c_{(m+1):}}\hspace{-1.5em}\varrho_k(c_{(m+1):h'})\mu_{h'}^k(s_h',a_h',c_{(m+1):}|s_{m+1}, c_{:m})\sum_{s_{h+1}'}p_{h'}(s_{h'+1}'|s_{h'}',a_{h'}')}\\
    &\quad\sqrt{\sum_{k,c}\sum_{w_h'}\frac{\varrho_k(c_{:h'})\ln\roundy{\frac{KASH}{\delta}}\sum_{s_{m+1}}\mu_{h'}^k(s_h',a_h',c_{(m+1):}|s_{m+1}, c_{:m})\mu_{m+1}(s_{m+1}, c_{:m})}{\max\curly{1,N_k(s_h',a_h')}}}\\
    &\hspace{-3em}=\sqrt{\sum_{k,c_{:m}}\sum_{w_m}\frac{\varrho_k(c_{:m})\mu_{m}^k(s_m,a_m,c_{:m})\ln\roundy{\frac{KASH}{\delta}}}{\max\curly{1,N_k(s_m,a_m)}}}\sqrt{\sum_{k,c}\sum_{w_h'}\frac{\varrho_k(c_{:h'})\mu_{h'}^k(s_h',a_h',c)\ln\roundy{\frac{KASH}{\delta}}}{\max\curly{1,N_k(s_h',a_h')}}}\\
    &\hspace{-3em}=\sqrt{S\sum_k\sum_{s_m,a_m}\frac{\sum_{c_{:m}}\varrho_k(c_{:m})\mu_{m}^k(s_m,a_m,c_{:m})\ln\roundy{\frac{KASH}{\delta}}}{\max\curly{1,N_k(s_m,a_m)}}}\sqrt{S\sum_k\sum_{s_h',a_h'}\frac{\sum_{c}\varrho_k(c_{:h'})\mu_{h'}^k(s_h',a_h',c)\ln\roundy{\frac{KASH}{\delta}}}{\max\curly{1,N_k(s_h',a_h')}}}\\
    &\hspace{-3em}=\sqrt{S\sum_k\sum_{s_m,a_m}\frac{q_m^k(s_m,a_m)\ln\roundy{\frac{KASH}{\delta}}}{\max\curly{1,N_k(s_m,a_m)}}}\sqrt{S\sum_k\sum_{s_h',a_h'}\frac{q_m^k(s_h',a_h')\ln\roundy{\frac{KASH}{\delta}}}{\max\curly{1,N_k(s_h',a_h')}}}\\
    &\hspace{-3em}\le SB_1\ln\roundy{\frac{KASH}{\delta}}
\end{align*}
In the first we used Cauchy-Schwartz.

Bounding $(iv)$ we get:
\begin{align*}
    (iv) &= \sum_{k,c}\sum_{w_{h'}'}\varrho_k(c_{:h'})\frac{\ln\roundy{\frac{KASH}{\delta}}\mu_{h'}^k(s_{h'}',a_{h'}',c)}{\max\curly{1,N_{k}(s_{h'}',a_{h'}')}}\\
    &= \sum_k \sum_{w_{h'}'}\frac{\ln\roundy{\frac{KASH}{\delta}}q_{h'}^k(s_{h'}',a_{h'}')}{\max\curly{1,N_{k}(s_{h'}',a_{h'}')}}\\
    &= SB_1\ln\roundy{\frac{KASH}{\delta}}
\end{align*}

Bounding $(v)$ we get:
\begin{align*}
    (iv) &= \sum_{k,c_{:m},w_m}\varrho_k(c_{:m})\frac{\ln\roundy{\frac{KASH}{\delta}}\mu_m^k(s_m,a_m, c_{:m})}{\max\curly{1,N_{k}(s_m,a_m)}}\\
    &= \sum_{k,w_m}\frac{\ln\roundy{\frac{KASH}{\delta}}q_m^k(s_m,a_m)}{\max\curly{1,N_{k}(s_m,a_m)}}\\
    &= SB_1\ln\roundy{\frac{KASH}{\delta}}
\end{align*}
 
Which means that we can bound $(ii)$:
\begin{align*}
    (ii) \le \Theta\roundy{SH^2B_1\ln\roundy{\frac{KASH}{\delta}}}
\end{align*}

Which means that the full bound is:
\begin{align*}
    \Theta\roundy{HB_2\sqrt{S\ln\roundy{\frac{KASH}{\delta}}}  + SH^2B_1\ln\roundy{\frac{KASH}{\delta}}}
\end{align*}

Plugging the bounds for $B_1,B_2$ from $G_5$ (\Cref{eq:G5}) gives the desired results.
\end{proof}

\begin{lemma}\label{lem:bias1}
Assume $G$. We have:
\begin{align*}
    \textsc{Bias1} &= \tilde O\roundy{\gamma KH\roundy{SA}^{\Lambda+1} + H^3S^2A + \sqrt{H^4S^2AK}}
\end{align*}
And if the adversarial steps are consecutive:
\begin{align*}
    \textsc{Bias1} &= \tilde O\roundy{\gamma KHS^2\roundy{A}^{\Lambda+1} + H^3S^2A + \sqrt{H^4S^2AK}}
\end{align*}
\end{lemma}
\begin{proof}
We can write:
\begin{align*}
    \textsc{Bias1} = \underbrace{\sum_{k,h,s,a,c}\hat{\mu}_h^k(s,a,c)\roundy{\varrho_k(c)\ell_h^k(s,a) - \E_k\squary{\hat{\ell}_h^k(s,a,c)}}}_{(i)} + \underbrace{\sum_{k,h,s,a,c}\hat{\mu}_h^k(s,a,c)\roundy{ \E_k\squary{\hat{\ell}_h^k(s,a,c)} - \hat{\ell}_h^k(s,a,c)}}_{(ii)}
\end{align*}

First we bound $(i)$:
\begin{align*}
    &\sum_{k,h,s,a,c}\hat{\mu}_h^k(s,a,c)\roundy{\varrho_k(c)\ell_h^k(s,a) - \E_k\squary{\hat{\ell}_h^k(s,a,c)}}\\
    &\qquad= \sum_{k,h,s,a,c}\hat{\mu}_h^k(s,a,c)\roundy{\varrho_k(c)\ell_h^k(s,a) - \ell_h^k(s,a)\frac{\varrho_k(c)\E_k\squary{\mathds{1}\curly{s_h^k=s,a_h^k=a,c_h^k=c}}}{u_h^k(s,a,c) + \gamma}}\\
    &\qquad= \sum_{k,h,s,a,c}\hat{\mu}_h^k(s,a,c)\ell_h^k(s,a)\roundy{\varrho_k(c) - \frac{\mu_h^k(s,a,c)\varrho_k(c)}{u_h^k(s,a,c) + \gamma}}\\
    &\qquad= \sum_{k,h,s,a,c}\frac{\hat{\mu}_h^k(s,a,c)}{u_h^k(s,a,c) + \gamma}\ell_h^k(s,a)\varrho_k(c)\roundy{u_h^k(s,a,c) + \gamma - \mu_h^k(s,a,c)}\\
    &\qquad\le \sum_{k,h,s,a,c}\varrho_k(c)\roundy{u_h^k(s,a,c) + \gamma - \mu_h^k(s,a,c)}\\
    &\qquad= \sum_{k,h,s,a,c}\squary{\varrho_k(c)\roundy{u_h^k(s,a,c) - \mu_h^k(s,a,c)}} + \gamma KH\roundy{SA}^{\Lambda+1} \tag{\Cref{lem:adv_dyn_bound}}\\
    &\qquad= \tilde{O}\roundy{\gamma KH\roundy{SA}^{\Lambda+1} + H^3S^2A + \sqrt{H^4S^2AK}}\tag{\Cref{lem:huge_lem}}
\end{align*}

Additionally, $(ii)$ is bounded in $G_4$ (\Cref{eq:G4}), which gives the desired bound.

The bound for consecutive adversarial steps is the same with \Cref{lem:adv_dyn_bound_consecutive} instead of \Cref{lem:adv_dyn_bound}.

\end{proof}

\begin{lemma}\label{lem:error}
\begin{align*}
    \textsc{Error} \le \tilde O\roundy{H^3S^2A + \sqrt{H^4S^2AK}}
\end{align*}
\end{lemma}
\begin{proof}
From \Cref{lem:com_to_om}:
\begin{align*}
    \textsc{Error} &= {\sum_{k,h,s,a}\varrho_k(c)\roundy{\mu_h^{p_k,\pi_k}(s,a,c)-\sum_{c\in\C_h}\hat{\mu}_h^k(s,a,c)}}\ell_h^k(s,a) \\
    &\le {\sum_{k,h,s,a}\varrho_k(c)\roundy{\mu_h^{p_k,\pi_k}(s,a,c)-\sum_{c\in\C_h}\hat{\mu}_h^k(s,a,c)}}
\end{align*}

From \Cref{lem:polytope_to_com}, for every $c$ there are dynamics $p^c$ such that:
\begin{align*}
    \textsc{Error} &\le {\sum_{k,h,s,a}\varrho_k(c)\roundy{\mu_h^{p_k,\pi_k}(s,a,c)-\sum_{c\in\C_h}{\mu}_h^{p^c,\pi_k}(s,a,c)}}
\end{align*}

\Cref{lem:huge_lem} concludes the proof.
\end{proof}

\begin{utheorem}[Restatement of \Cref{thm:first_regret}]
\Cref{alg:bandit-bandit com} with $\eta=\gamma=1/\sqrt{KA^{\Lambda+1}S}$ has w.p $1-9\delta$:
\begin{align*}
    \R_K^{\mathcal{H}} \le \tilde{O}\roundy{HS^{\Lambda}\sqrt{KSA^{\Lambda+1}} + H^3S^2A + \sqrt{H^4S^2AK}}
\end{align*}
And if the adversarial steps are consecutive, with $\eta=\gamma=\sqrt{S^{\Lambda-2}/KA^{\Lambda+1}}$:
\begin{align*}
    \R_K^{\mathcal{H}} \le \tilde{O}\roundy{H\sqrt{KS^{\Lambda+2}A^{\Lambda+1}} + H^3S^2A + \sqrt{H^4S^2AK}}
\end{align*}
\end{utheorem}
\begin{proof}
We will upper bound the regret assuming $G$ is true, which happens w.p $1-9\delta$ from \Cref{lem:good_event}.

From \Cref{lem:reg}:
\begin{align*}
    \textsc{Reg} &\le \tilde{O}\roundy{\frac{HS^{\Lambda}}{\eta} + \eta\roundy{KH\roundy{SA}^{\Lambda+1}+ \frac{H}{\gamma}}}\\
\end{align*}

From \Cref{lem:bias2}:
\begin{align*}
    \textsc{Bias2} \le \tilde{O}\roundy{\frac{H}{2\gamma}}
\end{align*}

From \Cref{lem:error}:
\begin{align*}
    \textsc{Error} \le \tilde O\roundy{H^3S^2A + \sqrt{H^4S^2AK}}
\end{align*}

From \Cref{lem:bias1}:
\begin{align*}
    \textsc{Bias1} \le  \tilde O\roundy{\gamma KH\roundy{SA}^{\Lambda+1} + H^3S^2A + \sqrt{H^4S^2AK}}
\end{align*}

From \Cref{lem:regret_decomposition}:
\begin{align*}
    \R_K^{\mathcal{H}} \le \tilde{O}\roundy{\frac{HS^{\Lambda}}{\eta} + \roundy{\eta + \gamma}KH\roundy{SA}^{\Lambda+1} + \frac{H}{\gamma} + H^3S^2A + \sqrt{H^4S^2AK}}
\end{align*}

Setting $\eta=\gamma=1/\sqrt{KA^{\Lambda+1}S}$:
\begin{align*}
    \R_K^{\mathcal{H}} \le \tilde{O}\roundy{HS^{\Lambda}\sqrt{KSA^{\Lambda+1}} + H^3S^2A + \sqrt{H^4S^2AK}}
\end{align*}

For the consecutive case, we use the other bound of \Cref{lem:reg,lem:bias1} and get:
\begin{align*}
    \R_K^{\mathcal{H}} \le \tilde{O}\roundy{\frac{HS^{\Lambda}}{\eta} + \roundy{\eta + \gamma}KHS^2\roundy{A}^{\Lambda+1} + \frac{H}{\gamma} + H^3S^2A + \sqrt{H^4S^2AK}}
\end{align*}

Setting $\eta=\gamma=\sqrt{S^{\Lambda-1}/KA^{\Lambda+1}}$:
\begin{align*}
    \R_K^{\mathcal{H}} \le \tilde{O}\roundy{H\sqrt{K(SA)^{\Lambda+1}} + H^3S^2A + \sqrt{H^4S^2AK}}
\end{align*}
\end{proof}

\section{Markov Regret}\label{apx:subpolicy}
\begin{algorithm}[ht]
    \caption{\texttt{COMSP-OMD}} 
    \label{alg:bandit-bandit subpolicy}
    \begin{algorithmic}[1]
        \STATE \textbf{Input:} Step size $\eta$, implicit exploration constant $\gamma$, confidence constant $\delta$.
        \STATE \textbf{Initialization:} Set $\mu^1$ (\Cref{def:initialization2}) and $\pi^1$ to be uniform.
        
        \FOR{$k=1,2,...,K$}
            \STATE $s_1^k = \sinit$
            \FOR{$h = 1,...,H$} 
                \STATE Play action $a_h^k \sim \pi_h^k(\cdot\mid s_h^k, c_h^k)$ where $c_h^k = \roundy{s_{\tilde h}^k, a_{\tilde h}^k, s^k_{\tilde{h}+1}}_{\tilde{h} \in \mathbf{\Lambda}_h}$
                and observe $s_{h+1}^k$
            \ENDFOR

            \STATE Update empirical mean $\{\bar p_h^k\}_{h=1}^H$ and confidence radiuses $\{\epsilon_h^k\}_{h=1}^H$

            \STATE Compute upper COM $u_h^k(s,a,c) =  \argmax_{\mu \in \Delta_{k-1}(\{\bar p_h^k, \epsilon_h^k\}_{h=1}^H)} \mu_h(s,a,c) $ for each visited triplet $(s,a,c)$ 
            \STATE Compute loss estimator $\hat \ell_h^k(s,a,c)$ according to \Cref{def:loss_estimator}
            
            \STATE Update COM by:
            $
                {\mu^{k+1} = \argmin_{\mu \in \Delta_k(\{\bar p_h^k, \epsilon_h^k\}_{h=1}^H)} \eta \langle \mu , \hat \ell^k \rangle + \KL{\mu}{\mu^k}.}
            $ 
            \STATE Update policy: $\pi_{h}^{k+1}(a\mid s,c)
            =\frac{\mu_{h}^{k+1}(s,a,c)}{\sum_{a'}\mu_{h}^{k+1}(s,a',c)}$.
        \ENDFOR
    \end{algorithmic}
\end{algorithm}

We assume all steps are stochastic except between $\tilde{h}_1$ to $\tilde{h}_2$ (the first is the transition $\tilde{h}_1\to\tilde h_1 + 1$ and the last is $\tilde{h}_2-1\to \tilde{h}_2$). The set of all deterministic sub-policies only for those steps is denoted by $\Sigma$. 

The COM $\mu$, in this algorithm, has in the $h_1$th step a sub-policy $\sigma\in\Sigma$ instead of an action. Thus, the policy $\pi$ induced from the COM has this $\sigma$ integrated into it. Thus, for every $\sigma\in \Sigma$ we write its probability on state $s$ as $\pi_{\tilde{h}_1}(\sigma \mid s)$. We also write $q_{\tilde{h}_1}^\pi(s,\sigma)$ to be the probability that $\pi$ will get to $s$ and play $\sigma$.

As you can see in both the definition of the polytope (\Cref{def:polytope2}) and the definition of COM (\Cref{def:com2}), the context is always $()$ for $h\le h_1$ and has the form $(s,\sigma,s')$ for $s,s'\in \S$ and $\sigma\in\Sigma$ for $h\ge h_2$.

We denote $\tilde{p}_k$ to be the realization of the dynamics in $\tilde{h}_1\to\tilde{h}_2$ in episode $k$. Given such realization, we denote $\Sigma_{\tilde{p},s}^{\vec{a}}$ to be all the sub-policies that will play $\vec{a}\in\A^{\Lambda}$ if $s_{\tilde{h}_1}=s$ and the realized dynamics is $\tilde{p}$. We shorten $\Sigma_{k,s}^{\vec{a}}\coloneqq \Sigma_{\tilde{p}_k,s}^{\vec{a}}$, $\Sigma_{k,s}^\sigma\coloneqq \Sigma_{\tilde{p}_k,s}^{\vec{a}(\sigma)}$, and $\Sigma_{k}^c\coloneqq \Sigma_{\tilde{p}_k,c_1}^{\vec{a}(c_2)}$ 

We note that in this section, the expectation conditioned on the history $\E_k$ is also conditioned in the realizations of the dynamics in the $k$th step.

\begin{remark}\label{re:sum_abuse}
When we sum over all steps, states and actions ($\sum_{h,s,a}$) we mean to sum over all $h\le\tilde{h}_1\cup h\ge \tilde{h}_2$ and the summation over actions is over subpolicies in $h=\tilde{h}_1$.
\end{remark}

\subsection{General definitions}
\begin{definition}\label{def:com2}
For every $h\ge \tilde{h}_2$:
\begin{align*}
    \mu_h^{p^{stat},\pi}(s,a,(s',\sigma,s'')) = q_{\tilde{h}_1}^\pi(s',\sigma)q_h^\pi(s,a\mid s_{\tilde{h}_2}=s'')
\end{align*}

For $h=\tilde{h}_1$:
\begin{align*}
    \mu_h^{p^{stat},\pi}(s,\sigma,()) = q_h^\pi(s, \sigma)
\end{align*}

And for every $h<\tilde{h}_1$:
\begin{align*}
    \mu_h^{p^{stat},\pi}(s,a,()) = q_h^\pi(s,a)
\end{align*}

Additionally, we abuse the notation and use:
\begin{align*}
    \mu_h^{p^{stat},\pi}(s,a,s',c) = \mu_h^{p^{stat},\pi}(s,a,c)p(s'\mid s,a)
\end{align*}
\end{definition}

\begin{definition}
    The probability that the agent will reach $s$ at $\tilde h_2$  given that the agent is at state $s$ at time $\tilde h_1$, the transition is $p$ and the agent plays $\sigma$ in the adversarial steps is denoted by,
\begin{align*}
    \varrho^p(s,\sigma,s') = q^{p,\sigma}_{\tilde{h}_2}(s' \mid s_{\tilde{h}_1}= s)
\end{align*}

We denote $\varrho_k(s,\sigma,s') = \varrho^{p_k}(s,\sigma, s')$. 

Given realized dynamics $\tilde{p}$ and $\vec{a}\in\A^{\Lambda}$ we have the same $\varrho^{\tilde{p}}(s,\sigma,s')$ for every $\sigma\in\Sigma_{\tilde{p},s}^{\vec{a}}$. Denote this value as $\varrho^{\tilde{p}}(s,\vec{a},s')$.
\end{definition}

\begin{definition}\label{def:polytope2}
Given confidence radiuses and empirical transition  $\epsilon_{h}(s,a,s'), \bar p_{h}(s'\mid s,a)$ ($ h\in[H],\; s,s'\in\S,\; a\in\A$), define the Polytope $\Delta(\{\epsilon_h,\bar p_h\}_{h=1}^H)$ where $\{q_h\}_{h=1}^H \in \Delta(\{\epsilon_h,\bar p_h\}_{h=1}^H)$ if and only if,
\begin{align}
    &\sum_{a,s'}\hat{\mu}_{h+1}(s,a,s',c) = \sum_{a,s'}\hat{\mu}_h(s', a, s, c)&\forall h\notin\curly{\tilde{h}_1,\tilde{h}_1-1},s\in\S, c\in\C_h\label{eq:polytope2_first}\\
    &\sum_{\sigma}\hat{\mu}_{\tilde{h}_1}(s,\sigma) = \sum_{s',a}\hat{\mu}_{\tilde{h}_1-1}(s',a,s,()) &\forall s\in \S\label{eq:polytope2_second}\\
    &\sum_{a,s''}\hat{\mu}_{\tilde{h}_2}(s',a,s'', (s,\sigma,s')) = \hat{\mu}_{\tilde{h}_1}(s,\sigma)&\forall s,s'\in \S, \sigma\in \Sigma\label{eq:polytope2_third}\\
    &\begin{cases}
        \sum_{s,\sigma}\hat{\mu}_1(s,\sigma) = 1 & \tilde{h}_1=1\\
        \sum_{s,a,s'}\hat{\mu}_1(s,a,s') = 1 & \tilde{h}_1>1\\
    \end{cases}\label{eq:polytope2_fourth}\\
    &\begin{aligned}
        \Big |\hat{\mu}_h(s,a,s',c) - \sum_{s''}\hat{\mu}_h(s,a,s'',c) & \bar{p}(s'\mid s,a)\Big |\le \\ 
    &\sum_{s''}\hat{\mu}_h(s,a,s'',c)\epsilon_h(s,a,s')
    \end{aligned}
    &\forall h\neq \tilde{h}_1, s,s'\in \S, a\in \A,c\in \C_h \label{eq:polytope2_fifth}\\
    &\hat{\mu}_{\tilde{h}_2}(s,a,s',(s'',\sigma,s''')) = 0 & \forall s\ne s''',s,s'\in\S,a\in\A,\sigma\in\Sigma \label{eq:polytope2_sixth}\\
    &\begin{cases}
     \hat{\mu}_1(s,\sigma) = 0 & \tilde{h}_1=1\\
     \hat{\mu}_1(s,a,s',()) = 0 & \tilde{h}_1>1
    \end{cases}
    &\forall s\ne s_{init},s'\in\S,\sigma\in\Sigma,a\in\A\label{eq:polytope2_seventh}
\end{align}
\end{definition}

\begin{definition}\label{def:initialization2}
The initialization $\hat{\mu}^1$ is:

For $h=1$:
\begin{align*}
\begin{cases}
    \hat{\mu}^1_h(s_{init},a,s',()) = \frac{1}{SA} & h< \tilde{h}_1\\
    \hat{\mu}^1_h(s_{init},\sigma,()) = \frac{1}{\abs{\Sigma}} & h= \tilde{h}_1
\end{cases}
\end{align*}

For $1<h<\tilde{h}$:
\begin{align*}
    \hat{\mu}^1_h(s,a,s',()) = \frac{1}{S^2A}
\end{align*}

For $\tilde{h}_1 > 1$:
\begin{align*}
    \hat{\mu}^1_h(s,\sigma) = \frac{1}{S\abs{\Sigma}}
\end{align*}

For $\tilde{h}_2$:
\begin{align*}
    \hat{\mu}^1_h(s,a,s',(s'',\sigma,s)) = \frac{1}{S^2A\abs{\Sigma}}
\end{align*}

For $h>\tilde{h}_2$:
\begin{align*}
    \hat{\mu}^1_h(s,a,s',c) = \frac{1}{S^3A\abs{\Sigma}}
\end{align*}
\end{definition}

\begin{definition}\label{def:loss_estimator}
The loss estimator $\hat{\ell}$ is:
\begin{align*}
    \hat{\ell}_h^k(s,a, ()) &= \frac{\ell_h^k(s,a)\mathds{1}\squary{s_h^k=s,a_h^k=a}}{u_h^k(s,a,()) + \gamma} & h < \tilde h_1\\
    \hat{\ell}_{\tilde h_1}^k(s,\sigma, ()) &= \frac{\ell_{\tilde h_1}^k(s,\sigma)\mathds{1}\squary{s_{\tilde h_1}^k=s,\vec{a}\roundy{\sigma^k}=\vec{a}\roundy{\sigma}}}{\sum_{\sigma'\in \Sigma_{k,s}^{\vec{a}_{k,s}(\sigma)}}u_{\tilde h_1}^k(s,\sigma,()) + \gamma}\\
    \hat{\ell}_h^k(s,a,(s',\sigma,s'')) &= \frac{\ell_h^k(s,a)\mathds{1}\squary{s_h^k=s,a_h^k=a,s_{\tilde{h}_1}=s',s_{\tilde{h}_2}=s'',\vec{a}\roundy{\sigma^k}=\vec{a}\roundy{\sigma}}}{\sum_{\sigma'\in \Sigma_{k,s}^{\vec{a}_{k,s}(\sigma)}}u_h^k(s,a,c) + \gamma} & h \ge \tilde h_2\\
\end{align*}
\end{definition}

\begin{lemma}
For $h\ge \tilde{h}_2$:
\begin{align*}
    q_h^{p,\pi}(s,a) = \sum_{c}\mu_h^{p,\pi}(s,a,c)\varrho^p(c)
\end{align*}
\end{lemma}
\begin{proof}
\begin{align*}
    q_h^{p,\pi}(s,a) &= \sum_{s'}q^{p,\pi}_h(s,a \mid s_{\tilde{h}_2} = s')q_{\tilde{h}_2}^{p,\pi}(s')\\
    &= \sum_{s',s'',\sigma} q^{p,\pi}_h(s,a \mid s_{\tilde{h}_2} = s') q_{\tilde h_1}^{p,\pi}(s'', \sigma)q_{\tilde{h}_2 }^{p,\sigma}(s' \mid s_{\tilde{h}_1}=s'')\\
    &= \sum_{s',s'',\sigma}\mu_h^{p,\pi}(s,a,(s',s'',\sigma))\varrho^p((s',s'',\sigma))\\
    &= \sum_{c}\mu_h^{p,\pi}(s,a,c)\varrho^p(c) 
\end{align*}
\end{proof}

\subsection{Good event}
\begin{definition}
The event $G_1$ - for every $s,a,s',h\ne \tilde{h}_1,k$:
\begin{align}\label{eq:G1_2}
    \abs{p_h(s' \mid s,a) - \bar{p}_h^k(s'\mid s,a)} \le \epsilon_h(s,a,s')
\end{align}

The event $G_2$ - for every $h$: 
\begin{align}\label{eq:G2_2}
    \sum_{k,s,a,c}\ell_h^k(s,a)\squary{\hat{\ell}_h^k(s,a,c) - \varrho_k(c)\ell_h^k(s,a)} \le \frac{1}{2\gamma}\ln\roundy{\frac{H}{\delta}}
\end{align}

The event $G_3$ - for every $h$: 
\begin{align}\label{eq:G3_2}
    \sum_{k,s,a,c}\mu^*_h(s,a,c)\squary{\hat{\ell}_h^k(s,a,c) - \varrho_k(c)\ell_h^k(s,a)} \le \frac{1}{2\gamma}\ln\roundy{\frac{H}{\delta}}
\end{align}

The event $G_4$ - 
\begin{align}\label{eq:G4_2}
    \sum_{k,h,s,a,c}\hat{\mu}_h^k(s,a,c)\roundy{\hat{\ell}_h^k(s,a,c) - \E\squary{\hat{\ell}_h^k(s,a,c)}} &\le H\sqrt{2K\ln\roundy{\frac{1}{\delta}}}
\end{align}

The event $G_5$ - 
\begin{align}\label{eq:G5_2}
    \max_h\sum_{k,s,a}\frac{q_h^k(s,a)}{\max\curly{1, N_k(s,a)}} &\le SA\ln\roundy{K} + \ln\roundy{\frac{H}{\delta}}\\
    \max_h\sum_{k,s,a}\frac{q_h^k(s,a)}{\sqrt{\max\curly{1, N_k(s,a)}}} &\le \sqrt{SAK} + HSA\ln\roundy{K} + \ln\roundy{\frac{H}{\delta}}
\end{align}

The intersection good event $G$ - 
\begin{align*}
    G = G_1 \cap G_2 \cap G_3 \cap G_4 \cap G_5
\end{align*}
\end{definition}

\begin{lemma}\label{lem:good_event2}
\begin{align*}
    \Pr\squary{G} \ge 1 - 9\delta
\end{align*}
\end{lemma}
\begin{proof}
Same proof as \Cref{lem:good_event}.
\end{proof}

\subsection{Polytope Properties}
\begin{lemma}\label{lem:com_to_om2}
\begin{align*}
    q_h^{p,\pi}(s,a) = \sum_c \mu_h^{p,\pi}(s,a,c)\varrho^p(c)
\end{align*}
\end{lemma}
\begin{proof}
For $h\le \tilde{h}_1$ it is by definition since $c=()$ and $\varrho() = 1$.

For $h\ge \tilde{h}_2$ we have:
\begin{align*}
    \sum_c \mu_h^{p,\pi}(s,a,(s',\sigma,s''))\varrho^p(c) &= q^{p,\pi}_{\tilde{h}_1}(s',\sigma)q_h^{p,\pi}(s,a\mid s_{\tilde{h}_2} = s'')  q_{\tilde{h}_2}^{p,\sigma}(s''\mid s_{\tilde{h}_1}=s')\\
    &= q_h^{p,\pi}(s,a\mid s_{\tilde{h}_2} = s'')q_{\tilde{h}_2}^{p,\pi}(s'')\\
    &= q_h^{p,\pi}(s,a)
\end{align*}
\end{proof}

\begin{lemma}
    $\hat{\mu}^1$ is in the polytope
\end{lemma}
\begin{proof}
\Cref{eq:polytope2_first}, for $h=1$ ($\tilde{h}_1 > 2$):
\begin{align*}
    \sum_{a,s'}\hat{\mu}_2^1(s,a,s', ()) &= \sum_{a,s'}\frac{1}{S^2A}\\
    &= \sum_{a}\frac{1}{SA}\\
    &= \sum_{a}\hat{\mu}_1^1(s_{init},a,s,())\\
    &= \sum_{a,s'}\hat{\mu}_1^1(s',a,s,())
\end{align*}
The last is because $\hat{\mu}_1^1(s',a,s,()) = 0$ for $s'\ne s_{init}$.

\Cref{eq:polytope2_first}, for $1 < h < \tilde{h}_1-1$:
\begin{align*}
    \sum_{a,s'}\hat{\mu}_{h+1}^1(s,a,s', ()) &= \sum_{a,s'}\frac{1}{S^2A}\\
    &= \sum_{a,s'}\hat{\mu}_h^1(s',a,s,())
\end{align*}

\Cref{eq:polytope2_first}, for $h = \tilde{h}_2$:
\begin{align*}
    \sum_{a,s'}\hat{\mu}_{h+1}^1(s,a,s',(s'',\sigma,s''')) &= \sum_{a,s'}\frac{1}{S^4A\abs{\Sigma}}\\
    &= \sum_{a}\frac{1}{S^3A\abs{\Sigma}}\\
    &= \sum_{a}\hat{\mu}_{\tilde{h}_2}^1(s''',a,s,(s'',\sigma,s'''))\\
    &= \sum_{a,s'}\hat{\mu}_{\tilde{h}_2}^1(s',a,s,(s'',\sigma,s'''))
\end{align*} 
The last is because $\hat{\mu}_{\tilde{h}_2}^1(s',a,s,c)=0$ if $c_3\ne s'$.

\Cref{eq:polytope2_first}, for $h > \tilde{h}_2$:
\begin{align*}
    \sum_{a,s'}\hat{\mu}_{h+1}^1(s,a,s', c) &= \sum_{a,s'}\frac{1}{S^4A\abs{\Sigma}}\\
    &= \sum_{a,s'}\hat{\mu}_h^1(s',a,s,c)
\end{align*}

\Cref{eq:polytope2_second}:
\begin{align*}
    \sum_{\sigma}\hat{\mu}^1_{\tilde{h}_1}(s,\sigma) &= \sum_{\sigma}\frac{1}{S\abs{\Sigma}}\\
    &= \frac{1}{S}\\
    &= \sum_{s',a}\frac{1}{S^2A}\\
    &= \sum_{s',a}\hat{\mu}_{\tilde{h}_1-1}(s',a,s',())
\end{align*}

\Cref{eq:polytope2_third}:
\begin{align*}
    \sum_{a,s''}\hat{\mu}^1_{\tilde{h}_2}(s',a,s'',(s,\sigma,s')) &= \sum_{a,s''}\frac{1}{S^2A\abs{\Sigma}}\\
    &= \frac{1}{S\abs{\Sigma}}\\
    &= \hat{\mu}^1_{\tilde{h}_1}(s,\sigma)
\end{align*}

\Cref{eq:polytope2_fourth}, if $\tilde{h}_1=1$:
\begin{align*}
    \sum_{s,\sigma}\hat{\mu}_1^1(s,\sigma) = \sum_{s,\sigma}\frac{1}{S\abs{\Sigma}} = 1
\end{align*}

\Cref{eq:polytope2_fourth}, if $\tilde{h}_1>1$:
\begin{align*}
    \sum_{s,a,s'}\hat{\mu}_1^1(s,a,s',()) = \sum_{s,a,s'}\frac{1}{S^2A} = 1
\end{align*}

\Cref{eq:polytope2_fifth}, notice that $\epsilon_h(s,a,s') = 1$ and $\bar{p}(s'\mid s,a) = \frac{1}{S}$.
\begin{align*}
    \hat{\mu}_h^1(s,a,s',c) &= \sum_{s''}\hat{\mu}_h^1(s,a,s',c)\bar{p}(s'\mid s,a)\\
    &= \sum_{s''}\hat{\mu}_h^1(s,a,s'',c)\bar{p}(s'\mid s,a)
\end{align*}
The last is because $ \hat{\mu}_h^1(s,a,s',c)= \hat{\mu}_h^1(s,a,s'',c)$ for every $s',s''$.

\Cref{eq:polytope_sixth,eq:polytope_seventh} is true by definition.
\end{proof}

\begin{lemma}\label{lem:com_polytope2}
Every COM with dynamics inside the confidence set is in the polytope.
\end{lemma}

\begin{proof}
Let $\mu$ be a COM with respect to dynamics $\tilde p$ and policy $\pi$.

\Cref{eq:polytope2_first}, if $h<\tilde{h}_1-1$:
\begin{align*}
    \sum_{a_{h+1},s_{h+2}}\mu_{h+1}(s_{h+1},a_{h+1},s_{h+2},()) &= \sum_{a_{h+1},s_{h+2}}q_{h+1}(s_{h+1},a_{h+1})\tilde{p}_{h+1}(s_{h+2}\mid s_{h+1},a_{h+1})\\
    &= \sum_{a_h,s_h}\sum_{a_{h+1},s_{h+2}}q_h(s_h,a_h)\tilde{p}_h(s_{h+1}\mid a_{h},s_h)\pi_{h+1}(a_{h+1}\mid s_{h+1})\tilde{p}(s_{h+2}\mid s_{h+1},a_{h+1})\\
    &= \sum_{a_h,s_h}q_h(s_h,a_h)\tilde{p}_h(s_{h+1}\mid a_{h},s_h)\sum_{a_{h+1},s_{h+2}}\pi_{h+1}(a_{h+1}\mid s_{h+1})\tilde{p}(s_{h+2}\mid s_{h+1},a_{h+1})\\
    &= \sum_{a_h,s_h}q_h(s_h,a_h)\tilde{p}_h(s_{h+1}\mid a_{h},s_h)\\
    &= \sum_{a_h,s_h}\mu_h(s_h,a_h,s_{h+1},())
\end{align*}

\Cref{eq:polytope2_first}, if $h\ge\tilde{h}_2$:
\begin{align*}
    &\sum_{a_{h+1},s_{h+2}}\mu_{h+1}(s_{h+1},a_{h+1},s_{h+2},(s,\sigma,s')) \\
    &\qquad= \sum_{a_{h+1},s_{h+2}}q_{\tilde{h}_1}(s,\sigma)q_{h+1}(s_{h+1},a_{h+1}\mid s_{\tilde{h}_2}=s')\tilde{p}_{h+1}(s_{h+2}\mid s_{h+1},a_{h+1})\\
    &\qquad= q_{\tilde{h}_1}(s,\sigma)\sum_{a_h,s_h}\sum_{a_{h+1},s_{h+2}}q_h(s_h,a_h\mid s_{\tilde{h}_2}=s')\tilde{p}_h(s_{h+1}\mid a_{h},s_h)\pi_{h+1}(a_{h+1}\mid s_{h+1})\tilde{p}_{h+1}(s_{h+2}\mid s_{h+1},a_{h+1})\\
    &\qquad= q_{\tilde{h}_1}(s,\sigma)\sum_{a_h,s_h}q_h(s_h,a_h\mid s_{\tilde{h}_2}=s')\tilde{p}_h(s_{h+1}\mid a_{h},s_h)\sum_{a_{h+1},s_{h+2}}\pi_{h+1}(a_{h+1}\mid s_{h+1})\tilde{p}_{h+1}(s_{h+2}\mid s_{h+1},a_{h+1})\\
    &\qquad= q_{\tilde{h}_1}(s,\sigma)\sum_{a_h,s_h}q_h(s_h,a_h\mid s_{\tilde{h}_2}=s')\tilde{p}_h(s_{h+1}\mid a_{h},s_h)\\
    &\qquad= \sum_{a_h,s_h}\mu_h(s_h,a_h,s_{h+1},(s,\sigma,s'))
\end{align*}

\Cref{eq:polytope2_second}:
\begin{align*}
    \sum_{\sigma}\mu_{\tilde{h}_1}(s,\sigma) &= \sum_{\sigma}q_{\tilde{h}_1}(s,\sigma)\\
    &= \sum_{s',a'}\sum_{\sigma}q_{\tilde{h}_1-1}(s',a')\tilde{p}_{\tilde{h}_1-1}(s\mid s',a')\pi_{\tilde{h}_1}(\sigma\mid s)\\
    &= \sum_{s',a'}q_{\tilde{h}_1-1}(s',a')\tilde{p}_{\tilde{h}_1-1}(s\mid s',a')\sum_{\sigma}\pi_{\tilde{h}_1}(\sigma\mid s)\\
    &= \sum_{s',a'}q_{\tilde{h}_1-1}(s',a')\tilde{p}_{\tilde{h}_1-1}(s\mid s',a')\\
    &= \sum_{s',a'}\mu_{\tilde{h}_1-1}(s',a',s,())
\end{align*}

\Cref{eq:polytope2_third}:
\begin{align*}
    \sum_{a,s''}\mu_{\tilde{h}_2}(s',a,s'',(s,\sigma,s')) &= \sum_{a,s''}q_{\tilde{h}_1}(s,\sigma)q_{\tilde{h}_2}(s',a\mid s_{\tilde{h}_2}=s')\tilde{p}_{\tilde{h}_2}(s''\mid s',a)
    \\
    &= \sum_{a,s''}q_{\tilde{h}_1}(s,\sigma)\pi_{\tilde{h}_2}(a\mid s')\tilde{p}_{\tilde{h}_2}(s''\mid s',a) \\
    &=q_{\tilde{h}_1}(s,\sigma) \sum_{a,s''}\pi_{\tilde{h}_2}(a\mid s')\tilde{p}_{\tilde{h}_2}(s''\mid s',a)\\
    &= q_{\tilde{h}_1}(s,\sigma)\\
    &= \mu_{\tilde{h}_1}(s,\sigma)\\
\end{align*}

\Cref{eq:polytope2_fourth}, $\tilde{h}_1=1$:
\begin{align*}
    \sum_{s,a}\mu_1(s,\sigma) = \sum_{s,a}q_1(s,\sigma) = 1
\end{align*}

\Cref{eq:polytope2_fourth}, $\tilde{h}_1>1$:
\begin{align*}
    \sum_{s,a,s'}\mu_1(s,a,s') = \sum_{s,a,s'}q_1(s,a)\tilde{p}_1(s'\mid s,a) = 1
\end{align*}

\Cref{eq:polytope2_fifth} -  we first show, for every $s,s'\in \S, a\in \A$ and $h\notin\mathbf{\Lambda}$:
\begin{align*}
    \mu_h(s,a,s',c) &= p(s'\mid s,a)\mu_h(s,a,c)\\
    &= \tilde{p}(s'\mid s,a)\sum_{s''}\mu_h(s,a,c)p(s''\mid s,a)\\
    &= \tilde{p}(s'\mid s,a)\sum_{s''}\mu_h(s,a,s'',c)\\
\end{align*}

Which means that \Cref{eq:polytope2_fifth} can be written as:
\begin{align*}
    \abs{\tilde{p}(s'\mid s,a) - \bar{p}(s'\mid s,a)} \le  \epsilon_h(s,a,s')
\end{align*}
Which is true if the dynamics are inside the confidence set.

\Cref{eq:polytope2_sixth,eq:polytope2_seventh} are true by definition.
\end{proof}

\begin{corollary}\label{lem:real_in_polytope2}
Assume $G$, the optimal COM (real dynamics with optimal policy) is inside the polytope
\end{corollary}
\begin{proof}
Directly from \Cref{lem:com_polytope2} and the definition of $G_1$ (\Cref{eq:G1_2}).
\end{proof}

\begin{lemma}\label{lem:polytope_to_com2}
Fix $\hat\mu$ in the polytope with $\pi$ being the policy corresponding to $\hat\mu$. There are dynamics in the confidence set $p$ such that for every $h < \tilde{h}_1$ and $s,a,s'$:
\begin{align*}
    \hat\mu_h(s,a,s',()) = \mu_h^{p,\pi}(s,a,s',())
\end{align*}
And for every $s,\sigma$:
\begin{align*}
    \hat\mu_{\tilde h_1}(s,\sigma,()) = \mu_{\tilde{h}_1}^{p,\pi}(s,\sigma,c)
\end{align*}

Additionally, for every $c$ there are dynamics $p^c$ such that for every $h\ge\tilde{h}_2$:
\begin{align*}
    \hat\mu_h(s,a,s',c) = \mu_h^{p^c,\pi}(s,a,s',c)
\end{align*}
\end{lemma}
\begin{proof}
In all proof we short $\mu$ to be the COM with the relevant dynamics ($p$ or $p^c$) and $\pi$. We fix $p^c$ (are $p$ for $c=()$) such that for every $s,a,s',h,c$:
\begin{align*}
    p_h^c(s'\mid s,a) = \frac{\hat\mu(s,a,s',c)}{\sum_{s''}\hat\mu(s,a,s'',c)}
\end{align*}
From \Cref{eq:polytope2_fifth} it is inside the confidence set.

For $h<\tilde{h}_1$ we prove with induction on $h$
\begin{align*}
    \hat{\mu}_{h}(s,a,s') &= \pi_h(a\mid s)p_{h}(s'\mid a,s)\sum_{a',s''}\hat{\mu}_h(s,a',s'')\\
    &= \pi_h(a\mid s)p_{h}(s'\mid a,s)\sum_{a',s''}\hat{\mu}_{h-1}(s'',a',s)\\
    &= \pi_h(a\mid s)p_{h}(s'\mid a,s)\sum_{a',s''}{\mu}_{h-1}(s'',a',s)\\
    &= \pi_h(a\mid s)p_{h}(s'\mid a,s)\sum_{a',s''}q_{h-1}(s'',a')p_{h-1}(s\mid s'',a')\\
    &= \pi_h(a\mid s)p_h(s'\mid a,s)q_{h}(s)\\
    &= q_h(s,a,s')\\
    &= \mu_h(s,a,s')
\end{align*}

For $h=\tilde{h}_1$:
\begin{align*}
    \hat{\mu}_h(s,\sigma) &= \pi(\sigma\mid s)\sum_{\sigma'}\hat{\mu}_h(s,\sigma')\\
    &= \pi(\sigma\mid s)\sum_{s',a}\hat{\mu}_{h-1}(s',a,s)\\
    &= \pi(\sigma\mid s)\sum_{s',a}{\mu}_{h-1}(s',a,s)\\
    &= \pi(\sigma\mid s)\sum_{s',a}{\mu}_{h-1}(s',a)p^c_{h-1}(s\mid s',a)\\
    &= \pi(\sigma\mid s)\sum_{s',a}q_{h-1}(s',a)p^c_{h-1}(s\mid s',a)\\
    &= \pi(\sigma\mid s)q_{h-1}(s)\\
    &= \mu_h(s,\sigma)
\end{align*}

For $h=\tilde{h}_2$:
\begin{align*}
    \hat{\mu}_{\tilde{h}_2}(s',a',\tilde{s}, (s,\sigma,s')) &= \pi_{\tilde{h}_2}(a'\mid s')p_{\tilde{h}_2}^c(\tilde{s}\mid s',a')\sum_{a,s''}\hat{\mu}_{\tilde{h}_2}(s',a,s'', (s,\sigma,s'))\\
    &= \pi_{\tilde{h}_2}(a'\mid s')p_{\tilde{h}_2}^c(\tilde{s}\mid s',a')\hat{\mu}_{\tilde{h}_1}(s,\sigma)\\
    &= \pi_{\tilde{h}_2}(a'\mid s')p_{\tilde{h}_2}^c(\tilde{s}\mid s',a'){\mu}_{\tilde{h}_1}(s,\sigma)\\
    &= q_{\tilde{h}_2}(s',a'\mid s_{\tilde{h}_2}=s')p_{\tilde{h}_2}^c(\tilde{s}\mid s',a'){q}_{\tilde{h}_1}(s,\sigma)\\
    &= \mu_{\tilde{h}_2}(s',a',(s,\sigma,s'))p_{\tilde{h}_2}^c(\tilde{s}\mid s',a')\\
    &= \mu_{\tilde{h}_2}(s',a',\tilde{s},(s,\sigma,s'))
\end{align*}

For $h>\tilde{h}_2$ we prove with induction on $h$:
\begin{align*}
    \hat{\mu}_h(s,a,s',c) &= \pi_h(a\mid s)p_h^c(s'\mid s,a)\sum_{a',s''}\hat{\mu}_h(s,a',s'',c)\\
    &= \pi_h(a\mid s)p_h^c(s'\mid s,a)\sum_{a',s''}\hat{\mu}_{h-1}(s'',a',s,c)\\
    &= \pi_h(a\mid s)p_h^c(s'\mid s,a)\sum_{a',s''}{\mu}_{h-1}(s'',a',s,c)\\
    &= \pi_h(a\mid s)p_h^c(s'\mid s,a)\sum_{a',s''}{\mu}_{h-1}(s'',a',c)p_{h-1}^c(s\mid s'',a')\\
    &= \pi_h(a\mid s)p_h^c(s'\mid s,a)\sum_{a',s''}{q}_{h-1}(s'',a'\mid s_{\tilde{h}_2} = c_3)q_{\tilde{h}_1}(c_1,c_2)p_{h-1}^c(s\mid s'',a')\\
    &= \pi_h(a\mid s)p_h^c(s'\mid s,a){q}_{h}(s\mid s_{\tilde{h}_2} = c_3)q_{\tilde{h}_1}(c_1,c_2)\\
    &= p_h^c(s'\mid s,a){q}_{h}(s,a\mid s_{\tilde{h}_2} = c_3)q_{\tilde{h}_1}(c_1,c_2)\\
    &= p_h^c(s'\mid s,a)\mu_h(s,a,c)\\
    &= \mu_h(s,a,s',c)
\end{align*}
\end{proof}

\begin{lemma}\label{lem:polytope_size2}
For every $\hat{\mu}$ in the polytope:
\begin{align*}
\begin{cases}
    \sum_{s,a,s'}\hat{\mu}_h(s,a,s',()) = 1 & h < \tilde{h}_1\\
    \sum_{s,\sigma}\hat{\mu}_h(s,\sigma) = 1 & h = \tilde{h}_1\\
    \sum_{s,a,s',c}\hat{\mu}_h(s,a,s',c) = S & h \ge \tilde{h}_2
\end{cases}
\end{align*}
\end{lemma}
\begin{proof}
For $h<\tilde{h}_1$ we prove by induction. For $h=1$ it is directly from \Cref{eq:polytope2_fourth}. The induction step is, from \Cref{eq:polytope2_first}:
\begin{align*}
    \sum_s\sum_{a,s'}\hat\mu_{h+1}(s,a,s',()) = \sum_s\sum_{a,s'}\hat{\mu}_h(s',a,s,())
\end{align*}

For $h=\tilde{h}_1$ we have:
\begin{align*}
    \sum_s\sum_\sigma\hat{\mu}_h(s,\sigma) = \sum_s\sum_{s',a}\hat{\mu}_{h-1}(s',a,s,())
\end{align*}
Which is $1$ since we already proved for $h<\tilde{h}_1$.

For $h\ge \tilde{h}_2$ we again prove by induction. The base is:
\begin{align*}
    \sum_{s}\sum_c\sum_{a,s'}\hat{\mu}_{\tilde{h}_2}(s,a,s',c) &= \sum_{s}\sum_{s'',\sigma}\sum_{a,s'}\hat{\mu}_{\tilde{h}_2}(s,a,s',(s'',\sigma,s))\tag{\Cref{eq:polytope2_sixth}}\\
    &= \sum_{s}\sum_{s'',\sigma}\hat{\mu}_{\tilde{h}_1}(s'', \sigma) \tag{\Cref{eq:polytope2_third}} \\
    &= \sum_s 1\\
    &= S
\end{align*}
The induction step is, from \Cref{eq:polytope2_first}:
\begin{align*}
    \sum_s\sum_{a,s'}\hat\mu_{h+1}(s,a,s',c) = \sum_s\sum_{a,s'}\hat{\mu}_h(s',a,s,c)
\end{align*}
\end{proof}

\begin{lemma}\label{lem:adv_dyn_bound2}
For every $\tilde{p}$:
\begin{align*}
    \sum_{s,s',\vec{a}}\varrho^{\tilde{p}}(s,\vec{a},s') = SA^{\Lambda}
\end{align*}
\end{lemma}
\begin{proof}
Fix $s,\vec{a}$ and let $\sigma$ be arbitrary sub-policy from $\Sigma_{\tilde{p},s}^{\vec{a}}$. We have:
\begin{align*}
    \sum_{s'}\varrho^{\tilde{p}}(\vec{a},s,s') &= \sum_{s'}\varrho^{\tilde{p}}(\sigma,s,s')\\
    &= \sum_{s'}q^{\tilde{p},\sigma}_{\tilde{h}_2}(s'\mid s_{\tilde{h}_1} = s)\\
    &= 1
\end{align*}
Which means:
\begin{align*}
    \sum_{s,\vec{a}}\sum_{s'}\varrho^{\tilde{p}}(\vec{a},s,s') = \sum_{s,\vec{a}}1 = SA^{\Lambda}
\end{align*}
\end{proof}

\subsection{Regret bound}
\begin{lemma}\label{lem:regret_decomposition2}
\begin{align*}
    \R_K^{\mathcal{M}} &= \underbrace{\sum_{k,h,s,a}\roundy{q_h^{p_k,\pi_k}(s,a)-\sum_{c\in\C_h}\hat{\mu}_h^k(s,a,c)\varrho_k(c)}\ell_h^k(s,a)}_{\textsc{Error}} + \underbrace{\sum_{k,h,s,a,c}\hat{\mu}_h^k(s,a,c)\roundy{\varrho_k(c)\ell_h^k(s,a) - \hat{\ell}_h^k(s,a,c)}}_{\textsc{Bias1}}\\
    &\quad + \underbrace{\sum_{k,s,h,a,c}\roundy{\hat{\mu}_h^k(s,a,c) - \mu_h^*(s,a,c)}\hat{\ell}_h^k(s,a,c)}_{\textsc{Reg}} + \underbrace{\sum_{k,s,h,a,c}\mu_h^*(s,a,c)\hat{\ell}_h^k(s,a,c) - \sum_{k,s,h,a}q^{p_k,\pi^*}_h(s,a)\ell_h^k(s,a)}_{\textsc{Bias2}}\\
\end{align*}
\end{lemma}
\begin{proof}
Same as \Cref{lem:regret_decomposition}.
\end{proof}

\begin{lemma}\label{lem:reg2}
Assume $G$, we have:
\begin{align*}
    \textsc{Reg} \le \frac{\ln\roundy{SAC}HS}{\eta} + \frac{\eta}{2}\roundy{KHS^2A^{\Lambda+1}+ \frac{H}{2\gamma}\ln\roundy{\frac{H}{\delta}}}
\end{align*}
\end{lemma}
\begin{proof}
From $G_1$ (\Cref{eq:G1_2}) we have that the optimal COM $\mu^*$ is in the polytope. The expression $\textsc{Reg}$ matches exactly the regret guarantee of the OMD the algorithm runs. Thus, it has a standard OMD upper bound (see e.g., Lemma 13 of \cite{jin2019learning}):
\begin{align*}
    \textsc{Reg} \le \frac{1}{\eta}KL(\mu^*\|\hat{\mu}_1) + \frac{\eta}{2}\sum_{k,s,h,a,c\in\C_h}\hat{\mu}_h^k(s,a,c)\hat{\ell}^2_k(s,a,c)
\end{align*}

We will now bound each term separately.
\begin{align*}
    KL(\mu^*\|\hat{\mu}^1) &= \sum_{s,h,a,c}\mu^*_h(s,a,c)\ln\roundy{\frac{\mu^*_h(s,a,c)}{\hat{\mu}^{1}_h(s,a,c)}}\\
    &\le \sum_{s,h,a,c}\mu^*_h(s,a,c)\ln\roundy{\frac{1}{\hat{\mu}^{1}_h(s,a,c)}}\\
    &\le \ln\roundy{SAC}\sum_{s,h,a,c}\mu^*_h(s,a,c) \tag{$\mu^1$ is uniform}\\
    &\le \ln\roundy{SAC}HS,
\end{align*}
where the last is due to \Cref{lem:polytope_size2}.

We'll separate the second term to 3 parts - $h<\tilde{h}_1$,$h=\tilde{h}_1$,$h\ge\tilde{h}_2$.

$h<\tilde{h}_1$:
\begin{align*}
    \sum_{k,h<\tilde{h}_1,s,a}\hat{\mu}_h^k(s,a,())\hat{\ell}_h^k(s,a)^2 &\le \sum_{k,h,s,a}\frac{\hat{\mu}_h^k(s,a,())\ell_h^k(s,a)}{u_h^k(s,a,())+\gamma}\hat{\ell}_h^k(s,a)\\
    &\le \sum_{k,h,s,a}\ell_h^k(s,a)\hat{\ell}_h^k(s,a)\\
    &\le \sum_{k,h,s,a} \ell_h^k(s,a)^2 + \frac{H}{2\gamma}\ln\roundy{\frac{H}{\delta}}\tag{$G_2$ (\Cref{eq:G2_2})}\\ 
    &\le SHAK+ \frac{H}{2\gamma}\ln\roundy{\frac{H}{\delta}}\\ 
\end{align*}

$h=\tilde{h}_1$. We short $\Sigma^{\vec{a}}_{k,s}\coloneqq\Sigma^{\vec{a}}_{\tilde{p}_k,s}$, i.e the set of all subpolicies that played $\vec{a}$ on the realized dynamics of episode $k$ starting from $s$.
\begin{align*}\sum_{k,,s,\sigma}\hat{\mu}_{\tilde{h}_1}^k(s,\sigma)\hat{\ell}_{\tilde{h}_1}^k(s,\sigma)^2 &\le \sum_{k,s,\sigma}\frac{\hat{\mu}_{\tilde{h}_1}^k(s,\sigma)\ell_{\tilde{h}_1}^k(s,\sigma)}{\sum_{\sigma'\in\Sigma^{\vec{a}_{k,s}(\sigma)}_{k,s}}u_{\tilde{h}_1}^k(s,\sigma')+\gamma}\hat{\ell}_{\tilde{h}_1}^k(s,\sigma)\\
&= \sum_{k,s,\vec{a}}\frac{\sum_{\sigma\in\Sigma^{\vec{a}}_k} \hat{\mu}_{\tilde{h}_1}^k(s,\sigma)\ell_{\tilde{h}_1}^k(s,\vec{a})}{\sum_{\sigma'\in\Sigma^{\vec{a}}_{k,s}}u_{\tilde{h}_1}^k(s,\sigma')+\gamma}\hat{\ell}_{\tilde{h}_1}^k(s,\vec{a})\\
    &\le \sum_{k,s,\vec{a}}\ell_{\tilde{h}_1}^k(s,\vec{a})\hat{\ell}_{\tilde{h}_1}^k(s,\vec{a})\\
    &\le \sum_{k,s,\vec{a}} \roundy{\ell_{\tilde{h}_1}^k(s,a)}^2 + \frac{\ell_{\tilde{h}_1}^k(s,a)}{2\gamma}\ln\roundy{\frac{H}{\delta}}\tag{$G_2$ (\Cref{eq:G2_2})}\\ 
    &\le S\Lambda^2A^{\Lambda}K+ \frac{\Lambda}{2\gamma}\ln\roundy{\frac{H}{\delta}}\\ 
\end{align*}

$h\ge \tilde{h}_2$:
\begin{align*}
    \sum_{k,h,s,a,c}\hat{\mu}_{h}^k(s,a,c)\hat{\ell}_{h}^k(s,a,c)^2 &\le \sum_{k,h,s,a,c}\frac{\hat{\mu}_{h}^k(s,a,c)\ell_{\tilde{h}_1}^k(s,a)}{\sum_{\sigma\in\Sigma^c_k}u_{h}^k(s,a,(c_1,\sigma,c_3))+\gamma}\hat{\ell}_{h}^k(s,a)\\
    &= \sum_{k,h,s,a,s',\sigma,s''}\frac{\hat{\mu}_{h}^k(s,a,(s',\sigma,s''))\ell_{\tilde{h}_1}^k(s,a)}{\sum_{\sigma'\in\Sigma^c_k}u_{h}^k(s,a,(s',\sigma',s''))+\gamma}\hat{\ell}_{h}^k(s,a)\\
&= \sum_{k,h,s,a,\vec{a},s',s''}\frac{\sum_{\sigma\in\Sigma^{\vec{a}}_{k,s'}}\hat{\mu}_{h}^k(s,a,(s',\sigma,s''))\ell_{\tilde{h}_1}^k(s,a)}{\sum_{\sigma\in\Sigma^{\vec{a}}_{k,s'}}u_{h}^k(s,a,(s',\sigma,s''))+\gamma}\hat{\ell}_{h}^k(s,a)\\
    &\le \sum_{k,h,s,a,\vec{a},s',s''}\ell_{h}^k(s,a)\hat{\ell}_{h}^k(s,a,(s',\sigma,s''))\\
    &\le \sum_{k,h,s,a,\vec{a},s',s''} \varrho(s',\vec{a},s'') + \frac{H}{2\gamma}\ln\roundy{\frac{H}{\delta}}\tag{$G_2$ (\Cref{eq:G2_2})}\\ 
    &\le KHS^2A^{\Lambda+1}+ \frac{H}{2\gamma}\ln\roundy{\frac{H}{\delta}}\\ \tag{\Cref{lem:adv_dyn_bound2}} 
\end{align*}

\end{proof}

\begin{lemma}\label{lem:bias2_2}
Assume $G$, we have
\begin{align*}
    \textsc{Bias2} \le \frac{H}{2\gamma}\ln\roundy{\frac{H}{\delta}}
\end{align*}
\end{lemma}
\begin{proof}
Same as \Cref{lem:bias2}.
\end{proof}

\begin{lemma}\label{lem:huge_lem2}
Assume $G$, for every step $h$ and a collection of transitions $\curly{p_k^{c,s}}_{c\in \C_h,\,s\in S}$ such that for all $c,s$, $p_s^{k,c}\in \mathcal{P}_k$ we have:
\begin{align*}
    \sum_{k,s,a,c}\varrho_k(c) \abs{\mu^{p_s^{k,c},\pi_k}_h(s,a,c) - \mu^{p,\pi_k}_h(s,a,c)} \le &H^2S\ln\roundy{\frac{KASH}{\delta}}\roundy{SA\ln\roundy{K} + \ln\roundy{\frac{H}{\delta}}} \\&+ H\sqrt{S\ln\roundy{\frac{KASH}{\delta}}}\roundy{\sqrt{SAK} + HSA\ln(K) + \ln\roundy{\frac{H}{\delta}}}
\end{align*}    
\end{lemma}
\begin{proof}
Same proof as \Cref{lem:huge_lem}.
\end{proof}

\begin{lemma}\label{lem:bias1_2}
Assume $G$. We have:
\begin{align*}
    \textsc{Bias1} \le \tilde{O}\roundy{\gamma KS^2HA^{\Lambda+1} + H^3S^2A + \sqrt{H^4S^2AK}}
\end{align*}
\end{lemma}
\begin{proof}
We can write:
\begin{align*}
    \textsc{Bias1} = \underbrace{\sum_{k,h,s,a,c}\hat{\mu}_h^k(s,a,c)\roundy{\varrho_k(c)\ell_h^k(s,a) - \E_k\squary{\hat{\ell}_h^k(s,a,c)}}}_{(i)} + \underbrace{\sum_{k,h,s,a,c}\hat{\mu}_h^k(s,a,c)\roundy{ \E_k\squary{\hat{\ell}_h^k(s,a,c)} - \hat{\ell}_h^k(s,a,c)}}_{(ii)}
\end{align*}
We'll bound $(i)$ in 3 parts - $h< \tilde{h}_1, h=\tilde{h}_1,h\ge\tilde{h}_2$.

$h<\tilde{h}_1$:
\begin{align*}
    \sum_{k,h,s,a}\hat{\mu}_h^k(s,a,())\roundy{\ell_h^k(s,a) - \E_k\squary{\hat{\ell}_h^k(s,a)}}&= \sum_{k,h,s,a}\hat{\mu}_h^k(s,a,())\ell_h^k(s,a)\roundy{1 - \frac{\mu_h^k(s,a)}{u_h^k(s,a) + \gamma}}\\
    &= \sum_{k,h,s,a}\frac{\hat{\mu}_h^k(s,a,())}{u_h^k(s,a) + \gamma}\ell_h^k(s,a)\roundy{u_h^k(s,a) + \gamma - \mu_h^k(s,a)}\\
    &\le \sum_{k,h,s,a}\roundy{u_h^k(s,a) + \gamma - \mu_h^k(s,a)}\\
    &= \sum_{k,h,s,a}\squary{\roundy{u_h^k(s,a,c) - \mu_h^k(s,a,c)}} + \gamma KHSA
\end{align*}

$h=\tilde{h}_1$ (in all notations here we omit the $\tilde{h}_1$ in the subscript). Recall that the $a$ here becomes $\sigma$ as explained in \Cref{re:sum_abuse}. 
\begin{align*}
    \sum_{k,s,\sigma}\hat{\mu}^k(s,\sigma)\roundy{\ell^k(s,\sigma) - \E_k\squary{\hat{\ell}(s,\sigma)}}&= \sum_{k,s,\sigma}\hat{\mu}^k(s,\sigma)\roundy{\ell^k(s,\sigma) - \frac{\sum_{\sigma'\in\Sigma_{k,s}^\sigma}\mu^k(s,\sigma')\ell^k(s,\sigma)}{\sum_{\sigma'\in\Sigma_{k,s}^\sigma}u^k(s,\sigma') + \gamma}}\\
    &= \sum_{k,s,\vec{a}}\sum_{\sigma\in\Sigma_{k,s}^{\vec{a}}}\hat{\mu}^k(s,\sigma)\roundy{\ell^k(s,\sigma) - \frac{\sum_{\sigma'\in\Sigma_{k,s}^\sigma}\mu^k(s,\sigma')\ell^k(s,\sigma)}{\sum_{\sigma'\in\Sigma_{k,s}^\sigma}u^k(s,\sigma') + \gamma}}\\
    &= \sum_{k,s,\vec{a}}\frac{\sum_{\sigma\in\Sigma_{k,s}^{\vec{a}}}\hat{\mu}^k(s,\sigma)}{\sum_{\sigma\in\Sigma_{k,s}^{\vec{a}}}u^k(s,\sigma) + \gamma}\ell^k(s,\sigma)\roundy{\sum_{\sigma\in\Sigma_{k,s}^{\vec{a}}}\squary{u^k(s,\sigma) - \mu^k(s,\sigma')} + \gamma}\\
    &\le \Lambda\sum_{k,s,\vec{a}}\sum_{\sigma\in\Sigma_{k,s}^{\vec{a}}}\squary{u^k(s,\sigma) - \mu^k(s,\sigma')} + \gamma\\
    &\le \Lambda\sum_{k,s,\sigma}\squary{u^k(s,\sigma) - \mu^k(s,\sigma')} + \gamma KSA^{\Lambda}\\
\end{align*}

$h\ge\tilde{h}_2$:
\begin{align*}
    &\sum_{k,h,s,a,c}\hat{\mu}_h^k(s,a,c)\roundy{\varrho_k(c)\ell_h^k(s,a) - \E_k\squary{\hat{\ell}_h^k(s,a,c)}}\\
    &\qquad= \sum_{k,h,s,a}\sum_{s',\sigma,s''}\hat{\mu}_h^k(s,a,(s',\sigma,s''))\ell_h^k(s,a)\roundy{\varrho_k(s',\sigma,,s'') - \frac{\sum_{\sigma'\in\Sigma_{k,s'}^{\sigma}}\mu_h^k(s,a,(s',\sigma',s''))\varrho_k(s',\sigma,s'')}{\sum_{\sigma'\in\Sigma_{k,s'}^{\sigma}}u_h^k(s,a,(s',\sigma',s'')) + \gamma}}\\
    \tag{$\varrho(s',\sigma,s'') = \varrho(s',\sigma',s'')$}
    &\qquad= \sum_{k,h,s,a}\sum_{s',\sigma,s''}\hat{\mu}_h^k(s,a,(s',\sigma,s''))\ell_h^k(s,a)\roundy{\varrho_k(s',\sigma,s'') - \frac{\sum_{\sigma'\in\Sigma_{k,s'}^{\sigma}}\mu_h^k(s,a,(s',\sigma',s''))\varrho_k(s',\sigma,s'')}{\sum_{\sigma'\in\Sigma_{k,s'}^{\sigma}}u_h^k(s,a,(s',\sigma',s'')) + \gamma}}\\
    &\qquad= \sum_{k,h,s,a}\sum_{s',\sigma,s''}\frac{\hat{\mu}_h^k(s,a,(s',\sigma,s''))}{\sum_{\sigma'\in\Sigma_{k,s'}^{\sigma}}u_h^k(s,a,(s',\sigma',s'') + \gamma}\ell_h^k(s,a)\varrho_k(s',\sigma',s'')\\
    &\qquad\roundy{\sum_{\sigma'\in\Sigma_{k,s'}^{\sigma}}u_h^k(s,a,(s',\sigma',s'')) + \gamma - \sum_{\sigma'\in\Sigma_{k,s'}^{\sigma}}\mu_h^k(s,a,(s',\sigma',s''))}\\
    &\qquad= \sum_{k,h,s,a}\sum_{s',\vec{a},s''}\frac{\sum_{\sigma\in\Sigma^{\vec{a}}_{k,s'}}\hat{\mu}_h^k(s,a,(s',\sigma,s''))}{\sum_{\sigma'\in\Sigma_{k,s'}^{\vec{a}}}u_h^k(s,a,(s',\sigma',s'') + \gamma}\ell_h^k(s,a)\varrho_k(s',\sigma',s'')\\
    &\qquad\roundy{\sum_{\sigma'\in\Sigma_{k,s'}^{\vec{a}}}u_h^k(s,a,(s',\sigma',s'')) + \gamma - \sum_{\sigma'\in\Sigma_{k,s'}^{\sigma}}\mu_h^k(s,a,(s',\sigma',s''))}\\
    &\qquad\le \sum_{k,h,s,a}\sum_{s',\vec{a},s''}\varrho_k(s',\sigma',s'')\roundy{\sum_{\sigma'\in\Sigma_{k,s'}^{\vec{a}}}u_h^k(s,a,(s',\sigma',s'')) + \gamma - \sum_{\sigma'\in\Sigma_{k,s'}^{\sigma}}\mu_h^k(s,a,(s',\sigma',s''))}\\
    &\qquad= \sum_{k,h,s,a,c}\squary{\varrho_k(c)\roundy{u_h^k(s,a,c) - \mu_h^k(s,a,c)}} + \sum_{k,h,s,a}\sum_{s',\vec{a},s''}\gamma \varrho(s',\vec{a},s'')\\
    &\qquad\le \sum_{k,h,s,a,c}\squary{\varrho_k(c)\roundy{u_h^k(s,a,c) - \mu_h^k(s,a,c)}} + \gamma KS^2HA^{\Lambda+1}
\end{align*}

The first term in both 3 parts can be bounded in the same way as \Cref{lem:huge_lem2} to get a total:
\begin{align*}
    (i) \le \tilde{O}\roundy{\gamma KS^2HA^{\Lambda+1} + H^3S^2A + \sqrt{H^4S^2AK}}
\end{align*}

Additionally, $(ii)$ is bounded in $G_4$ (\Cref{eq:G4_2}), which gives the desired bound.
\end{proof}

\begin{lemma}\label{lem:error2}
\begin{align*}
    \textsc{Error} \le \tilde O\roundy{H^3S^2A + \sqrt{H^4S^2AK}}
\end{align*}
\end{lemma}
\begin{proof}
From \Cref{lem:com_to_om2}:
\begin{align*}
    \textsc{Error} &= {\sum_{k,h,s,a}\varrho_k(c)\roundy{\mu_h^{p_k,\pi_k}(s,a,c)-\sum_{c\in\C_h}\hat{\mu}_h^k(s,a,c)}}\ell_h^k(s,a) \\
    &\le {\sum_{k,h,s,a}\varrho_k(c)\roundy{\mu_h^{p_k,\pi_k}(s,a,c)-\sum_{c\in\C_h}\hat{\mu}_h^k(s,a,c)}}
\end{align*}

From \Cref{lem:polytope_to_com2}, for every $c$ there are dynamics $p^c$ such that:
\begin{align*}
    \textsc{Error} &\le {\sum_{k,h,s,a}\varrho_k(c)\roundy{\mu_h^{p_k,\pi_k}(s,a,c)-\sum_{c\in\C_h}{\mu}_h^{p^c,\pi_k}(s,a,c)}}
\end{align*}

\Cref{lem:huge_lem2} concludes the proof.
\end{proof}

\begin{theorem}
Assume $G$ and $\eta=\gamma=\sqrt{\frac{1}{SKA^{\Lambda+1}}}$:
\begin{align*}
    \R_K^{\mathcal{M}} \le \tilde{O}\roundy{\sqrt{KH^2S^3A^{\Lambda+1}} + H^3S^2A + \sqrt{H^4S^2AK}}
\end{align*}
\end{theorem}
\begin{proof}
From \Cref{lem:reg2}:
\begin{align*}
    \textsc{Reg} &\le \tilde{O}\roundy{\frac{HS}{\eta} + \eta\roundy{KHS^2A^{\Lambda+1}+ \frac{H}{\gamma}}}\\
\end{align*}

From \Cref{lem:bias2_2}:
\begin{align*}
    \textsc{Bias2} \le \tilde{O}\roundy{\frac{H}{\gamma}}
\end{align*}

From \Cref{lem:error2}:
\begin{align*}
    \textsc{Error} \le \tilde O\roundy{H^3S^2A + \sqrt{H^4S^2AK}}
\end{align*}

From \Cref{lem:bias1_2}:
\begin{align*}
    \textsc{Bias1} \le  \tilde{O}\roundy{\gamma KS^2HA^{\Lambda+1} + H^3S^2A + \sqrt{H^4S^2AK}}
\end{align*}

From \Cref{lem:regret_decomposition2}:
\begin{align*}
    \R_K^{\mathcal{M}} \le \tilde{O}\roundy{\frac{HS}{\eta} + \roundy{\eta +\gamma}KS^2HA^{\Lambda+1} + \frac{H}{\gamma} + H^3S^2A + \sqrt{H^4S^2AK}}
\end{align*}

Placing $\eta,\gamma$:
\begin{align*}
    \R_K^{\mathcal{M}} \le \tilde{O}\roundy{\sqrt{KH^2S^3A^{\Lambda+1}} + H^3S^2A + \sqrt{H^4S^2AK}}
\end{align*}

\end{proof}

\section{Unknown adversarial steps}\label{apx:unknown}
\begin{algorithm}[t]
    \caption{\texttt{COM-OMD (Unknown adversarial steps)}} 
    \label{alg:bandit-bandit com unknown}
    \begin{algorithmic}[1]
        
        \STATE \textbf{Initialization:} Initialize $\mathcal{A}_{1},...,\mathcal{A}_{H \choose \Lambda}$ instances of \texttt{COM-OMD} - each corresponds to a different set  out of the ${H \choose \Lambda}$ possibilities of adversarial steps; also initialize a uniform probability over $[{H \choose \Lambda}]$: $\nu^1(i) = \frac{1}{{H \choose \Lambda}}$ and $\hat\nu^1(i) = \frac{1}{{H \choose \Lambda}}$.
        
        \FOR{$k=1,2,...,K$}
            \STATE Sample an instance $I_k \sim \nu^k$
            \STATE Get the next policy from $I_k$: $\pi^k \leftarrow\mathcal{A}_{I_k}$
            \STATE $s_1^k = \sinit$
            \FOR{$h = 1,...,H$}     
                \STATE Play action $a_h^k \sim \pi_h^k(\cdot\mid s_h^k, c_h^k)$ where $c_h^k = (s_{\tilde h}^k)_{\tilde{h} \in \mathbf{\Lambda}_h}$
                and observe $s_{h+1}^k$
            \ENDFOR
            \STATE 
            Feed $A_{I_k}$ with the trajectory and loss feedback $(s_h^k,a_h^k,\bar\ell_h^k(s_h^k,a_h^k))_{h=1}^H$ where $\bar\ell_h^k(s_h^k,a_h^k) = \frac{\ell_h^k(s_h^k,a_h^k)}{\nu(I_k)}$
            \STATE Define instances loss estimator $\hat L^k(i) = \frac{\mathbb{I}\{I_k = i\} \sum_{h=1}^H \ell_h^k(s_h^k,a_h^k)}{\nu^k(i)}$
            \STATE Update $\hat\nu^{k+1}(i) = \frac{\hat\nu^k(i)e^{-\eta \hat L^k(i)}}{\sum_{i'}\hat\nu^k(i')e^{-\eta \hat L^k(i')}}$
            \STATE Update $\nu^{k+1}(i) =\roundy{1 - \xi {H \choose \Lambda}} \hat{\nu}^{k+1}(i) + \xi$ 
        \ENDFOR
    \end{algorithmic}
\end{algorithm}

\begin{lemma}\label{lem:reduction}
With unknown adversarial steps, the regret of algorithm \Cref{alg:bandit-bandit com unknown} is bounded by,
\begin{align*}
    \E\squary{\R_K^{\mathcal{H}}} = O\roundy{\sqrt{H^{\Lambda}K} + \E\squary{\sum_{k}V_1^{k,\pi_{\A^*}^k}(s_{init};\;\bar{\ell}^k) - V_1^{k,\pi^*}(s_{init};\;\bar{\ell}^k)} + \xi H^{\Lambda+1}K}
\end{align*}
\end{lemma}
\begin{proof}
We can decompose the regret as:
\begin{align*}
    \R_K^{\mathcal{H}} = \underbrace{\sum_{k}V_1^{k,\pi_k}(s_{init};\;\ell^k) - V_1^{k,\pi_{\A^*}^k}(s_{init};\;\ell^k)}_{\textsc{Reg}} + \underbrace{\sum_{k} V_1^{k,\pi_{\A^*}^k}(s_{init};\;\ell^k) - V_1^{k,\pi^*}(s_{init};\;\ell^k)}_{\textsc{Cheat}}
\end{align*}

By standard EXP3 regret bound:
\begin{align*}
    \E\squary{\textsc{Reg}} \le \sqrt{H^{\Lambda}K} + \xi K
\end{align*}

Using the fact that $\bbE_k [\frac{\mathds{1}\squary{I_k=i^*}}{\nu^k(i^*)} ] = 1$ and the linearity of the value function with respect to the loss function:
\begin{align*}
    \E\squary{\textsc{Cheat}} &= \E\squary{\sum_{k} \frac{\mathds{1}\squary{I_k=i^*}}{\nu^k(i^*)} \roundy{V_1^{k,\pi_{\A^*}^k}(s_{init};\;\ell) - V_1^{k,\pi^*}(s_{init};\;\ell)}}\\
    &= \E\squary{\sum_{k} \roundy{V_1^{k,\pi_{\A^*}^k}(s_{init};\;\frac{\mathds{1}\squary{I_k=i^*}}{\nu^k(i^*)}\ell) - V_1^{k,\pi^*}(s_{init};\;\frac{\mathds{1}\squary{I_k=i^*}}{\nu^k(i^*)}\ell)}}\\
    &= \E\squary{\sum_{k}V_1^{k,\pi_{\A^*}^k}(s_{init};\;\bar{\ell}) - V_1^{k,\pi^*}(s_{init};\;\bar{\ell})}.
\end{align*}
\end{proof}

\begin{lemma}\label{lem:reduction_huge_lemma}
Assume that \Cref{alg:bandit-bandit com} learns the trajectory (e.g, the counter $N$ increases) in episode only w.p $\nu_k$. Assume for every $k$, $\nu_k \ge \xi$. Then, w.p $1-10\delta$ the same term as in \Cref{lem:huge_lem} can be bounded by:
\begin{align*}
    \tilde{O}\roundy{HS\sqrt{\frac{KA}{\xi}} + \frac{S^2H^2A}{\xi}}
\end{align*}
\end{lemma}
\begin{proof}
Assume $G$ (\Cref{def:good_event}), the bound of \Cref{lem:huge_lem} is:
\begin{align*}
    \tilde O\roundy{HB_2\sqrt{S}  + SH^2B_1}
\end{align*}

Since we have less information, we need to bound $B_1$,$B_2$ again. Note that the probability that $N_h^k(s,a)$  increases is exactly $q_h^k(s,a)\nu_k$ (i.e., the probability that we choose this specific sub-algorithm times the probability to reach $s,a$ in time $h$ given that we play this sub-algorithm), thus, in the same way as in $G_5$ (\Cref{eq:G5}), $\max_h \sum_{k,s,a}\frac{q_h^k(s,a)\nu_k}{\max\curly{1,N_k(s,a)}} \le \tilde{O}\roundy{SA}$. Since $\nu_k \geq \frac1\xi$,
\begin{align*}
    B_1 &= \max_h \sum_{k,s,a}\frac{q_h^k(s,a)}{\max\curly{1,N_k(s,a)}} \le \frac{1}{\xi}\max_h \sum_{k,s,a}\frac{q_h^k(s,a)\nu_k}{\max\curly{1,N_k(s,a)}} \le \tilde{O}\roundy{\frac{SA}{\xi}}
\end{align*}
Where the last is since the  .

\begin{align*}
    B_2 &= \max_h \sum_{k,s,a}\frac{q_h^k(s,a)}{\sqrt{\max\curly{1,N_k(s,a)}}}\\
    &= \max_h \sum_{k,s,a}\sqrt{\frac{q_h^k(s,a)\nu_k}{\max\curly{1,N_k(s,a)}}}\sqrt{\frac{q_h^k(s,a)}{\nu_k}}\\
    &\le \max_h \sqrt{\sum_{s,a,k}\frac{ q_h^k(s,a)\nu_k}{\max\curly{1,N_k(s,a)}}}\sqrt{\sum_{s,a,k}\frac{q_h^k(s,a)}{\nu_k}}\tag{Cauchy-Schwarz}\\
    &\le \tilde{O}\roundy{\sqrt{SA}\sqrt{\sum_{s,a,k}\frac{q_h^k(s,a)}{\nu_k}}} \tag{as in $G_5$ (\Cref{eq:G5})}\\
    &= \tilde{O}\roundy{\sqrt{SA}\sqrt{\sum_{k}\frac{1}{\nu_k}}} \tag{$\sum_{s,a}q(s,a) = 1$}\\
    &\le \tilde{O}\roundy{\sqrt{\frac{KSA}{\xi}}}
\end{align*}
Since $G$ is true w.p $1-9\delta$ (\Cref{lem:good_event}), this concludes the proof.
\end{proof}

\begin{theorem}
Initializing all sub-algorithms in algorithm \Cref{alg:bandit-bandit com unknown} as \Cref{alg:bandit-bandit com} with:
\begin{align*}
    \eta &= K^{-2/3} H^{1/3}
S^{\Lambda/3-1/3}H^{-\Lambda/3}
A^{-\Lambda/3-1/3}\\
    \xi &= K^{-1/3} H^{2/3}S^{2\Lambda/3+1/3}A^{\Lambda/3+1/3}H^{-2\Lambda/3}\\
    \gamma &= K^{-1/3}S^{-\Lambda/3-2/3}A^{-2\Lambda/3-2/3}H^{\Lambda/3}
\end{align*}
We get that the expected regret of algorithm \Cref{alg:bandit-bandit com unknown} is bounded by,
\begin{align*}
    \E[\R_K^{\mathcal{H}}]  \le \tilde{O}\Big(
& H^{1+\frac{\Lambda}{3}}
S^{\frac{2\Lambda+1}{3}}
A^{\frac{\Lambda+1}{3}}
K^{\frac{2}{3}} \\
& +\;
H^{1+\frac{\Lambda}{3}}
S^{\frac{5-2\Lambda}{6}}
A^{\frac{2-\Lambda}{6}}
K^{\frac{2}{3}} \\
& +\;
H^{1-\frac{\Lambda}{3}}
S^{\frac{\Lambda+2}{3}}
A^{\frac{2(\Lambda+1)}{3}}
K^{\frac{1}{3}} \\
& +\;
H S^{\frac{\Lambda+1}{2}}A^{\frac{\Lambda+1}{2}}K^{\frac{1}{2}} \\
& +\;
H^{\frac{\Lambda}{2}}K^{\frac{1}{2}} \\
& +\;
H^{2+\frac{2\Lambda}{3}}
S^{\frac{5-2\Lambda}{3}}
A^{\frac{2-\Lambda}{3}}
K^{\frac{1}{3}}.
\end{align*}
\end{theorem}
\begin{proof}
By \Cref{lem:reduction} we only need to bound:
\begin{align*}
    \bar{\R}_K^{\mathcal{H}}\coloneqq\E\squary{\sum_{k}V_1^{k,\pi_{\A^*}^k}(s_{init};\;\bar{\ell}) - V_1^{k,\pi^*}(s_{init};\;\bar{\ell})}
\end{align*}
Which is essentially the regret of $\A^*$. 

We will use the same decomposition of the regret as in \Cref{lem:regret_decomposition}. Notice that for \textsc{Bias1},\textsc{Bias2} and \textsc{Error} the fact that we use $\bar{\ell}$ instead of $\ell$ doesn't make any difference in expectation. That is because the randomness of the outer algorithm is independent of the randomness of the inner algorithm and thus the expectation is separable and $\E_k\squary{\bar{\ell}_k} = \ell_k$. For example for the \textsc{Error} term:
\begin{align*}
    \E\squary{\textsc{Error}} &= \E\squary{\sum_{k,h,s,a}\roundy{q_h^{p_k,\pi_k}(s,a)-\sum_{c\in\C_h}\hat{\mu}_h^k(s,a,c)\varrho_k(c)}\bar\ell_h^k(s,a)}\\
    &= \sum_{k,h,s,a}\E_k\squary{\roundy{q_h^{p_k,\pi_k}(s,a)-\sum_{c\in\C_h}\hat{\mu}_h^k(s,a,c)\varrho_k(c)}\bar\ell_h^k(s,a)}\tag{tower rule}\\
    &= \sum_{k,h,s,a}\E_k\squary{\roundy{q_h^{p_k,\pi_k}(s,a)-\sum_{c\in\C_h}\hat{\mu}_h^k(s,a,c)\varrho_k(c)}}\E_k\squary{\bar\ell_h^k(s,a)}\\
    &= \sum_{k,h,s,a}\E_k\squary{\roundy{q_h^{p_k,\pi_k}(s,a)-\sum_{c\in\C_h}\hat{\mu}_h^k(s,a,c)\varrho_k(c)}}\ell_h^k(s,a)
\end{align*}
The same argument goes for \textsc{Bias1} and \textsc{Bias2}. However, the same thing doesn't go for \textsc{Reg} - since the algorithm sees $\bar{\ell}$ and not $\ell$, the same bound won't work if we change $\bar{\ell}$ to $\ell$. 

Since the estimator is optimistic in expectation, the expectation of \textsc{Bias2} is negative so we can omit that. 

Similar to the proof of \Cref{lem:reg} we can bound:
\begin{align*}
    \E\squary{\textsc{Reg}} \le \tilde{O}\roundy{\frac{HS^{\Lambda}}{\eta} + \eta\roundy{\roundy{SA}^{\Lambda}\sum_{k,h,s,a}\E\squary{\E_k\squary{\bar{\ell}_h^k(s,a)^2}} + \frac{H}{\gamma}}}
\end{align*}

We have for every $k,a,s,h$:
\begin{align*}
   \E_k\squary{\bar{\ell}_h^k(s,a)^2}  = \E_k\squary{\frac{\mathds1_k\squary{I_k=i^*}}{\nu^k(i^*)^2}} = \frac{1}{\nu^k(i^*)} \le \frac{1}{\xi}
\end{align*}

Which means:
\begin{align*}
    \E\squary{\textsc{Reg}} \le \tilde{O}\roundy{\frac{HS^{\Lambda}}{\eta} + \eta\roundy{\frac{1}{\xi}KH\roundy{SA}^{\Lambda+1} + \frac{H}{\gamma}}}
\end{align*}

From \Cref{lem:bias1}:
\begin{align*}
    \E\squary{\textsc{Bias1}} \le \sum_{k,h,s,a,c}\squary{\varrho_k(c)\roundy{u_h^k(s,a,c) - \mu_h^k(s,a,c)}} + \gamma KH\roundy{SA}^{\Lambda+1}
\end{align*}

From \Cref{lem:reduction_huge_lemma}:
\begin{align*}
    \E\squary{\textsc{Bias1}}= \tilde{O}\roundy{\gamma KH\roundy{SA}^{\Lambda+1} + HS\sqrt{\frac{KA}{\xi}} + \frac{S^2H^2A}{\xi}}
\end{align*}

In the same way, from \Cref{lem:error,lem:reduction_huge_lemma}:
\begin{align*}
    \E\squary{\textsc{Error}} \le \tilde{O}\roundy{HS\sqrt{\frac{KA}{\xi}} + \frac{S^2H^2A}{\xi}}
\end{align*}

From \Cref{lem:regret_decomposition}:
\begin{align*}
    \E\squary{\bar{\R}_K^{\mathcal{H}}} = \tilde{O}\roundy{\frac{HS^\Lambda}{\eta} + \roundy{\frac{\eta}{\xi} + \gamma}KH\roundy{SA}^{\Lambda+1} + \frac{H}{\gamma} + HS\sqrt{\frac{KA}{\xi}} + \frac{S^2H^2A}{\xi} + \frac{S^2H^2A}{\xi} + \sqrt{H^\Lambda K} + \xi H^{\Lambda + 1}K}
\end{align*}

Placing $\eta,\gamma,\xi$ concludes the proof.
\end{proof}

\footerofallfooters
\end{document}